\documentclass[10pt]{article} 
\usepackage[preprint]{tmlr}

\usepackage{amsmath,amsfonts,bm}

\def\eqref#1{equation~\ref{#1}}

\def\1{\bm{1}}

\DeclareMathAlphabet{\mathsfit}{\encodingdefault}{\sfdefault}{m}{sl}
\SetMathAlphabet{\mathsfit}{bold}{\encodingdefault}{\sfdefault}{bx}{n}

\usepackage{hyperref}
\usepackage{url}
\usepackage{amsthm}
\usepackage{graphicx}
\theoremstyle{definition}

\newtheorem{theorem}{Theorem}[section]

\usepackage{multirow}
\usepackage{makecell}
\usepackage{pifont}
\usepackage{subcaption} 
\usepackage{wrapfig}
\usepackage{tabularx}
\usepackage{arydshln}
\usepackage{booktabs} 
\usepackage{lipsum} 
\title{Uncertainty-quantified Pulse Signal Recovery from Facial Video using Regularized Stochastic Interpolants}

\newcommand{\VS}[1]{\textcolor{black}{#1}}

\author{\name Vineet R. Shenoy \email vshenoy4@jhu.edu \\
      \addr Department of Electrical and Computer Engineering\\
      Johns Hopkins University
      \AND
      \name Cheng Peng \email  xuz7wn@virginia.edu \\
      \addr School of Data Science\\
      University of Virginia
      \AND
      \name Rama Chellappa \email rchella4@jhu.edu \\
      \addr Department of Electrical and Computer Engineering\\
      Johns Hopkins University
      \AND
      \name Yu Sun \email ysun214@jhu.edu\\
      \addr Department of Electrical and Computer Engineering\\
      Johns Hopkins University
      }

\newcommand\norm[1]{\left\lVert#1\right\rVert}

\newcommand{\cmark}{\ding{51}}
\newcommand{\xmark}{\ding{55}}%
\def\month{MM}  
\def\year{YYYY} 
\def\openreview{\url{https://openreview.net/forum?id=XXXX}} 

\begin{document}

\maketitle

\begin{abstract}
Imaging Photoplethysmography (iPPG), an optical procedure which recovers a human's blood volume pulse (BVP) waveform using pixel readout from a camera, is an exciting research field with many researchers performing clinical studies of iPPG algorithms. While current algorithms to solve the iPPG task have shown outstanding performance on benchmark datasets, no state-of-the art algorithms, to the best of our knowledge, performs test-time sampling of solution space, precluding an uncertainty analysis that is critical for clinical applications. We address this deficiency though a new paradigm named \textit{Regularized Interpolants with Stochasticity for iPPG (RIS-iPPG)}. Modeling iPPG recovery as an inverse problem, we build probability paths that evolve the camera pixel distribution to the ground-truth signal distribution by predicting the instantaneous flow and score vectors of a time-dependent stochastic process; and at test-time, we sample the posterior distribution of the correct BVP waveform given the camera pixel intensity measurements by solving a stochastic differential equation. Given that physiological changes are slowly varying, we show that iPPG recovery can be improved through regularization that maximizes the correlation between the residual flow vector predictions of two adjacent time windows. Experimental results on three datasets show that RIS-iPPG provides superior reconstruction quality and uncertainty estimates of the reconstruction, a critical tool for the widespread adoption of iPPG algorithms in clinical and consumer settings. 
\end{abstract}

\section{Introduction}\label{sec:introduction}
\begin{figure*}
    \centering
    \includegraphics[width=0.79\textwidth]{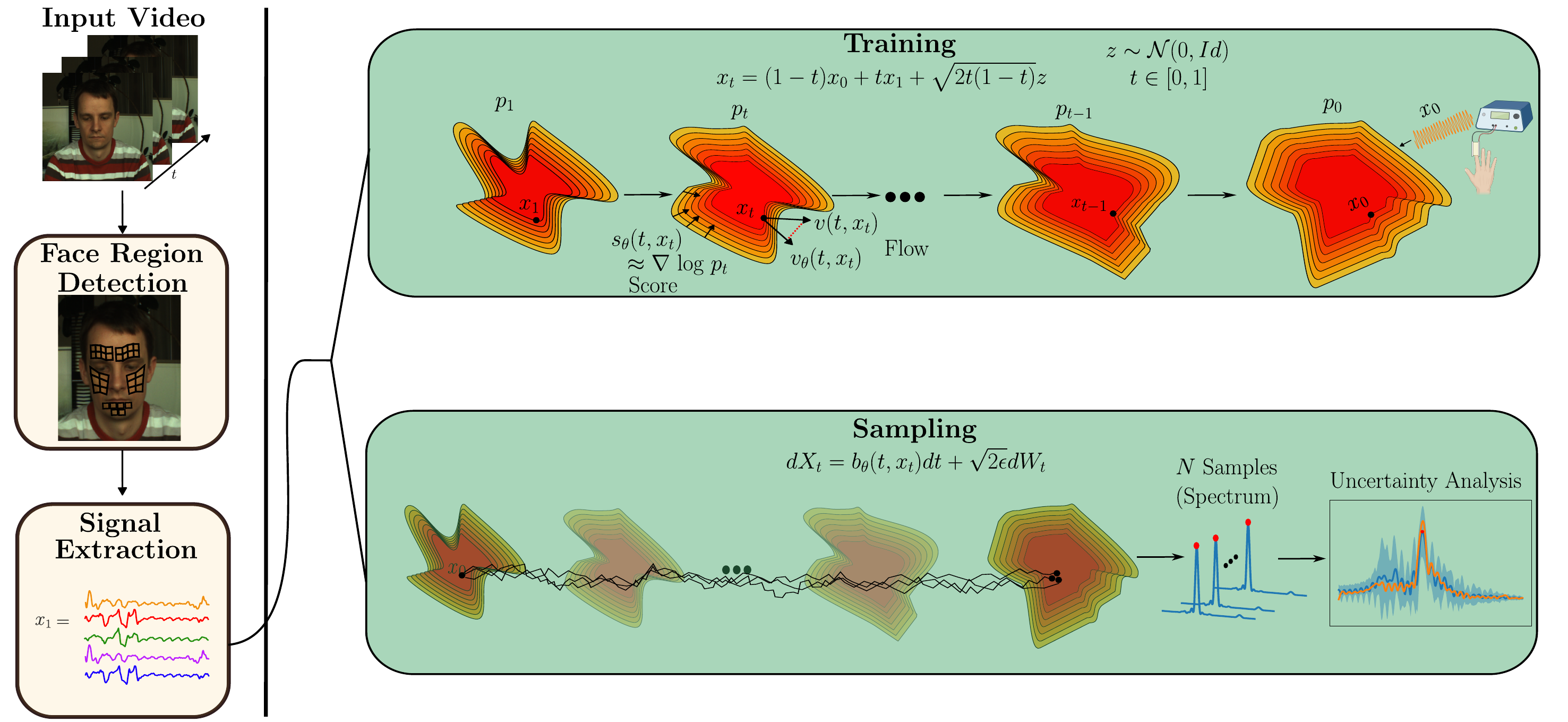}
    \caption{We first preprocess the video to extract a signal estimate from various facial regions. During training time, we learn the score and flow  between the measurement distribution and the ground-truth distribution, regularized based on the temporal characteristics of the signal. During test time, we solve an SDE with the measurement as the initial condition, and sample solution space. We then perform an uncertainty analysis.}
    \label{fig:main_fig}
\end{figure*}

Vital signs estimation using cameras has recently received strong interest in the research community. Extending photoplethysmography (PPG)---the technique in which a light is shined transdermally through the skin, the reflections of which capture volumetric changes due to blood flow---imaging Photoplethysmography (iPPG) seeks to observe the same volumetric changes using a non-contact imager of the skin, typically an RGB camera. The pixel intensity of the camera, under mild assumptions and noise,
captures the skin pigmentation changes due to blood flow. Current iPPG algorithms that denoise the camera signal to estimate the BVP signal are based on traditional signal processing or deep learning methods. Traditional signal processing methods~\cite{de2013robust,de2014improved,poh2010non,lewandowska2012measuring,nowara2021benefit} recover signals in training-free paradigms by assuming inherent signal structure and priors---whether it be statistical independence~\cite{poh2010non}, uncorrelated signals~\cite{lewandowska2012measuring}, color demixing~\cite{de2013robust,de2014improved}, or Fourier-based sparsity~\cite{nowara2020near}.
Deep learning methods generally perform better, but require training: supervised learning methods use synchronized facial video and contact PPG, and assume that the pulse signal structure is best learned though a model that maps from video to PPG data. Advances in self-supervised~\cite{gideon2021way,sun2024contrast,yue2023facial,speth2023non,liu2024rppgmae} algorithms learn using facial video data only, obviating contact PPG, and achieve competitive performance with fully supervised methods. These algorithmic advances on lab data have led to  clinical validation studies of iPPG algorithms~\cite{wangClinicalinfant} in neonatal care units, or in emergency departments for acute trauma injuries~\cite{00006534-990000000-02657}.

However, \VS{previous iPPG algorithms are deficient at rationalizing the results at test time for end users like clinicians.  A previous study~\cite{tonekaboni2019clinicians} interviewed clinicians regarding their \textit{trust} of machine learning models, which noted that ``metrics such as reliability, specificity, and sensitivity were important for the initial uptake of an AI tool, [but] a critical factor for continued usage was whether the tool was repeatedly successful in prognosticating patient's condition in [the doctor's] personal experience"}. Each of the aforementioned algorithms achieved state-of-the-art performance on population level metrics, but to the best of our knowledge, \VS{no previous iPPG algorithm} samples the solution space of potential iPPG signals at test time \textit{for each test sample}. Sampling the solution space would allow for \textit{uncertainty quantification} for each individual test-time sample, a crucial tool for eventual clinical adoption of such algorithms~\cite{begoli2019need, tonekaboni2019clinicians}. \VS{This type of analysis is especially important on protected attributes like skin tone and gender}. 

To address uncertainty quantification, we propose \textit{Regularized Interpolants with Stochasticity for imaging Photoplethysmography (RIS-iPPG)}, a flow-based diffusion model framework that learns a probability path from the distribution of camera pixel intensity signals to blood volume pulse signals, and allows for posterior sampling and \textit{uncertainty estimation} of the recovered pulse signals. We achieve such advances by formulating pulse signal recovery as an inverse problem with coupled camera measurements and ground-truth signals. We then learn the drift coefficient of a Stochastic Differential Equation (SDE)~\cite{albergo2023stochasticinterpolantsunifyingframework} that maps the distribution of camera measurements to the distribution of ground-truth signals, implemented in practice by learning flow and score vectors for the training data. However,  unregularized flow models do not typically yield best predictions. Given that physiological changes in blood volume are usually slowly varying~\cite{nichols2022mcdonald}, we propose to regularize the predicted flow by maximizing the correlation between the residual flow vectors (i.e. ground-truth flow minus predicted flow) from two adjacent time windows. After training our regularized model, at inference time we extract a pulse signal estimate from facial video and repeat it $N$ times as our initial condition, where $N$ is user specified,  and solve for the pulse signal via the aforementioned SDE. Given the $N$ signal estimates, we perform an uncertainty analysis of pulse signal recovery from a facial video.

In summary, our contributions are as follows:
\begin{itemize}
    \item We formulate pulse signal recovery from video as a posterior sampling method using flow-based diffusion models. Using the framework of stochastic interpolants, we learn the flow and score vectors that, when integrated into the drift coefficient of an SDE, transform the camera pixel signal to the blood volume pulse signal.
    \item We propose a Residual Correlation Loss (RCL) that maximizes the correlation between the residuals of predicted flow vectors from two overlapping, adjacent time windows. We show that this regularization can lead to better recovery results.
    \item We evaluate our algorithm using three datasets and perform test-time sampling of solution space. We show that even when our final prediction is incorrect, our sampling procedure highlights other possible solutions. We are able to capture the modes of the distribution more effectively, while also minimizing the uncertainty around frequency bins that are not of interest. \VS{Finally, we evaluate the quality of our uncertainty estimates on protected attributes of the data, establishing the first baselines for uncertainty calibration for the iPPG task.}
\end{itemize}

\section{Related Works}\label{sec:related_works}
\subsection{Imaging Photoplethysmography}\label{subsec:iPPG_related_works}

After preliminary investigations~\cite{wu2000photoplethysmography} showed that peripheral blood volume could be measured using an RGB imager, signal processing algorithms and later deep learning algorithms have been proposed to recover the blood volume signal in noise. Early signal processing methods, modeling the camera signal as a mixture of signals one of which was the pulse, demixed the signals by assuming the pulse signal to be statistically independent of the mixture (Independent Component Analysis)~\cite{lewandowska2012measuring} or assumed signals could be demixed along the directions of maximum variance (Principal Component Analysis)~\cite{poh2010non}. Recognizing skin reflection properties were critical for pulse signal recovery, both \cite{de2013robust} and \cite{wang2016algorithmic} performed skin tone corrections before projecting the signal onto an optimal plane for recovery. Other researchers viewed iPPG recovery as an optimization problem and imposed explicit sparsity~\cite{nowara2020near} and low-rank~\cite{tulyakov2016self} constraints on the recovered signal. 

Recent iPPG research is dominated by deep learning methods, which have shown improved performance compared to model-based methods. Underlying all these methods is the assumption that the non-linear noise and imaging processes that modulate the pulse signal from the body to the camera can be learned through sophisticated neural networks. Using paired video and ground-truth data, previous works developed techniques such as frame differencing~\cite{chen2018deepphys}, temporal shift modules~\cite{liu2020multi}, spatiotemporal CNN~\cite{yu2019remote}, and transformers~\cite{yu2022physformer} to extract the rPPG signals; these architectures were further adapted by others to learn noise profiles~\cite{nowara2021benefit, liu2023robust} for better signal denoising. Recent algorithmic advances have demonstrated that deep learning-based iPPG extractors can be learned without a supervisory PPG signal~\cite{gideon2021way, yue2023facial, sun2024contrast, liu2024rppgmae} and have achieved competitive performance with fully supervised methods. Researchers have also explored newer problem domains for iPPG such as federated learning~\cite{liu2022federated}, few-shot learning~\cite{liu2021metaphys}, and on-device iPPG recovery~\cite{liu2023efficientphys}. 


\subsection{Flow-based Diffusion Models}\label{subsec:flow_related_works}
The inspiration for density estimation and transport-based sampling was founded on Gaussianizing data through some transformation, and undoing that transformation to recover the distribution~\cite{tabak2010density, chen2000gaussianization}. A few works obtained the transformation as the solution of an ODE~\cite{chen2018neural, grathwohl2019scalable}, of which the drift coefficient could be learned through neural networks. However, learning the drift is intractable at large scale due to the simulation of the ODE for learning. While some other works proposed to regularize the path~\cite{finlay2020train, onken2021ot, tong2020trajectorynet}, many problems persisted.

Others took a stochastic view of the problem, modeling the transformation of a data distribution to a Gaussian as the evolution of of an Ornstein-Ulhenbeck (OU) process. Traversing this process forward in time simply involves adding Gaussian noise, while reversing this process can be done if the gradient of the log of the time-dependent data density is available~\cite{hyvarinen2005estimation, vincent2011connection}: this quantity, called the score function, could be estimated via least-squares regression~\cite{song2019generative}. The major drawback of this method was its reliance on the OU process and Gaussians; while some works used bridges to map between arbitrary distributions, these formulations were complex and inexact.

Recent work has introduced simulation-free methods for mapping between two arbitrary probability distributions. The key idea is that this mapping can happen gradually over time as one distribution transforms to another~\cite{lipman2022flow, tong2023conditional}. This can be generalized to stochastic dynamics as well~\cite{albergo2023stochasticinterpolantsunifyingframework}, which defines a stochastic process that maps from one point to another. In either the deterministic or stochastic versions, the goal is to learn the instantaneous change of the time-dependent distribution towards the target, which can be predicted via a neural network and is known as the "flow". In the stochastic case, the score is learned as well. After both are learned, the drift coefficient of the corresponding ODE/SDE can be learned, after which off-the-shelf solvers can solve the differential equation given an initial condition, mapping the point from one distribution to another. These techniques can learn the time-dependent vital sign trajectories of patients in the ICU~\cite{zhang2025trajectoryflowmatchingapplications}, generate new types of inorganic crystalline materials~\cite{hoellmer2025openmaterialsgenerationstochastic}, and learn the manifold of cellular dynamics~\cite{tong2023conditional}.

\section{Background: Stochastic Interpolants}\label{sec:background}

Our goal is to link \VS{two related probability distributions, the camera pixel distribution and the ground-truth pulse distribution}, and build a time-dependent probability path between them. The recently proposed work of Stochastic Interpolants~\cite{albergo2023stochasticinterpolantsunifyingframework} achieves this in finite time, and exactly, by defining a stochastic process that \VS{smoothly interpolates from a sample in one distribution to a sample in another distribution}. The key goal is to learn, through neural networks,  the instantaneous flow and score at all points interpolated between our two distributions. If the flow and score can be learned effectively, then we can take a series of steps (i.e. solving an SDE) that maps an initial condition (i.e. camera pixel signal) from one distribution to a point in the other distribution (i.e. blood volume pulse signal). Let our two arbitrary distributions be $p_0$ and $p_1$, let $\mathbf{x_0} \sim p_0$ and $\mathbf{x_1} \sim p_1$. \VS{Consider the following stochastic interpolant:}

\begin{equation}\label{eq:base_interpolant}
    \mathbf{x}_t = I(t, \mathbf{x}_0, \mathbf{x}_1) + \gamma(t)\mathbf{z},  \text{ where } \mathbf{z} \sim \mathcal{N}(0, Id), t\in[0, 1]
\end{equation}

\VS{The function $I(t, \mathbf{x}_0, \mathbf{x}_1)$ smoothly interpolates between the points $\mathbf{x}_0$ and $\mathbf{x}_1$ as a function of $t$ such that $\mathbf{x}_{t=0} = \mathbf{x}_0$ and $\mathbf{x}_{t=1} = \mathbf{x}_1$, while $\gamma(t)\mathbf{z}$ is an added noise term. To bridge between these points, we would like to know the instantaneous change of $\mathbf{x}_t$, which if learned for all $t$, would enable us to transport $\mathbf{x}_1$ to $\mathbf{x}_0$ or vice versa via an SDE.}  The instantaneous change of the interpolant $\mathbf{x}_t$ with respect to time, $\mathbf{b}(t,\mathbf{x})$, is  known as velocity, and the score of the distribution at a time $t$, $\mathbf{s}(t,\mathbf{x})$,  is

\begin{equation}
    \mathbf{b}(t,\mathbf{x}) = \mathbb{E}[\frac{\partial}{\partial t}\mathbf{x_t} | \mathbf{x_t =x}],  \mathbf{s}(t,\mathbf{x}) = \nabla \text{ log } p_t(\mathbf{x}) = -\gamma^{-1}(t) \mathbb{E}(\mathbf{z} | \mathbf{x_t=x})
\end{equation}

It is still to be shown, however, that these quantities can be used to form a valid drift coefficient of an SDE that maps probability mass from one distribution to another. We repeat below the theorem of Albergo which showed that this interpolant, and associated flow and score, satisfy both the continuity equation and a family of Fokker-Planck equations, allowing us to build the drift coefficient of an SDE that solves the mapping between the two distributions.

\begin{theorem}{(Theorem 2.6 of~\cite{albergo2023stochasticinterpolantsunifyingframework})}\label{eq:si_theorem}
The probability distribution of the interpolant $\mathbf{x}_t$ is absolutely continuous with respect to the Lebesgue measure at times $t\in[0,1]$ and solves the transport equation
\begin{equation}
    \frac{\partial}{\partial t} p_t + \nabla \cdot (\mathbf{b}p) = 0 
\end{equation}

In addition, the forward and backward Fokker-Planck equations are satisfied

\begin{equation}
    \frac{\partial}{\partial t} p_t + \nabla \cdot (\mathbf{b}_Fp) = \epsilon(t)\Delta p, \mathbf{b}_F = \mathbf{b}(t,\mathbf{x}) + \epsilon(t)\mathbf{s}(t,\mathbf{x}) 
\end{equation}
\begin{equation}
    \frac{\partial}{\partial t} p_t + \nabla \cdot (\mathbf{b}_Bp) = -\epsilon(t)\Delta p, \mathbf{b}_B = \mathbf{b}(t,\mathbf{x}) - \epsilon(t)\mathbf{s}(t,\mathbf{x}) 
\end{equation}

where $\epsilon(t)$ is some noise schedule.
\end{theorem}
\qed

As a consequence of this theorem, if we can learn the flow and score of the interpolant bridging our two data distributions, we can build a drift coefficient and solve a SDE that smoothly transforms a data point from one distribution to another.

\section{Problem Formulation and Approach}\label{sec:approach}

\subsection{The iPPG signal model}\label{subsec:signal_model}

The arterial tree can be modeled as a branching system of elastic tubes that carry blood to the body~\cite{nichols2022mcdonald}. Pressure differences at the ends of the tubes, induced by pump a called the heart, generate the flow of the liquid through the tubes~\cite{nichols2022mcdonald}. The flow can be measured using optical sensors which shine light transdermally and record reflection of the light corresponding to the blood volume; this technique is called photoplethsmography~\cite{alian2014photoplethysmography} and is common in consumer smartwatches. Imaging Photoplethysmography or remote Photoplethysmography aims to replicate PPG but replaces the contact-based optical sensor with a non-contact camera~\cite{mcduff2023camera}.

This imaging setup can be modeled as two processes. Given the volumetric flow signal $\mathbf{x}_0$, the first process generates an analog signal on the skin via reflections of incident illumination on $\mathbf{x}_0$ modulated by physical structures in the dermis/epidermis (known as the physiological forward process) and physiological noise. The second process converts the analog skin signal to a digital signal; in non-contact iPPG, the analog signal is modulated by digital camera hardware (known as the forward imaging process) and imaging noise. To simplify the model, we unify both processes to represent a digital camera signal $\mathbf{x}_1$ as:
\begin{equation}\label{eq:sig_model}
    \mathbf{x}_1 = A(\mathbf{x}_0) + \mathbf{n}
\end{equation}
where the unknown $A(\cdot)$ models the unified forward processes composed of the imaging forward processes acting on the result of the physiological forward process and noise, and $\mathbf{n}$ is the sum of the physiological and imaging noise. \VS{Previous work models the recovery of $\mathbf{x}_0$ as a regularized optimization problem with an approximate forward operator. In the next section, we show the deficiency of this approach, and the need for a new paradigm.}


\subsection{Preliminary Investigation}\label{subsec:preliminary_investigation}

\VS{We begin by generating $(\mathbf{x}_0, \mathbf{x}_1)$ pairs from the data. Given a video signal of time $T$, we extract time-domain pixel intensity signals from five regions of the face to generate $\mathbf{x}_1  \in \mathbb{R}^{T\times5}$ as in Figure~\ref{fig:main_fig}; the corresponding ground-truth BVP signal measured from the finger is repeated five times to generate $\mathbf{x}_0  \in \mathbb{R}^{T\times5}$, resulting in paired data $(\mathbf{x}_0, \mathbf{x}_1)$}. Our goal, given the camera pixel intensity signal $\mathbf{x}_1$, is to recover the signal $\mathbf{x}_0$ \textit{and} sample the space of possible solutions. \VS{Starting from Equation~\ref{eq:sig_model}, we decided to follow a regularized optimization scheme. Using $A = \mathbf{F}^{-1}$~\cite{shenoy2023unrolled, nowara2020near}, where $\mathbf{F}^{-1}$ is the inverse Fourier Transform, our initial investigation modeled signal recovery as a regularized optimization problem}, $\min_{\mathbf{x}} \frac{1}{2}\norm{\mathbf{x}_1 - A(\mathbf{x})}_2^2 + \lambda\cdot h(\mathbf{x})$, where $h(\mathbf{x}) = -\text{log } p(\mathbf{x})$ is the log of the data distribution. This is commonly referred to as the score function~\cite{song2019generative}, and is learned through neural networks. The optimization problem can then be solved with posterior sampling via the Plug-and-Play Monte-Carlo Approach of~\cite{sun2024provable}.  Our preliminary results, labeled as \textit{PMC-iPPG}, are in Table~\ref{tab:pmc_ippg_testing_main}.

\begin{table}[h!]
    \centering
    \caption{Applying PMC~\cite{sun2024provable} to the iPPG task, and comparing against regularized optimization methods with sparse priors~\cite{nowara2020near} and learned priors~\cite{shenoy2023unrolled}. Formulating the iPPG task as regularized optimization problem with a plugged-in prior is ineffective.}
    \resizebox{0.60\textwidth}{!}{\begin{tabular}{c|c|c|c}
         \toprule
         Method & Sampling?  &  MAE (bpm) $\downarrow$ & RMSE (bpm) $\downarrow$ \\
         \hline
         AutoSparsePPG~\cite{nowara2020near} & \xmark &  4.55 & 14.42 \\
         Unrolled-iPPG~\cite{shenoy2023unrolled} & \xmark & 1.11 & 2.97 \\
         PMC-iPPG & \cmark & 12.42 & 23.98 \\
         \bottomrule
    \end{tabular}}
    \label{tab:pmc_ippg_testing_main}
\end{table}

While this approach allows for test-time sampling, the performance is unsatisfactory. The key drawback is that the score only learns the signal prior i.e. the distribution of ground-truth volumetric signals, not the mapping $A(\cdot)$ from pulse to camera signals. \VS{The key failing is the approximation $A(\cdot) \approx \mathbf{F}^{-1}$, which is equivalent to a signal denoising problem in additive gaussian noise; this was a valid assumption in ~\cite{nowara2020near} because they assumed Gaussian denoising and  used orthogonal projection to reduce noise \textit{before} implementing their signal recovery algorithm. While the approximation of the unknown $A(\cdot) \approx \mathbf{F}^{-1}$ was sufficient for~\cite{shenoy2023unrolled}, their unrolling method \textit{implicitly} corrected the approximate forward model through end-to-end training. Fundamentally, however, PMC-iPPG reveals that $A=\mathbf{F}^{-1}$ is inadequate without extra noise reduction and implicit model correction.}

\VS{Naturally, the next step would be to build a more accurate forward model. Initial experiments to learn the forward model followed the approach of previous literature~\cite{Lunz2020OnLO, arridge2023inverse}; however, we noticed that models would not converge because the forward model is a time-dependent and facial-region-specific function of the pulse signal, specular and diffuse reflections, skin tone, and motion. See Appendix~\ref{sec:appendix_inadequacy} for a full discussion. Instead of learning networks to model both the forward model (which did not converge) as well as the signal prior, a single posterior sampling framework that encapsulates both of the aforementioned models and the many-to-one mapping between face regions and pulse signals may be more effective. Stochastic interpolants provides such a framework that allows for test-time sampling.}

\subsection{Unregularized Stochastic Interpolants for iPPG}\label{sec:unreg_si}

Without an explicit forward model, we seek to learn an implicit mapping of BVP signals to camera pixel signals. We assume that there exists a distinct BVP and camera pixel distribution, and that there exists a mapping \textit{in distribution} between them. 

Given these two paired data distributions, \VS{we seek to define a a stochastic process that smoothly interpolates between samples from two distributions, namely the camera intensity signals $\mathbf{x}_1\sim p_1$ and its ground-truth pulse signal $\mathbf{x}_0 \sim p_0$.} 


\vspace{-1.0em}
\begin{equation}\label{eq:interpolant}
    \mathbf{x}_t = (1-t)\mathbf{x}_0 + t\mathbf{x}_1 + \sqrt{2t(1-t)} \mathbf{z}
\end{equation}

\VS{This choice of $I(t, \mathbf{x}_0, \mathbf{x}_1) = (1-t)\mathbf{x}_0 + t\mathbf{x}_1$ and $\gamma(t) = \sqrt{2t(1-t)}$ as in Equation~\ref{eq:base_interpolant} ensures that at $\mathbf{x}_{t=0} = \mathbf{x}_0$ and $\mathbf{x}_{t=1} = \mathbf{x}_1$. Using this interpolant, our goal is to build a drift coefficient that satisfies Theorem~\ref{eq:si_theorem}. We are looking for the instantaneous change of Equation~\ref{eq:interpolant}: therefore, we are seeking
$\mathbf{b}(t,\mathbf{x}) = \mathbb{E}[\frac{\partial}{\partial t }\big((1-t)\mathbf{x}_0 + t\mathbf{x}_1 + \sqrt{2t(1-t)} \mathbf{z}\big)]$}, learned through neural networks, as described in Section~\ref{sec:background}. In practice, however, we can decompose $\mathbf{b}(t,\mathbf{x}) = \mathbf{v}(t,\mathbf{x}) - \gamma(t) \cdot \frac{\partial}{\partial t} \big(\sqrt{2t(1-t)}\big) \cdot \mathbf{s}(t,\mathbf{x})$, and further decompose the score using Tweedie's formula as $\mathbf{s}(t,\mathbf{x}) = -\mathbf{n_z}(t,\mathbf{x}) / \gamma(t)$.\VS{The decomposition of the drift and denoiser can be found in Section~\ref{subsec:appendix_deriving_sde} and Section~\ref{subsec:denoiser_deriving} and follows the derivation presented in~\cite{albergo2023stochasticinterpolantsunifyingframework}}. This simplifies the practical implementation to learning two independent networks, one to learn the interpolant flow and another to learn the denoiser. 
\begin{equation}\label{eq:flow_loss_august}
    \mathbf{v}_\theta(t, \mathbf{x}) \approx \frac{\partial}{\partial t} \big((1-t)\mathbf{x}_0  + t\mathbf{x}_1\big)
\end{equation}
\begin{equation}\label{eq:score_loss_august}
     \mathbf{n}_\theta(t,x) \approx \mathbf{z}
 \end{equation}
 
 \VS{This is preferable to learning the drift $\mathbf{b}(t, \mathbf{x})$ directly as we can avoid evaluating $\gamma^{-1}(t)$ near $t\approx 0$ and $t\approx 1$}. We can learn both of these networks by minimizing the MSE ~\cite{albergo2022building}: 

\begin{equation}\label{eq:flow_loss}
    \mathcal{L}_{\text{flow}} = \mathcal{L}_{\text{MSE}}(\mathbf{v}_\theta(t, \mathbf{x}), \mathbf{v}(t,\mathbf{x}))
\end{equation}
\vspace{-1.0em}
\begin{equation}\label{eq:score_loss}
    \mathcal{L}_{\text{score}} = \mathcal{L}_{\text{MSE}}(\mathbf{n}_\theta(t, \mathbf{x}), \mathbf{z})
\end{equation}

 After these networks are learned, we can build a drift coefficient (as shown in Section~\ref{subsec:test_time_sampling}) and use off-the-shelf solvers to obtain a pulse estimate given the camera measurements (i.e. initial condition).

\begin{wrapfigure}{r}{0.30\textwidth}
\includegraphics[width=0.30\textwidth]{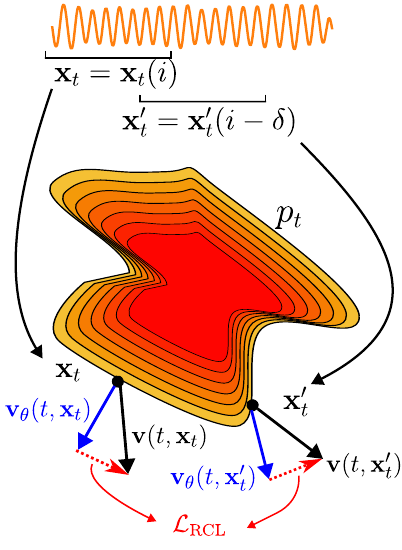} 
\caption{We sample a time-window and its time-shifted version, and predict the flow for both. For two adjacent and overlapping time-windows, the residual vector between predicted and ground-truth flows should point in the same direction, which is promoted by by minimizing the Residual Correlation Loss. }
\label{fig:correlation_loss}
\vspace{-2.0em} 
\end{wrapfigure}

While this worked in practice, we noticed many failure cases. \VS{Facial data is often corrupted by out-of-distribution and unconstrained motion noise. To make our models more robust, we attempted to condition our flow and score networks on guidance signals based on the noise in the video; however, our investigation yielded negative results (see Appendix~\ref{subsec:appendix_guidance}). We took a different approach when noticing that physiological changes are slowly varying. Medical research has discovered that physiological changes under normal conditions are often slowly varying~\cite{nichols2022mcdonald} and changes in physiological state are often time-delayed. This implies that in adjacent and overlapping time windows, physiological signals should be similar.  A robust flow model should ensure such temporal consistency, which we learned through a regularization scheme below.}


\subsection{Residual Correlation Loss (RCL) for temporal consistency in adjacent time windows}\label{subsec:rcl}

We learn such temporal consistency by correlating the residual vectors between the predicted and ground-truth flows of two adjacent time windows as shown in Figure~\ref{fig:correlation_loss}. Assume we are given two pairs of data,
\begin{equation}\label{eq:signal_stride}
    (\mathbf{x}_0(i), \mathbf{x}_1(i)) \text{ and } (\mathbf{x}_0' = \mathbf{x}_0(i - \delta), \mathbf{x}_1' = \mathbf{x}_1(i - \delta))
\end{equation}
where the latter pair is an overlapping, time-shifted version of the first pair. During training, we generate $\mathbf{x}_t$ and $\mathbf{x}_t'$ according to Equation~\ref{eq:interpolant}, after which we predict the flow at each of these points $\mathbf{v}_\theta(t, \mathbf{x}_t)$ and $\mathbf{v}_\theta(t, \mathbf{x}_t')$ as described in Section~\ref{sec:unreg_si}. We note the predicted flows should be regressed to their corresponding targets $\mathbf{v}(t, \mathbf{x}_t)$ and $\mathbf{v}(t, \mathbf{x}'_t)$; however, given that adjacent and overlapping time windows should have consistent physiological behavior, the error vector between the predicted and ground-truth flows (i.e. the residual) should be correlated. More precisely, we would like the residual vectors to point in the same direction.

To encourage vectors to point in the same direction, we aim to maximize the normalized dot product between the vectors. This is achieved by minimizing the proposed Residual Correlation Loss, which is equivalent to minimizing one-minus the Pearson Correlation Coefficient~\cite{cohen2009pearson}. Let $\mathbf{p} = \mathbf{v}(t, \mathbf{x}_t) - \mathbf{v}_\theta(t, \mathbf{x}_t)$ and $\mathbf{q} = \mathbf{v}(t, \mathbf{x}_t') - \mathbf{v}_\theta(t, \mathbf{x}_t')$. Then, the RCL loss is defined as 
\vspace{-0.5em}
\begin{equation}\label{eq:rcl_loss}
    \mathcal{L}_{\text{RCL}}(\mathbf{p}, \mathbf{q}) = 1- \frac{T \cdot \mathbf{p}^\top\mathbf{q} - \mu_p\mu_q}{\sqrt{(T\cdot \mathbf{p^\top p} - \mu_p^2)(T \cdot \mathbf{q}^\top \mathbf{q} - \mu_q^2)}}
\end{equation}

where $\mu_p$ and $\mu_q$ are the means of the signals, and $T$ is the length of the signal.  We will show in Section~\ref{sec:experiments} that minimizing this loss leads to improved iPPG recovery.

\VS{The correlation loss leads to a regularized network that promotes temporal consistency. This loss can simply be added at the training stage of our network, without any changes to the model architecture. Therefore, during training, we can add the RCL loss to the flow and score losses from Equation~\ref{eq:flow_loss} and Equation~\ref{eq:score_loss}.} The final loss is given by
\vspace{-0.15em}
\begin{equation}\label{eq:full_loss}
    \mathcal{L} = \underbrace{\mathcal{L}_{\text{MSE}}(\mathbf{v}_\theta(t, \mathbf{x}), \mathbf{v}(t,\mathbf{x}))}_{\text{flow}} + \underbrace{\mathcal{L}_{\text{MSE}}(\mathbf{n}_\theta(t, \mathbf{x}), \mathbf{z})}_{\text{denoiser}} + \lambda_{\text{RCL}} \mathcal{L}_{\text{{RCL}}}(\mathbf{p}, \mathbf{q})
\end{equation}

\subsection{Test-time Sampling}\label{subsec:test_time_sampling}

Our proposed method, as compared to all previous iPPG methods, is able to sample solution space at test time and generate multiple realizations of the solution. \VS{As compared to other posterior sampling methods, the benefit of using an SDE-based method is that we can visualize solutions at all time points, an example of which is shown in Appendix Figure~\ref{fig:time-layout}.} Given that our iPPG interpolant satisfies both the continuity equation and the Fokker-Planck equation from Theorem~\ref{eq:si_theorem}, we write a reverse SDE that, when solved, produces an estimate of the pulse signal given the measurements i.e. initial condition. First, define the sampling SDE as 
\vspace{-0.5em}
\begin{equation}\label{eq:sampling_sde}
    d\mathbf{x}_t = \mathbf{b}_F(t, \mathbf{x}_t)dt + \sigma_tdW_t, \text{  } \sigma_t = \sqrt{2\epsilon(t)}
\end{equation}

where $W_t$ is a Wiener process. The drift $\mathbf{b}(t, \mathbf{x})$ coefficient then becomes: 
\vspace{-0.5em}
\begin{equation}\label{eq:sampling_drift}
    \mathbf{b}_F(t, \mathbf{x}_t) = \Big[\mathbf{v}_\theta(t, \mathbf{x}_t) - \gamma(t)\cdot \big(\frac{d}{dt} \gamma(t)\big) \cdot \mathbf{s}_\theta(t, \mathbf{x}_t) \Big] + \epsilon(t) \mathbf{s}_\theta(t, \mathbf{x}_t)
\end{equation}

\VS{See Appendix~\ref{subsec:appendix_deriving_sde} for the derivation of the drift.} While any standard solver can be used to solve Equation~\ref{eq:sampling_sde}, we chose to use the implementations of~\cite{li2020scalable, kidger2021neuralsde}, and as recommended by ~\cite{albergo2023stochasticinterpolantsunifyingframework}, we set $\epsilon(t)$ to a constant for all $t$. In the next section, we present the results of our approach, and demonstrate the efficacy of both our uncertainty quantification as well as the RCL loss for iPPG recovery.

\section{Implementation Details and Experimental Results}\label{sec:experiments}

\begin{table*}[]
    \caption{Heart rate estimation results on MMSE-HR, PURE and UBFC-rPPG datasets. Best results in each column are \textbf{bold}; second-best are \underline{underlined}. Our method is the only one that addresses uncertainty quantification (UQ)}
    \centering
    \resizebox{0.75\textwidth}{!}{
    \begin{tabular}{c|c|ccc|ccc|ccc}
    \toprule
    \multirow{3}{*}{\textbf{Type}} & \multirow{3}{*}{\textbf{Method}}  & \multicolumn{3}{c}{MMSE-HR} &  \multicolumn{3}{c}{PURE} & \multicolumn{3}{c}{UBFC-rPPG} \\
    & & \textbf{MAE}$\downarrow$ & \textbf{RMSE}$\downarrow$& $\rho$ $\uparrow$ &  \textbf{MAE}$\downarrow$ & \textbf{RMSE}$\downarrow$& $\rho$ $\uparrow$ & \textbf{MAE}$\downarrow$ & \textbf{RMSE}$\downarrow$& $\rho$ $\uparrow$ \\
    \hline
\multirow{4}{*}{\makecell{Model-Based \\ Unsupervised}} & ICA~\cite{poh2010non}  & 5.44 & 12.00 & - &  - & - & - & 7.50 & 14.41 & 0.62 \\
    & CHROM~\cite{de2013robust} & 3.74 & 8.11 & 0.55 & 2.07 & 9.92 & - & 2.37 & 4.91 & 0.89 \\
    & POS \cite{wang2016algorithmic} & 3.90 & 9.61 & - & 5.44 & 12.00 & - & 4.05 & 8.75 & 0.78 \\
    & AutoSparsePPG\cite{nowara2020near}  & 4.55 & 14.42 & - & - & - & - & - & - & - \\
    \hline
    \multirow{12}{*}{\makecell{Model-Based \\ Unsupervised}} & HR-CNN~\cite{Spetlik2018VisualHR} & - & - & - & 1.84 & 2.37 & - & - & - & - \\
    & SynRhythm~\cite{synrhythm} & - & - & - & - & - & - & 5.59 & 6.82 & 0.72 \\  
    & CAN~\cite{chen2018deepphys} & 4.06 & 9.51 & - & - & - & - & - & - & - \\
    & CVD~\cite{niu2020video} & - & 6.04 & 0.84 & 1.29 & 2.01 & 0.98 & 2.19 & 3.12 & 0.99 \\
    & PulseGAN~\cite{SongPulseGAN} & - & - & -  & - & - & - & 1.19 & 2.19 & 0.98 \\
    & InverseCAN~\cite{nowara2021benefit} & 2.27 & 4.90 & - & - & - & - & - & - & - \\
/    
    & DualGAN~\cite{dualganLu2021} & - & - & - & 0.82 & 1.31 & 0.99 & \textbf{0.44} & \textbf{0.67} & \textbf{0.99} \\
    & Physformer~\cite{yu2022physformer} & 2.84 & 5.36 & - & - & - & - & - & - & - \\
    & Federated~\cite{liu2022federated} & 2.99 & - & 0.79 & - & - & - & 2.00 & 4.38 & 0.93 \\
    & EfficientPhys-C~\cite{liu2023efficientphys} & 2.91 & 5.43 & 0.86 & - & - & - & - & - & - \\
    & ContrastPhys-100 (PAMI'24)~\cite{sun2024contrast} & \textbf{1.11} & \textbf{3.83} & 0.96 & \underline{0.48} & \textbf{0.98} & 0.99 & 0.50 & 0.84 & 0.99 \\
    \hline
    \multirow{3}{*}{\makecell{Data-Driven \\ Unsupervised}} 
    &  Gideon~\cite{gideon2021way} &  3.98 & 9.65 & 0.85 & 2.3 & 2.9 & 0.99 & 3.60 & 4.60 & 0.95 \\
    & Yue ~\cite{self_supervised_rppg} & - & - & - & 1.23 & 2.01 & 0.99 & - & - & - \\
    &  ContrastPhys-0~\cite{sun2024contrast} & \underline{1.82} & 6.69 & 0.96 & 1.00 & 1.40 & 0.99 & - & - & - \\
    \hline
    \multirow{2}{*}{\makecell{\textbf{Ours} \\ Data-Driver, Supervised}}  & \multirow{2}{*}{RIS-iPPG}  & \multirow{2}{*}{1.97} & \multirow{2}{*}{\underline{3.73}} & \multirow{2}{*}{0.97} & \multirow{2}{*}{\textbf{0.38}} & \multirow{2}{*}{\underline{1.28}} & \multirow{2}{*}{0.99}  & \multirow{2}{*}{\underline{0.47}} & \multirow{2}{*}{\underline{0.80}} & \multirow{2}{*}{0.98}  \\
    & & & & & & & & & &  \\
    \bottomrule
    \end{tabular}}
    \label{tab:allresults}
\end{table*}

\subsection{Datasets}
We evaluate our algorithm using three datasets, which are described below

\begin{itemize}

    \item \textbf{MMSE-HR}~\cite{zhang2016multimodal,ertugrul2019cross}: The MMSE-HR dataset recorded facial video at 1040$\times$1392 pixels and 25 FPS while capturing synchronized blood pressure waveforms using a Biopac NIBP100D recording at 1000Hz (which we downsampled to 25 Hz to align with the video). Seventeen male and twenty-three female subjects were asked to perform a variety of tasks that induced motion as well as a change in heart rate, which resulted in 102 videos. We train and test on 10 second time windows, and evaluate using the leave-one-subject-out evaluation protocol of~\cite{nowara2021benefit}.

    \item \textbf{PURE}~\cite{Stricker2014NoncontactVP}: Recorded at 30 FPS and a resolution of 640$\times$480 pixels, the PURE dataset contains 10 subjects each of whom perform six task to induce facial motion. Corresponding pulse oximetry data were captured at 60Hz, which was downsampled to 30 Hz to align with the video data. The models were trained on 10-second time windows of pulse oximeter data, and were evaluated on 30-second windows according to the splits of~\cite{vspetlik2018visual}.

    \item \textbf{UBFC-rPPG}~\cite{ubfc-rppg}. The UBFC-rPPG dataset contains 43 subjects, each recording one video captured at $640\times480$ px and 30 FPS while playing a game to induce pulse rate changes. Simultaneous pulse waves were captured using a pulse oximeter recording at 30Hz. We train on 10-second windows with a frame stride of 2.4 seconds, and evaluate according to the protocol in~\cite{dualganLu2021}: we evaluate on 10-second time windows for each video, and average all heart rates in a video for a single heart rate estimate.
   
\end{itemize}

\subsection{Evaluation Metrics}

\textbf{Pulse Rate Estimation Metrics}: We follow the protocol of previous work~\cite{shenoy2023unrolled} and measure the predicted pulse rate versus the ground-truth pulse rate in the windows of interest. We compute the pulse rates by first multiplying the signal by a Hanning window, followed by taking an $L = 10 \times \text{signal length}$ FFT, after which we compute the power by squaring the magnitude of the FFT. We sum the power spectra across all facial regions, and all samples from the SDE solutions, after which we chose the frequency bin with the greatest power as the pulse rate. We then compute the mean absolute error (MAE) and root mean squared error (RMSE) between the predicted and ground-truth pulse rates, as well as the Pearson Correlation Coefficient between predicted and ground-truth pulse rates \VS{as in \cite{liu2022toolbox}}:

\textbf{Spectrum Estimation Accuracy Metrics:} \VS{We measure the accuracy of spectrum estimation by comparing the predicted spectrum against that of the ground-truth spectrum. We measure the spectrum MAE and RMSE, as well as other standard regression metrics such as the coefficient of determine ($R^2$) and the Pearson Correlation Coefficient (PCC).}

\begin{figure}
    \centering
    \includegraphics[width=0.79\textwidth]{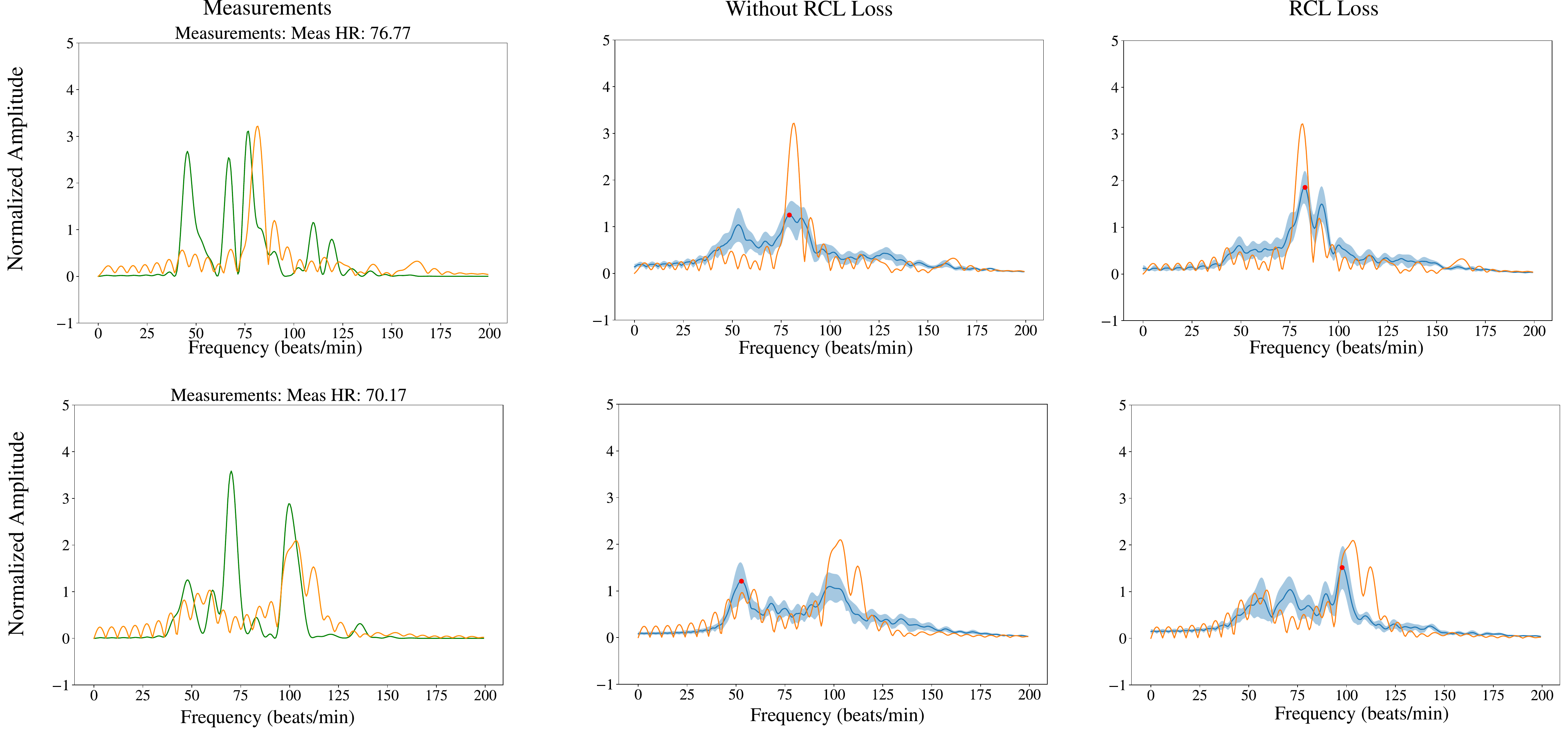}
    \caption{Qualitative results with and without the RCL loss. We plot the camera pixel measurements (green), ground-truth PPG (orange), the mean of 100 realizations of sampling (bolded blue), and  95\% confidence interval of the power in each bin (light blue). Our algorithm with the RCL loss is better at predicting the true heart rate. In row 2, we see that the RCL loss also captures the modes of the measurements more effectively (i.e. the correct peak around $\sim$100 bpm as well as the peak of the measurements $\sim$75 bpm).}
    \label{fig:multimodes}
\end{figure}
\textbf{Predictive Uncertainty Quantification Metrics}: \VS{The use of stochastic sampling allows us to measure predictive uncertainty. Using the implementation of~\cite{chung2021uncertainty}, we compare our predicted spectral magnitude against the ground-truth, and calculate the Negative Log Likelihood (NLL), the Sharpness, Continuous Ranked Probability Score, the Check Score, and the Interval Score of our predictive uncertainty. In addition, we plot the calibration curves of our method over the UBFC-rPPG and PURE test sets, as well as over protected attributes of the MMSE-HR dataset. A full description of these metrics are reserved for Appendix~\ref{sec:appendix_eval_mets} due to space constraints.}

\subsection{Implementation Details}\label{sec:implementation}

All computation, including preprocessing, was implemented on an A5000 NVIDIA GPU. To generate the pre-processed time-series, we first pass the input video through the OpenFace~\cite{amos2016openface} face detector, followed by the landmark detection using LDEQ~\cite{Micaelli_2023_CVPR}. These landmarks are interpolated across the face to delineate the right and left forehead, right and left cheek, and chin. In each region, we perform per-channel averaging of all pixels, and as in~\cite{shenoy2023unrolled}, we take the ratio of the red channel to the green channel to obtain our signal estimate as in Figure~\ref{fig:main_fig}. 

Identical U-Nets, adapted from guided diffusion~\cite{dhariwal2021diffusion}, learn both the flow and score in our framework. We train these networks using the Adam optimizer~\cite{kingma2014adam} with an initial learning rate of 1$e$-3. The number of training epochs varied for each dataset. Please refer to the Appendix Section~\ref{sec:appendix_hyperparameters} for more information.

\subsection{Results}\label{subsec:results}

\textbf{Pulse Rate Estimation}: \VS{Before we present UQ metrics, we must determine whether our method is accurate.} We present the results of RIS-iPPG with the RCL loss in Table~\ref{tab:allresults}. The results for the MMSE-HR dataset are aggregated over subject-independent (i.e. 40-fold) cross-validation, while the results for PURE and UBFC-rPPG are evaluated on their test sets in accordance with previous literature. On all three datasets, we achieve very competitive performance with previous benchmarks, and achieve a new state-of-the-art on the PURE dataset. For the benchmarks in which we do not perform best, we still achieve $<$ 1bpm error on UBFC-rPPG and $<$ 2 bpm on MMSE-HR.

Qualitative results of the power spectrum are shown in Figure~\ref{fig:multimodes}, where the orange signals are the ground-truth in all cases, the green signals are the measurements, and the bolded blue signals are the means from 100 realizations of our algorithm. The light blue shading represents the 95\% confidence interval of the power over each frequency bin of our spectrum, and the red dot shows the peak of the spectrum. In both cases, we notice that the model with the RCL loss produces a more accurate pulse rate estimation and spectrum. The second row shows a particularly interesting example: with the RCL loss, we see greater uncertainty at the three modes of the distribution (the ground-truth frequency, the peak frequency in the measurements, and the frequency at $\sim$55bpm) as compared to the prediction without the RCL loss, which not only incorrectly predicts the pulse rate, but also smooths the spectrum and fails to capture the uncertainty around the peak measurement frequency. More examples on MMSE-HR, UBFC-rPPG, and PURE are included in Appendix~\ref{sec:appendix_qualitative}.

\begin{figure}
    \centering
    \includegraphics[width=0.85\textwidth]{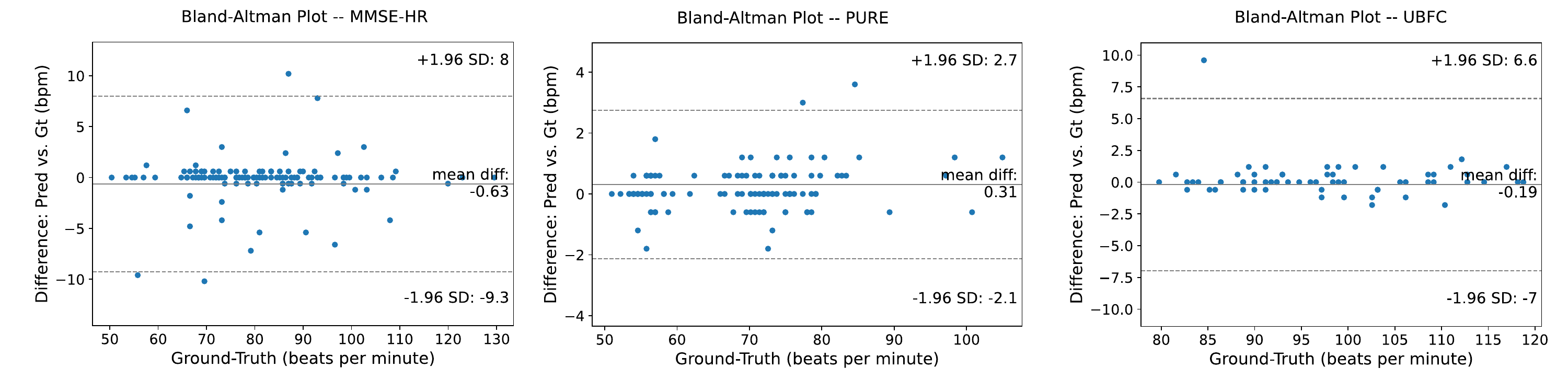}
    \caption{Bland-Altman Plots for the predicted heart rate against the ground-truth for all time windows on the test sets. We plot the ground-truth heart rate against the predicted heart rate, as well as the 95\% confidence intervals. We see that the mean difference is close to zero for all three datasets, with reasonable confidence intervals.}
    \label{fig:three_ba_plots}
\end{figure}


\textbf{Statistical Analysis}: In addition to a comparison against the state-of-the-art, we perform a Bland-Altman analysis to understand the performance of our algorithm. We plot the ground-truth heart rate in a window against the prediction from RIS-iPPG, and plot the 95\% confidence interval for the heart rate predictions. The results from the MMSE-HR dataset show a mean difference of approximately -0.63 bpm, indicating that our method overestimates the actual heart rate by an average of over half a beat. Our heart standard deviation is relatively large; however, this is metric is dominated by the few large outliers produced when the input measurements are too noisy. We see similar trends on both the UBFC-rPPG dataset and the PURE dataset. Models on both datasets achieve a mean difference close to zero, while maintaining single-digit standard deviations. 

\textbf{Effect of overlap and weight}: We further explore the inclusion of the RCL loss by performing a grid-search over the weight parameter $\lambda_{\text{RCL}} \in [0.0, 1.0]$ and the time-shift $\delta \in [1, 10]$ seconds, with the results as shown in Table~\ref{tab:rcl_ablation1}. Note that we train our model using 10-second windows; therefore, a stride of $\delta = 10$ corresponds to the ``no-overlap'' scenario. On the MMSE-HR dataset, we achieve best performance with a stride of $\delta = 9$ seconds and a weight of $\lambda=0.1$. We also note that we can get significant improvements over traditional stochastic interpolants by including the RCL loss. This experiment was conducted on a small validation set of the MMSE-HR dataset; after selecting $\delta = 9$ seconds and a weight of $\lambda=0.1$, we perform 40-fold cross-validation and report the final results in Table~\ref{tab:allresults}. Ablations on the UBFC-rPPG dataset are included in Appendix Table~\ref{tab:rcl_ablation2}.

\textbf{\VS{Spectrum Estimation Performance}}: \VS{In Table~\ref{tab:waveform_estimation} we measure the spectrum estimation accuracy average \textit{over all frequency bins} (i.e. as compared to a single heart rate) calculated for the pulse rate. As shown in Table~\ref{tab:waveform_estimation}, our model with the RCL loss does better than the model without the RCL loss on both the PURE and UBFC-rPPG datasets. Our RMSE, indicating the magnitude of the outliers, is lower for the model with the RCL loss, and the PCC between the predicted and ground-truth spectrum magnitude is higher. These results indicate that the RCL loss is beneficial to RIS-iPPG.  }

\begin{table}[]
    \caption{Varying the stride $\delta$ of Equation~\ref{eq:signal_stride} and weight $\lambda_{\text{RCL}}$ of Equation~\ref{eq:full_loss} for the RCL loss on the MMSE-HR Dataset. We notice best performance, on average, when using $\delta=9$ and $\lambda_{\text{RCL}} = 0.1$}
    \centering
    \setlength\tabcolsep{6pt}
    \resizebox{0.65\textwidth}{!}{
    \begin{tabular}{cc|cccccccccc|c}
    \toprule
    & & \multicolumn{10}{c|}{Window Shift $\delta$ (seconds)} &   \\
    & & 1 & 2 & 3 & 4 & 5 & 6 & 7 & 8 & 9 & 10  & Average\\
    \hline
    \parbox[t]{2mm}{\multirow{10}{*}{\rotatebox[origin=c]{90}{RCL weight $\lambda_{\text{RCL}}$}}} & 0.0 & 3.72 & 3.72 & 3.72 & 3.72 & 3.72 & 3.72 & 3.72 & 3.72 & 3.72 & 3.72 & 3.72 \\
    & 0.1 & 2.67 & 2.54 & 2.47 & 3.44 & 2.79 & 2.54 & 3.51 & 2.52 & 1.39 & 2.89 & \textbf{2.67} \\
    & 0.2 & 2.64 & 3.77 & 2.42 & 2.04 & 2.77 & 4.07 & 2.82 & 1.94 & 2.52 & 2.59 & 2.74\\
    & 0.3 & 2.54 & 4.67 & 3.87 & 2.39 & 4.74 & 2.92 & 3.09 & 2.42 & 2.34 & 2.24 & 3.12 \\
    & 0.4 & 2.59 & 2.42 & 5.97 & 3.39 & 4.19 & 2.67 & 2.69 & 3.42 & 2.12 & 2.87 & 3.23\\
    & 0.5 & 3.02 & 2.67 & 2.51 & 2.57 & 1.92 & 3.19 & 2.74 & 2.85 & 2.84 & 2.67 & 2.73\\
    & 0.6 & 3.02 & 3.01 & 2.89 & 3.07 & 2.12 & 2.77 & 2.72 & 2.47 & 3.72 & 2.34 & 2.83\\
    & 0.7 & 2.77 & 4.07 & 2.19 & 2.69 & 2.92 & 2.44 & 2.44 & 3.09 & 1.82 & 3.57 & 2.8\\
    & 0.8 & 2.67 & 2.87 & 2.84 & 2.09 & 4.57 & 3.69 & 3.22 & 2.37 & 2.14 & 2.49 & 2.89\\
    & 0.9 & 1.77 & 2.57 & 2.25 & 3.32 & 1.49 & 3.34 & 2.94 & 3.82 & 2.72 & 2.89 & 2.71\\
    & 1.0 & 2.69 & 4.17 & 3.22 & 2.34 & 2.79 & 6.54 & 3.27 & 1.62 & 1.59 & 2.19 & 3.04\\
    \hline
    Average & & 2.73 & 3.34 & 3.12 & 2.82 & 3.09 & 3.44 & 3.01 & 2.74 & \textbf{2.44} & 2.79 &  \\
    \bottomrule

    \end{tabular}
    }
    \label{tab:rcl_ablation1}
\end{table}

\begin{table}[]
    \caption{\VS{Comparing the spectrum estimation performance and uncertainty quantification metrics with and without the RCL loss}}
    \centering
    \resizebox{0.65\textwidth}{!}{
    \begin{tabular}{c|ccc|ccccc}
    \toprule
    & \multicolumn{3}{c}{Waveform Accuracy} & \multicolumn{5}{c}{Uncertainty Metrics} \\
    Method & \textbf{MAE}$\downarrow$ & \textbf{RMSE}$\downarrow$& \ \textbf{PCC}$\uparrow$ & \textbf{NLL}$\downarrow$ & \textbf{Sharpness}$\uparrow$& \textbf{CRPS} $\downarrow$ &  \textbf{Check Score}$\downarrow$ & \textbf{Interval Score}$\downarrow$\\
    \hline
    PURE w/out RCL & 0.057 & 0.224  & 0.681 & \textbf{-3.825} &  0.201 & 0.041 & 0.021 & 0.217 \\
    PURE w/RCL & \textbf{0.056} & \textbf{0.218} & \textbf{0.700} & -3.820 & \textbf{0.202} & \textbf{0.040} & \textbf{0.020} & \textbf{0.208} \\
    \hline
    UBFC-rPPG w/out RCL & 0.073 & 0.237 & 0.651 & -3.475 & \textbf{0.177} & 0.042 & 0.021 & 0.203 \\
    UBFC-rPPG w/RCL & \textbf{0.049} & \textbf{0.185}  & \textbf{0.795} & \textbf{-3.516} & 0.168 & \textbf{0.034} & \textbf{0.017} & \textbf{0.162} \\
    \bottomrule
    \end{tabular}}
    \label{tab:waveform_estimation}
\end{table}


\subsection{\VS{Predictive Uncertainty Quantification}}

\VS{To the best of our knowledge, RIS-iPPG is the first iPPG method to develop an stochastic sampling method for iPPG estimation, which allows us to analyze predictive uncertainty and establish new baselines. We enumerate the UQ metrics averaged over the PURE and UBFC-rPPG test sets, respectively, as shown in Table~\ref{tab:waveform_estimation}. The metrics were measured by computing the Fourier spectrum magnitude across all bins of the predicted signal against that of the ground-truth. We notice that over most metrics, our model with residual correlation regularization performs better than the model without it, and over the metrics on which we do not as well, our results are still comparable.
The NLL values indicate that the ground-truth spectral magnitude has high likelihood given our model's estimated distribution, while the CPRS indicates that our model's predicted CDF is similar to that of the ground-truth. The check score (also known as the pinball loss), which measures quantile prediction performance, is lower for the RCL-regularized model, while the interval score, which balances the sharpness of the distribution versus the calibration, is lower for the model with the RCL loss.}

\begin{wrapfigure}{r}{0.3\textwidth}
    \centering
    \begin{subfigure}[b]{0.35\textwidth}
        \includegraphics[width=0.9\textwidth]{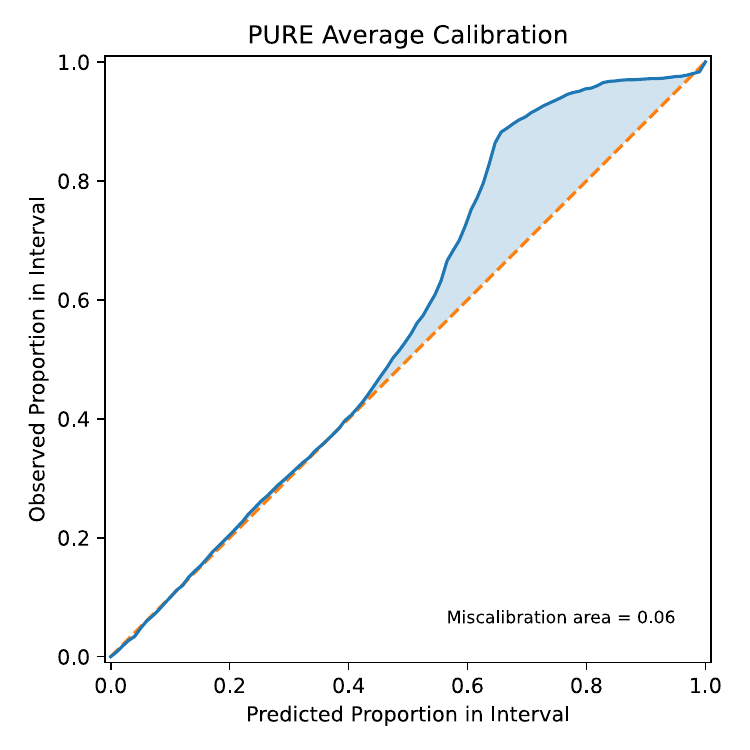}
        \label{fig:pure_calib}
    \end{subfigure}
    \\[1em] 
    \begin{subfigure}[b]{0.35\textwidth}
        \includegraphics[width=0.9\textwidth]{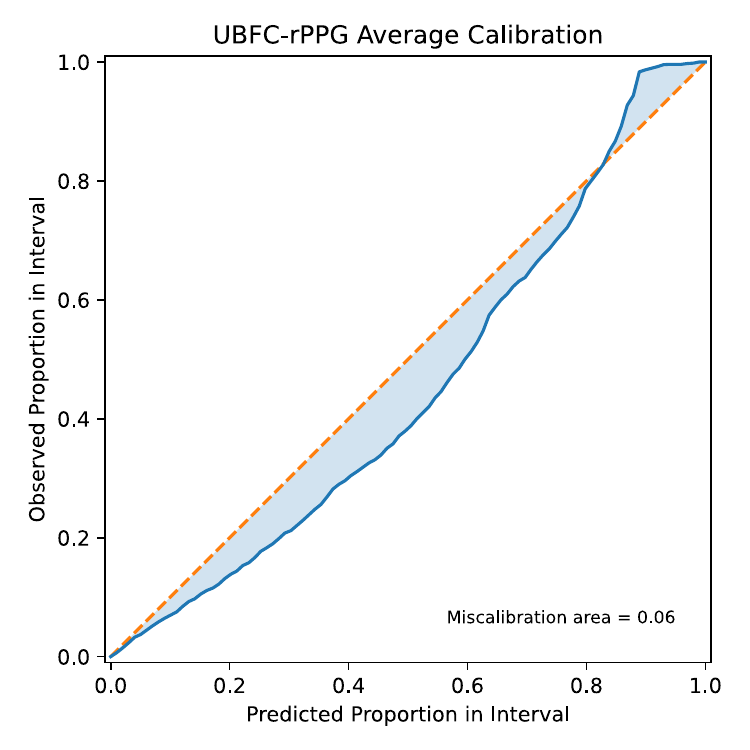}
        \label{fig:ubfc_calib}
    \end{subfigure}
    \caption{\VS{Calibration curves for PURE and UBFC-rPPG datasets.}}
    \label{fig:calibration_curves}
    \vspace{-2.0em}
\end{wrapfigure}

\VS{Next, we measure uncertainty calibration~\cite{kuleshov2018accurate}, which intuitively means that when a model assigns a probability $p$ to an event, then that event should happen $100*p$ percent empirically. We plot the calibration curves in Figure~\ref{fig:calibration_curves}, which plots the predicted proportion of samples in an interval against the actual observed proportion: perfect calibration corresponds to the diagonal. For the PURE dataset, we observe that we achieve good calibration, but our model under-predicts the observed proportion, leading to a miscalibration error of 0.06. On the UBFC-rPPG dataset, we also observe good calibration performance; our model is slightly overconfident at lower observed proportions, while it is slightly overconfident at higher observed proportions. On our average, however, the models on both datasets do well.}

\subsubsection{Uncertainty Quantification on Protected Attributes}

\begin{table}[]
    \caption{\VS{Proper Scoring Rule metrics on Protected Attributes of MMSE-HR dataset}}
    \centering
    \resizebox{0.65\textwidth}{!}{
    \begin{tabular}{c|cc|ccccc}
    \toprule
     & \multicolumn{2}{c}{HR Accuracy} & \multicolumn{5}{c}{Uncertainty Metrics} \\
    Protected Attributed & \textbf{HR MAE}$\downarrow$ & \textbf{HR RMSE}$\downarrow$ & \textbf{NLL}$\downarrow$ & \textbf{Sharpness}$\uparrow$& \textbf{CRPS} $\downarrow$ &  \textbf{Check Score}$\downarrow$ & \textbf{Interval Score}$\downarrow$\\
    \hline
    Light Skin Tone & \textbf{2.17} & \textbf{3.74} & \textbf{-3.430} &  \textbf{0.203} & 0.042 & 0.021 & 0.209 \\
    Dark Skin Tone & 4.79 & 6.84 & -2.828 & 0.191 & \textbf{0.035} & \textbf{0.018} & \textbf{0.188} \\
    \hline
    Men & \textbf{2.07} & 3.83 & \textbf{-3.435} & 0.196 & \textbf{0.032} & \textbf{0.016} & \textbf{0.167} \\
    Women & 2.23 & \textbf{3.69} & -3.415 & \textbf{0.203} & 0.039 & 0.020 & 0.195 \\
    \bottomrule
    \end{tabular}}
    \label{tab:score_rules_protected}
\end{table}

\VS{In addition to calibration metrics over the entire dataset, we measure calibration metrics over  protected attributes i.e. skin color and gender. We establish the first baselines for uncertainty quantification and calibration on protected attributes on the MMSE-HR dataset in Table~\ref{tab:score_rules_protected}. Before we establish the UQ metrics, we enumerate the pulse rate estimation metrics in Table~\ref{tab:score_rules_protected}. Our results on people with darker skin tones is consistent with that of previous work: pulse rate estimation is superior for those people who are better represented in the dataset. We notice better UQ performance for men, while the results between light and dark skin tones seem to be mixed; we notice that three UQ metrics indicate better results for darker skinned people that lighter skinned people. We plot calibration curves for all four subgroups in Appendix Figure~\ref{fig:four_figures}; we notice that our miscalibration error for all four groups is 0.10 or better, showing the effectiveness of our algorithm.}


\section{Conclusion}

To be truly accepted by clinicians and medical personnel, machine learning algorithms for healthcare should be ``repeatedly successful in prognosticating patient’s condition in [the doctor’s] personal experience''~\cite{tonekaboni2019clinicians}. While previous iPPG algorithms output point estimates of the pulse signal, we introduce the first posterior sampling method for iPPG that repeatedly samples likely pulse signal estimates given camera measurements, permitting an uncertainty analysis that can help doctors make better decisions. We achieve this by modeling a stochastic process between camera measurements and pulse signals, and learn the flow and score of this process to build the drift coefficient of an SDE. We improved results by temporally regularizing the flow, and show that this helps us capture the modes of the signal distribution. While we achieve strong results on intra-dataset evaluation, future work should address domain shifts between training and testing datasets, as well as addressing performance gaps on protected subpopulations.


\bibliography{main}

@book{nichols2022mcdonald,
  title={McDonald’s blood flow in arteries: theoretical, experimental and clinical principles},
  author={Nichols, Wilmer W and O'Rourke, Michael and Edelman, Elazer R and Vlachopoulos, Charalambos},
  year={2022},
  publisher={CRC press}
}

@article{begoli2019need,
  title={The need for uncertainty quantification in machine-assisted medical decision making},
  author={Begoli, Edmon and Bhattacharya, Tanmoy and Kusnezov, Dimitri},
  journal={Nature Machine Intelligence},
  volume={1},
  number={1},
  pages={20--23},
  year={2019},
  publisher={Nature Publishing Group UK London}
}

@inproceedings{tonekaboni2019clinicians,
  title={What clinicians want: contextualizing explainable machine learning for clinical end use},
  author={Tonekaboni, Sana and Joshi, Shalmali and McCradden, Melissa D and Goldenberg, Anna},
  booktitle={Machine learning for healthcare conference},
  pages={359--380},
  year={2019},
  organization={PMLR}
}

@article{gneiting2007strictly,
  title={Strictly proper scoring rules, prediction, and estimation},
  author={Gneiting, Tilmann and Raftery, Adrian E},
  journal={Journal of the American statistical Association},
  volume={102},
  number={477},
  pages={359--378},
  year={2007},
  publisher={Taylor \& Francis}
}

@article{chung2021uncertainty,
  title={Uncertainty Toolbox: an Open-Source Library for Assessing, Visualizing, and Improving Uncertainty Quantification},
  author={Chung, Youngseog and Char, Ian and Guo, Han and Schneider, Jeff and Neiswanger, Willie},
  journal={arXiv preprint arXiv:2109.10254},
  year={2021}
}

@article{tran2020methods,
  title={Methods for comparing uncertainty quantifications for material property predictions},
  author={Tran, Kevin and Neiswanger, Willie and Yoon, Junwoong and Zhang, Qingyang and Xing, Eric and Ulissi, Zachary W},
  journal={Machine Learning: Science and Technology},
  volume={1},
  number={2},
  pages={025006},
  year={2020},
  publisher={IOP Publishing}
}

@article{steinwart2011estimating,
  title={Estimating conditional quantiles with the help of the pinball loss},
  author={Steinwart, Ingo and Christmann, Andreas},
  year={2011}
}

@ARTICLE{wangClinicalinfant,

  author={Huang, Dongmin and Zeng, Yongshen and Zhu, Yingen and Song, Xiaoyan and Pan, Liping and Yang, Jie and Wang, Yanrong and Lu, Hongzhou and Wang, Wenjin},

  journal={IEEE Journal of Biomedical and Health Informatics}, 

  title={Camera-Based Respiratory Imaging System for Monitoring Infant Thoracoabdominal Patterns of Respiration}, 

  year={2024},

  volume={},

  number={},

  pages={1-14},

  doi={10.1109/JBHI.2024.3482569}}

@inproceedings{shenoy2023unrolled,
  title={Unrolled ippg: Video heart rate estimation via unrolling proximal gradient descent},
  author={Shenoy, Vineet R and Marks, Tim K and Mansour, Hassan and Lohit, Suhas},
  booktitle={2023 IEEE International Conference on Image Processing (ICIP)},
  pages={2715--2719},
  year={2023},
  organization={IEEE}
}

@article{nowara2020near,
  title={Near-infrared imaging photoplethysmography during driving},
  author={Nowara, Ewa M and Marks, Tim K and Mansour, Hassan and Veeraraghavan, Ashok},
  journal={IEEE transactions on intelligent transportation systems},
  volume={23},
  number={4},
  pages={3589--3600},
  year={2020},
  publisher={IEEE}
}

@article{sun2024provable,
  title={Provable probabilistic imaging using score-based generative priors},
  author={Sun, Yu and Wu, Zihui and Chen, Yifan and Feng, Berthy T and Bouman, Katherine L},
  journal={IEEE Transactions on Computational Imaging},
  year={2024},
  publisher={IEEE}
}

@inproceedings{kuleshov2018accurate,
  title={Accurate uncertainties for deep learning using calibrated regression},
  author={Kuleshov, Volodymyr and Fenner, Nathan and Ermon, Stefano},
  booktitle={International conference on machine learning},
  pages={2796--2804},
  year={2018},
  organization={PMLR}
}

@article{Stricker2014NoncontactVP,
  title={Non-contact video-based pulse rate measurement on a mobile service robot},
  author={Ronny Stricker and Steffen M{\"u}ller and Horst-Michael Gro{\ss}},
  journal={The 23rd IEEE Inter. Symposium on Robot and Human Interactive Communication},
  year={2014},
  url={https://api.semanticscholar.org/CorpusID:8529212}
}

@inproceedings{vspetlik2018visual,
  title={Visual heart rate estimation with convolutional neural network},
  author={{\v{S}}petl{\'\i}k, Radim and Franc, Vojtech and Matas, Jir{\'\i}},
  booktitle={Proceedings of the british machine vision conference, Newcastle, UK},
  pages={3--6},
  year={2018}
}

@inproceedings{zhang2016multimodal,
  title={Multimodal spontaneous emotion corpus for human behavior analysis},
  author={Zhang, Zheng and Girard, Jeff M and Wu, Yue and Zhang, Xing and Liu, Peng and Ciftci, Umur and Canavan, Shaun and Reale, Michael and Horowitz, Andy and Yang, Huiyuan and others},
  booktitle={Proceedings of the IEEE conference on computer vision and pattern recognition},
  pages={3438--3446},
  year={2016}
}

@article{ubfc-rppg,
title = {Unsupervised skin tissue segmentation for remote photoplethysmography},
journal = {Pattern Recognition Letters},
volume = {124},
year = {2019},
note = {Award Winning Papers from the 23rd Inter. Conf. on Pattern Recognition (ICPR)},
issn = {0167-8655},
doi = {https://doi.org/10.1016/j.patrec.2017.10.017},
url = {https://www.sciencedirect.com/science/article/pii/S0167865517303860},
author = {Serge Bobbia and Richard Macwan and Yannick Benezeth and Alamin Mansouri and Julien Dubois}
}

@inproceedings{ertugrul2019cross,
  title={Cross-domain au detection: Domains, learning approaches, and measures},
  author={Ertugrul, Itir Onal and Cohn, Jeffrey F and Jeni, L{\'a}szl{\'o} A and Zhang, Zheng and Yin, Lijun and Ji, Qiang},
  booktitle={2019 14th IEEE international conference on automatic face \& gesture recognition (FG 2019)},
  pages={1--8},
  year={2019},
  organization={IEEE}
}

@inproceedings{nowara2021benefit,
  title={The benefit of distraction: Denoising camera-based physiological measurements using inverse attention},
  author={Nowara, Ewa M and McDuff, Daniel and Veeraraghavan, Ashok},
  booktitle={Proceedings of the IEEE/CVF international conference on computer vision},
  pages={4955--4964},
  year={2021}
}

@inproceedings{gideon2021way,
  title={The way to my heart is through contrastive learning: Remote photoplethysmography from unlabelled video},
  author={Gideon, John and Stent, Simon},
  booktitle={Proceedings of the IEEE/CVF international conference on computer vision},
  pages={3995--4004},
  year={2021}
}

@article{yue2023facial,
  title={Facial video-based remote physiological measurement via self-supervised learning},
  author={Yue, Zijie and Shi, Miaojing and Ding, Shuai},
  journal={IEEE Transactions on Pattern Analysis and Machine Intelligence},
  year={2023},
  publisher={IEEE}
}

@article{sun2024contrast,
  title={Contrast-phys+: Unsupervised and weakly-supervised video-based remote physiological measurement via spatiotemporal contrast},
  author={Sun, Zhaodong and Li, Xiaobai},
  journal={IEEE Transactions on Pattern Analysis and Machine Intelligence},
  year={2024},
  publisher={IEEE}
}

@inproceedings{speth2023non,
  title={Non-contrastive unsupervised learning of physiological signals from video},
  author={Speth, Jeremy and Vance, Nathan and Flynn, Patrick and Czajka, Adam},
  booktitle={Proceedings of the IEEE/CVF Conference on Computer Vision and Pattern Recognition},
  pages={14464--14474},
  year={2023}
}

@article{liu2024rppgmae,
  title={rPPG-MAE: Self-supervised pretraining with masked autoencoders for remote physiological measurements},
  author={Liu, Xin and Zhang, Yuting and Yu, Zitong and Lu, Hao and Yue, Huanjing and Yang, inJgyu},
  journal={IEEE Transactions on Multimedia},
  year={2024},
  publisher={IEEE}
}

@article{poh2010non,
  title={Non-contact, automated cardiac pulse measurements using video imaging and blind source separation.},
  author={Poh, Ming-Zher and McDuff, Daniel J and Picard, Rosalind W},
  journal={Optics express},
  volume={18},
  number={10},
  pages={10762--10774},
  year={2010},
  publisher={Optica Publishing Group}
}

@article{lewandowska2012measuring,
  title={Measuring pulse rate with a webcam},
  author={Lewandowska, Magdalena and Nowak, J{\k{e}}drzej},
  journal={Journal of Medical Imaging and Health Informatics},
  volume={2},
  number={1},
  pages={87--92},
  year={2012},
  publisher={American Scientific Publishers}
}

@article{de2013robust,
  title={Robust pulse rate from chrominance-based rPPG},
  author={De Haan, Gerard and Jeanne, Vincent},
  journal={IEEE transactions on biomedical engineering},
  volume={60},
  number={10},
  pages={2878--2886},
  year={2013},
  publisher={IEEE}
}

@article{de2014improved,
  title={Improved motion robustness of remote-PPG by using the blood volume pulse signature},
  author={De Haan, Gerard and Van Leest, Arno},
  journal={Physiological measurement},
  volume={35},
  number={9},
  pages={1913},
  year={2014},
  publisher={IOP Publishing}
}

@misc{albergo2023stochasticinterpolantsunifyingframework,
      title={Stochastic Interpolants: A Unifying Framework for Flows and Diffusions}, 
      author={Michael S. Albergo and Nicholas M. Boffi and Eric Vanden-Eijnden},
      year={2023},
      eprint={2303.08797},
      archivePrefix={arXiv},
      primaryClass={cs.LG},
      url={https://arxiv.org/abs/2303.08797}, 
}

@misc{zhang2025trajectoryflowmatchingapplications,
      title={Trajectory Flow Matching with Applications to Clinical Time Series Modeling}, 
      author={Xi Zhang and Yuan Pu and Yuki Kawamura and Andrew Loza and Yoshua Bengio and Dennis L. Shung and Alexander Tong},
      year={2025},
      eprint={2410.21154},
      archivePrefix={arXiv},
      primaryClass={cs.LG},
      url={https://arxiv.org/abs/2410.21154}, 
}

@misc{hoellmer2025openmaterialsgenerationstochastic,
      title={Open Materials Generation with Stochastic Interpolants}, 
      author={Philipp Hoellmer and Thomas Egg and Maya M. Martirossyan and Eric Fuemmeler and Amit Gupta and Zeren Shui and Pawan Prakash and Adrian Roitberg and Mingjie Liu and George Karypis and Mark Transtrum and Richard G. Hennig and Ellad B. Tadmor and Stefano Martiniani},
      year={2025},
      eprint={2502.02582},
      archivePrefix={arXiv},
      primaryClass={cs.LG},
      url={https://arxiv.org/abs/2502.02582}, 
}

@InProceedings{dualganLu2021,
    author    = {Lu, Hao and Han, Hu and Zhou, S. Kevin},
    title     = {Dual-GAN: Joint BVP and Noise Modeling for Remote Physiological Measurement},
    booktitle = {Proceedings of the IEEE/CVF Conference on Computer Vision and Pattern Recognition (CVPR)},
    month     = {June},
    year      = {2021},
    pages     = {12404-12413}
}

@ARTICLE{self_supervised_rppg,
  author={Yue, Zijie and Shi, Miaojing and Ding, Shuai},
  journal={IEEE Transactions on Pattern Analysis and Machine Intelligence}, 
  title={Facial Video-Based Remote Physiological Measurement via Self-Supervised Learning}, 
  year={2023},
  volume={45},
  number={11},
  doi={10.1109/TPAMI.2023.3298650}}

@InProceedings{chen2018deepphys,
author = {Chen, Weixuan and McDuff, Daniel},
title = {DeepPhys: Video-Based Physiological Measurement Using Convolutional Attention Networks},
booktitle = {Proceedings of the European Conference on Computer Vision (ECCV)},
month = {September},
year = {2018}
}

@inproceedings{Spetlik2018VisualHR,
  title={Visual Heart Rate Estimation with Convolutional Neural Network},
  author={Radim Spetlik and Vojtech Franc and Jan Cech and Jiri Matas},
  booktitle={British Machine Vision Conference},
  year={2018},
  url={https://api.semanticscholar.org/CorpusID:52219725}
}

@INPROCEEDINGS{synrhythm,
  author={Niu, Xuesong and Han, Hu and Shan, Shiguang and Chen, Xilin},
  booktitle={2018 24th International Conference on Pattern Recognition (ICPR)}, 
  title={SynRhythm: Learning a Deep Heart Rate Estimator from General to Specific}, 
  year={2018},
  volume={},
  number={},
  doi={10.1109/ICPR.2018.8546321}}

@article{wang2016algorithmic,
  title={Algorithmic principles of remote PPG},
  author={Wang, Wenjin and Den Brinker, Albertus C and Stuijk, Sander and De Haan, Gerard},
  journal={IEEE Transactions on Biomedical Engineering},
  volume={64},
  number={7},
  year={2016},
  publisher={IEEE}
}

@inproceedings{niu2020video,
  title={Video-based remote physiological measurement via cross-verified feature disentangling},
  author={Niu, Xuesong and Yu, Zitong and Han, Hu and Li, Xiaobai and Shan, Shiguang and Zhao, Guoying},
  booktitle={Computer Vision--ECCV 2020: 16th European Conference, Glasgow, UK, August 23--28, 2020, Proceedings, Part II 16},
  year={2020},
  organization={Springer}
}

@ARTICLE{SongPulseGAN,
  author={Song, Rencheng and Chen, Huan and Cheng, Juan and Li, Chang and Liu, Yu and Chen, Xun},
  journal={IEEE Journal of Biomedical and Health Informatics}, 
  title={PulseGAN: Learning to Generate Realistic Pulse Waveforms in Remote Photoplethysmography}, 
  year={2021},
  volume={25},
  number={5},
  doi={10.1109/JBHI.2021.3051176}}

@InProceedings{yu2022physformer,
    author    = {Yu, Zitong and Shen, Yuming and Shi, Jingang and Zhao, Hengshuang and Torr, Philip H.S. and Zhao, Guoying},
    title     = {PhysFormer: Facial Video-Based Physiological Measurement With Temporal Difference Transformer},
    booktitle = {Proceedings of the IEEE/CVF Conference on Computer Vision and Pattern Recognition (CVPR)},
    month     = {June},
    year      = {2022},
    pages     = {4186-4196}
}

@InProceedings{liu2022federated,
    author    = {Liu, Xin and Zhang, Mingchuan and Jiang, Ziheng and Patel, Shwetak and McDuff, Daniel},
    title     = {Federated Remote Physiological Measurement With Imperfect Data},
    booktitle = {Proceedings of the IEEE/CVF Conference on Computer Vision and Pattern Recognition (CVPR) Workshops},
    month     = {June},
    year      = {2022},
    pages     = {2155-2164}
}

@InProceedings{liu2023efficientphys,
    author    = {Liu, Xin and Hill, Brian and Jiang, Ziheng and Patel, Shwetak and McDuff, Daniel},
    title     = {EfficientPhys: Enabling Simple, Fast and Accurate Camera-Based Cardiac Measurement},
    booktitle = {Proceedings of the IEEE/CVF Winter Conference on Applications of Computer Vision (WACV)},
    month     = {January},
    year      = {2023},
    pages     = {5008-5017}
}

@inproceedings{wu2000photoplethysmography,
  title={Photoplethysmography imaging: a new noninvasive and noncontact method for mapping of the dermal perfusion changes},
  author={Wu, Ting and Blazek, Vladimir and Schmitt, Hans Juergen},
  booktitle={Optical techniques and instrumentation for the measurement of blood composition, structure, and dynamics},
  volume={4163},
  pages={62--70},
  year={2000},
  organization={SPIE}
}

@InProceedings{tulyakov2016self,
author = {Tulyakov, Sergey and Alameda-Pineda, Xavier and Ricci, Elisa and Yin, Lijun and Cohn, Jeffrey F. and Sebe, Nicu},
title = {Self-Adaptive Matrix Completion for Heart Rate Estimation From Face Videos Under Realistic Conditions},
booktitle = {Proceedings of the IEEE Conference on Computer Vision and Pattern Recognition (CVPR)},
month = {June},
year = {2016}
}

@article{liu2020multi,
  title={Multi-task temporal shift attention networks for on-device contactless vitals measurement},
  author={Liu, Xin and Fromm, Josh and Patel, Shwetak and McDuff, Daniel},
  journal={Advances in Neural Information Processing Systems},
  volume={33},
  pages={19400--19411},
  year={2020}
}

@article{yu2019remote,
  title={Remote photoplethysmograph signal measurement from facial videos using spatio-temporal networks},
  author={Yu, Zitong and Li, Xiaobai and Zhao, Guoying},
  journal={arXiv preprint arXiv:1905.02419},
  year={2019}
}

@inproceedings{liu2021metaphys,
  title={MetaPhys: few-shot adaptation for non-contact physiological measurement},
  author={Liu, Xin and Jiang, Ziheng and Fromm, Josh and Xu, Xuhai and Patel, Shwetak and McDuff, Daniel},
  booktitle={Proceedings of the conference on health, inference, and learning},
  pages={154--163},
  year={2021}
}

@article{liu2023robust,
  title={Robust remote photoplethysmography estimation with environmental noise disentanglement},
  author={Liu, Si-Qi and Yuen, Pong C},
  journal={IEEE Transactions on Image Processing},
  volume={33},
  pages={27--41},
  year={2023},
  publisher={IEEE}
}

@article{chen2000gaussianization,
  title={Gaussianization},
  author={Chen, Scott and Gopinath, Ramesh},
  journal={Advances in neural information processing systems},
  volume={13},
  year={2000}
}

@article{tabak2010density,
author = {Esteban G. Tabak and Eric Vanden-Eijnden},
title = {{Density estimation by dual ascent of the log-likelihood}},
volume = {8},
journal = {Communications in Mathematical Sciences},
number = {1},
publisher = {International Press of Boston},
pages = {217 -- 233},
year = {2010},
}

@article{chen2018neural,
  title={Neural ordinary differential equations},
  author={Chen, Ricky TQ and Rubanova, Yulia and Bettencourt, Jesse and Duvenaud, David K},
  journal={Advances in neural information processing systems},
  volume={31},
  year={2018}
}

@inproceedings{grathwohl2019scalable,
  title={Scalable reversible generative models with free-form continuous dynamics},
  author={Grathwohl, Will and Chen, Ricky TQ and Bettencourt, Jesse and Duvenaud, David},
  booktitle={International Conference on Learning Representations},
  volume={3},
  year={2019}
}

@inproceedings{finlay2020train,
  title={How to train your neural ODE: the world of Jacobian and kinetic regularization},
  author={Finlay, Chris and Jacobsen, J{\"o}rn-Henrik and Nurbekyan, Levon and Oberman, Adam},
  booktitle={International conference on machine learning},
  pages={3154--3164},
  year={2020},
  organization={PMLR}
}

@inproceedings{onken2021ot,
  title={Ot-flow: Fast and accurate continuous normalizing flows via optimal transport},
  author={Onken, Derek and Fung, Samy Wu and Li, Xingjian and Ruthotto, Lars},
  booktitle={Proceedings of the AAAI Conference on Artificial Intelligence},
  volume={35},
  number={10},
  pages={9223--9232},
  year={2021}
}

@inproceedings{tong2020trajectorynet,
  title={Trajectorynet: A dynamic optimal transport network for modeling cellular dynamics},
  author={Tong, Alexander and Huang, Jessie and Wolf, Guy and Van Dijk, David and Krishnaswamy, Smita},
  booktitle={International conference on machine learning},
  pages={9526--9536},
  year={2020},
  organization={PMLR}
}

@article{hyvarinen2005estimation,
  title={Estimation of non-normalized statistical models by score matching.},
  author={Hyv{\"a}rinen, Aapo and Dayan, Peter},
  journal={Journal of Machine Learning Research},
  volume={6},
  number={4},
  year={2005}
}

@article{vincent2011connection,
  title={A connection between score matching and denoising autoencoders},
  author={Vincent, Pascal},
  journal={Neural computation},
  volume={23},
  number={7},
  pages={1661--1674},
  year={2011},
  publisher={MIT Press}
}

@article{song2019generative,
  title={Generative modeling by estimating gradients of the data distribution},
  author={Song, Yang and Ermon, Stefano},
  journal={Advances in neural information processing systems},
  volume={32},
  year={2019}
}

@article{albergo2022building,
  title={Building normalizing flows with stochastic interpolants},
  author={Albergo, Michael S and Vanden-Eijnden, Eric},
  journal={arXiv preprint arXiv:2209.15571},
  year={2022}
}

@article{lipman2022flow,
  title={Flow matching for generative modeling},
  author={Lipman, Yaron and Chen, Ricky TQ and Ben-Hamu, Heli and Nickel, Maximilian and Le, Matt},
  journal={arXiv preprint arXiv:2210.02747},
  year={2022}
}

@article{tong2023conditional,
  title={Conditional flow matching: Simulation-free dynamic optimal transport},
  author={Tong, Alexander and Malkin, Nikolay and Huguet, Guillaume and Zhang, Yanlei and Rector-Brooks, Jarrid and Fatras, Kilian and Wolf, Guy and Bengio, Yoshua},
  journal={arXiv preprint arXiv:2302.00482},
  volume={2},
  number={3},
  year={2023}
}

@techreport{amos2016openface,
  title={OpenFace: A general-purpose face recognition
    library with mobile applications},
  author={Amos, Brandon and Bartosz Ludwiczuk and Satyanarayanan, Mahadev},
  year={2016},
  institution={CMU-CS-16-118, CMU School of Computer Science},
}

@InProceedings{Micaelli_2023_CVPR,
    author    = {Micaelli, Paul and Vahdat, Arash and Yin, Hongxu and Kautz, Jan and Molchanov, Pavlo},
    title     = {Recurrence Without Recurrence: Stable Video Landmark Detection With Deep Equilibrium Models},
    booktitle = {Proceedings of the IEEE/CVF Conference on Computer Vision and Pattern Recognition (CVPR)},
    month     = {June},
    year      = {2023},
    pages     = {22814-22825}
}

@article{dhariwal2021diffusion,
  title={Diffusion models beat gans on image synthesis},
  author={Dhariwal, Prafulla and Nichol, Alexander},
  journal={Advances in neural information processing systems},
  volume={34},
  pages={8780--8794},
  year={2021}
}

@article{kingma2014adam,
  title={Adam: A method for stochastic optimization},
  author={Kingma, Diederik P and Ba, Jimmy},
  journal={arXiv preprint arXiv:1412.6980},
  year={2014}
}

@article{li2020scalable,
  title={Scalable gradients for stochastic differential equations},
  author={Li, Xuechen and Wong, Ting-Kam Leonard and Chen, Ricky T. Q. and Duvenaud, David},
  journal={International Conference on Artificial Intelligence and Statistics},
  year={2020}
}

@article{kidger2021neuralsde,
  title={Neural {SDE}s as {I}nfinite-{D}imensional {GAN}s},
  author={Kidger, Patrick and Foster, James and Li, Xuechen and Oberhauser, Harald and Lyons, Terry},
  journal={International Conference on Machine Learning},
  year={2021}
}

@Article{00006534-990000000-02657,
author={Shenoy, Vineet R.
and Kingston, Carly Q.
and Singh, Mantej
and Fleming, Ike C.
and Durr, Nicholas J.
and Chellappa, Rama
and Giladi, Aviram M.},
title={``Perfusion Assessment of Healthy and Injured Hands Using Video-Based Deep Learning Models''},
journal={Plastic and Reconstructive Surgery},
year={2025},
issn={0032-1052},
url={https://journals.lww.com/plasreconsurg/fulltext/9900/_perfusion_assessment_of_healthy_and_injured_hands.2657.aspx}
}

@software{Falcon_PyTorch_Lightning_2019,
author = {Falcon, William and {The PyTorch Lightning team}},
doi = {10.5281/zenodo.3828935},
license = {Apache-2.0},
month = mar,
title = {{PyTorch Lightning}},
url = {https://github.com/Lightning-AI/lightning},
version = {1.4},
year = {2019}
}

@article{alian2014photoplethysmography,
  title={Photoplethysmography},
  author={Alian, Aymen A and Shelley, Kirk H},
  journal={Best Practice \& Research Clinical Anaesthesiology},
  volume={28},
  number={4},
  pages={395--406},
  year={2014},
  publisher={Elsevier}
}

@article{liu2022toolbox,
  title={rPPG-Toolbox: Deep Remote PPG Toolbox},
  author={Liu, Xin and Narayanswamy, Girish and Paruchuri, Akshay and Zhang, Xiaoyu and Tang, Jiankai and Zhang, Yuzhe and Wang, Yuntao and Sengupta, Soumyadip and Patel, Shwetak and McDuff, Daniel},
  journal={arXiv preprint arXiv:2210.00716},
  year={2022}
}

@article{mcduff2023camera,
  title={Camera measurement of physiological vital signs},
  author={McDuff, Daniel},
  journal={ACM Computing Surveys},
  volume={55},
  number={9},
  pages={1--40},
  year={2023},
  publisher={ACM New York, NY}
}

@article{cohen2009pearson,
  title={Pearson correlation coefficient},
  author={Cohen, Israel and Huang, Yiteng and Chen, Jingdong and Benesty, Jacob and Benesty, Jacob and Chen, Jingdong and Huang, Yiteng and Cohen, Israel},
  journal={Noise reduction in speech processing},
  pages={1--4},
  year={2009},
  publisher={Springer}
}

@article{arridge2023inverse,
  title={Inverse problems with learned forward operators},
  author={Arridge, Simon and Hauptmann, Andreas and Korolev, Yury},
  journal={arXiv preprint arXiv:2311.12528},
  year={2023}
}

@article{Lunz2020OnLO,
  title={On Learned Operator Correction in Inverse Problems},
  author={Sebastian Lunz and Andreas Hauptmann and Tanja Tarvainen and Carola-Bibiane Sch{\"o}nlieb and Simon Robert Arridge},
  journal={SIAM journal on imaging sciences},
  year={2020},
  volume={14},
  pages={92 - 127},
  url={https://api.semanticscholar.org/CorpusID:224899525}
}

@article{chen2016variational,
  title={Variational lossy autoencoder},
  author={Chen, Xi and Kingma, Diederik P and Salimans, Tim and Duan, Yan and Dhariwal, Prafulla and Schulman, John and Sutskever, Ilya and Abbeel, Pieter},
  journal={arXiv preprint arXiv:1611.02731},
  year={2016}
}

@inproceedings{bowman2016generating,
  title={Generating sentences from a continuous space},
  author={Bowman, Samuel and Vilnis, Luke and Vinyals, Oriol and Dai, Andrew and Jozefowicz, Rafal and Bengio, Samy},
  booktitle={Proceedings of the 20th SIGNLL conference on computational natural language learning},
  pages={10--21},
  year={2016}
}

@inproceedings{serban2017hierarchical,
  title={A hierarchical latent variable encoder-decoder model for generating dialogues},
  author={Serban, Iulian and Sordoni, Alessandro and Lowe, Ryan and Charlin, Laurent and Pineau, Joelle and Courville, Aaron and Bengio, Yoshua},
  booktitle={Proceedings of the AAAI conference on artificial intelligence},
  volume={31},
  number={1},
  year={2017}
}

@inproceedings{xu2024assessing,
  title={Assessing sample quality via the latent space of generative models},
  author={Xu, Jingyi and Le, Hieu and Samaras, Dimitris},
  booktitle={European Conference on Computer Vision},
  pages={449--464},
  year={2024},
  organization={Springer}
}

@inproceedings{dai2020usual,
  title={The usual suspects? Reassessing blame for VAE posterior collapse},
  author={Dai, Bin and Wang, Ziyu and Wipf, David},
  booktitle={International conference on machine learning},
  pages={2313--2322},
  year={2020},
  organization={PMLR}
}

@article{lucas2019don,
  title={Don't blame the elbo! a linear vae perspective on posterior collapse},
  author={Lucas, James and Tucker, George and Grosse, Roger B and Norouzi, Mohammad},
  journal={Advances in Neural Information Processing Systems},
  volume={32},
  year={2019}
}

@article{silvestro2020prior,
  title={Prior choice affects ability of Bayesian neural networks to identify unknowns},
  author={Silvestro, Daniele and Andermann, Tobias},
  journal={arXiv preprint arXiv:2005.04987},
  year={2020}
}

@phdthesis{rasmussen1997evaluation,
  title={Evaluation of Gaussian processes and other methods for non-linear regression},
  author={Rasmussen, Carl Edward},
  year={1997},
  school={University of Toronto Toronto, Canada}
}

@inproceedings{sun2018differentiable,
  title={Differentiable compositional kernel learning for Gaussian processes},
  author={Sun, Shengyang and Zhang, Guodong and Wang, Chaoqi and Zeng, Wenyuan and Li, Jiaman and Grosse, Roger},
  booktitle={International Conference on Machine Learning},
  pages={4828--4837},
  year={2018},
  organization={PMLR}
}

@article{eriksson2018scaling,
  title={Scaling Gaussian process regression with derivatives},
  author={Eriksson, David and Dong, Kun and Lee, Eric and Bindel, David and Wilson, Andrew G},
  journal={Advances in neural information processing systems},
  volume={31},
  year={2018}
}

@article{sohn2015learning,
  title={Learning structured output representation using deep conditional generative models},
  author={Sohn, Kihyuk and Lee, Honglak and Yan, Xinchen},
  journal={Advances in neural information processing systems},
  volume={28},
  year={2015}
}

@article{jospin2022hands,
  title={Hands-on Bayesian neural networks—A tutorial for deep learning users},
  author={Jospin, Laurent Valentin and Laga, Hamid and Boussaid, Farid and Buntine, Wray and Bennamoun, Mohammed},
  journal={IEEE Computational Intelligence Magazine},
  volume={17},
  number={2},
  pages={29--48},
  year={2022},
  publisher={IEEE}
}
\bibliographystyle{tmlr}

\appendix
\section{Appendix}

\subsection{Implementation Details: Hyperparameters}\label{sec:appendix_hyperparameters}

As mentioned in the main paper, we implement our code in PyTorch using the PyTorch Lightning~\cite{Falcon_PyTorch_Lightning_2019} library. Our models are identical learnable UNets from~\cite{dhariwal2021diffusion}. The hyperparameters used to train our models are in Table~\ref{tab:unet_parameters}.

\begin{table}[]
    \centering
    \caption{Hyperparameters used to train our model}
    \begin{tabular}{lccc}
        \toprule
        Parameter & MMSE-HR & PURE & UBFC-rPPG \\
        \midrule
        Passband Frequency (bpm) & 42 & 42 & 42  \\
        Cutoff Frequency (bpm) & 150 & 150 & 150  \\
        Num Taps & 5 & 5 & 5 \\
        Frame Stride (sec) & 0.4 & 0.4 & 0.4 \\
        Frame Stride Test (sec) & 10 & 10 & 10\\
        FPS & 25 & 30 & 30\\
        Signal Length & 250 & 300 & 300 \\
        Num Res Block & 1 & 1 & 1 \\
        Attention Resolution & [2, 4] & [2, 4] & [2, 4] \\
        \hline
        Learning Rate & 1e-3 & 1e-3 & 1e-3 \\
        Weight Decay & 0 & 0 & 0 \\
        Dropout & 0 & 0 & 0 \\
        Epochs & 10 & 15 & 15 \\
        \bottomrule
    \end{tabular}
    
    \label{tab:unet_parameters}
\end{table}

\subsection{Evaluation Metrics}\label{sec:appendix_eval_mets}

We provide a more complete discussion of the evaluation metrics below. Let $y_{i,j}$ be the ground-truth, $\mu_{i, j}$ be the mean, and $\sigma_{i, j}$ of samples generated from test sample $i$ in frequency bin $j$.

\begin{itemize}
    \item \textbf{Mean Absolute Error (MAE): } \VS{A measure of the difference between a predicted quantity $\hat{R}_i$ and the ground-truth quantity $R_i$. This quantity is less sensitive to outliers than the RMSE. We use this metric to measure the pulse rate estimation error as well as the spectrum estimation error. Lower values indicate better performance, and the lowest possible value is 0.}
    \begin{equation}
    \text{MAE } = \frac{1}{N} \sum_{i=1}^N |R_i - \hat{R}_i |
\end{equation}
    \item \textbf{Root Mean Squared Error (RMSE):} \VS{A measure of the difference between a predicted quantity $\hat{R}_i$ and the ground-truth quantity $R_i$, which weights outliers heavily. We use this metric to measure the pulse rate estimation error as well as the spectrum estimation error. Lower values indicate better performance, and the lowest possible value is 0.}
    \begin{equation}
        \text{ RMSE } = \sqrt{\frac{1}{N} \sum_{i=1}^N (R_i - \hat{R}_i)^2}
    \end{equation}
    \item \textbf{Pearson Correlation Coefficient (PCC)}: \VS{This quantity measure the linear relationship and strength of the linear relationship between two variables. We use this metric to measure the pulse rate estimation error as well as the spectrum estimation error. Higher values indicate better performance, and the highest possible score is 1.}
    \begin{equation}
        \rho = \frac{\sum_{i=1}^N (R_i - \mu_{R_i})(\hat{R}_i - \mu_{\hat{R}_i})}{\sqrt{\sum_{i=1}^N (R_i - \mu_{R_i})^2(\hat{R}_i - \mu_{\hat{R}_i})^2}}
    \end{equation}

    \VS{where $\hat{R}_i$ is the predicted pulse rate, $R_i$ is the ground-truth pulse rate, $N$ is the number of time windows, and $\mu_{R_i}$ and $\mu_{\hat{R}_i}$ are the means of the predicted and ground-truth pulse rates, respectively.}
    
    \item \textbf{Negative Log Likelihood (NLL)}\cite{tran2020methods}: \VS{NLL provides an overall measure of both the predictive accuracy and quality of the predictive uncertainty quantification. Given the mean and variance of a set of predictions, construct a Gaussian distribution. Then measure the likelihood of the ground-truth given this distribution. Lower values are better, and can be negative.}
    \begin{equation}
        \text{NLL}(\mathbf{y}, \mathbf{\mu}, \mathbf{\sigma}) = - \frac{1}{N\cdot L} \sum_{i=1}^N\sum_{j=1}^L \text{ln} P\big(y_{i, j} | \mathcal{N}(\mu_{i,j}, \sigma_{i, j})\big)
    \end{equation}
    \item \textbf{Continuous Ranked Probability Score}~\cite{gneiting2007strictly}: \VS{This quantity measures how good the predicted distribution is compared to the ground-truth value, and is often used in meteorolgy. Lower values are better and the lowest possible value is zero. Under the assumption of the Gaussian,  the CRPS is defined as:}
    \begin{equation}
        \text{CRPS}(\mathbf{y}, \mathbf{\mu}, \mathbf{\sigma}) = \frac{1}{N\cdot L}\sum_{i=1}^N\sum_{j=1}^L \sigma_{i, j} \left[ \frac{1}{\sqrt{\pi}}  - 2\psi\Big(\frac{y_{i, j} - \mu_{i, j}}{\sigma_{i, j}}\Big) - \frac{y_{i, j} - \mu_{i, j}}{\sigma_{i, j}}\Big(2\Phi\big(\frac{y_{i, j} - \mu_{i, j}}{\sigma_{i, j}}\big) - 1\Big) \right]
    \end{equation}
    where $\psi$ and $\Phi$ denote the PDF and CDF of a standard gaussian.
    \item \textbf{Sharpness}\cite{tran2020methods}: \VS{A quantity that measures the concentration of the predicted distribution, independent of the ground-truth distribution. Given the CDF of a sample $i$ in bin $j$ the sharpness is defined as}
    \begin{equation}
        \text{Sharpness} = \frac{1}{N\cdot L} \sum_{i=1}^N\sum_{j=1}^L \text{var}(F_{i,j})
    \end{equation}
    \VS{Higher values indicate better performance.}
    \item \textbf{Check Score/Pinball Loss}\cite{steinwart2011estimating}: \VS{This metric measure the prediction of quantiles. This is a non-symmetric loss that penalizes larger quantiles more. Lower values are better}
    \begin{equation}
        \text{Check Score/Pinball Loss} =  \frac{1}{N\cdot L} \sum_{i=1}^N\sum_{j=1}^Q
        L(y_{i}, \hat{y}_i)=
        \begin{cases}
            (y_ - \hat{y}_i) \cdot \tau_j & \text{if } y \geq \hat{y} \\
            (y_ - \hat{y}_i)(1 - \tau_j) & \text{if } y < \hat{y}
        \end{cases}
    \end{equation}
    \VS{where $y_i$ is the true value, $\hat{y}_i$ is the predicted quantile and $\tau_j$ is the target quantile between 0 and 1. Lower values are better.}
    \item \textbf{Interval Score}~\cite{gneiting2007strictly}: \VS{Given a lower bound $L$ and upper bound $U$ that is intended to cover the true value $y_{i,j}$ with probability $1 - \alpha$, the interval score is defined as}

    \begin{equation}
        S_\alpha(L, U, y_{i,j}) = (U - L) + \frac{2}{\alpha}(L - y) \cdot \mathbf{1}(y < L) + \frac{2}{\alpha}(y - U) \cdot \mathbf{1}(y > U)
    \end{equation}

    \VS{This measures whether the true value falls within the interval, and whether the width of the interval is narrow. Smaller values are better.}
\end{itemize}

\subsection{Preliminary Investigation}\label{subsec:appendix_prelimanary}

\subsubsection{Inadequacy of the $A=F^{-1}$ forward model}\label{sec:appendix_inadequacy}

\VS{In the main manuscript, we implement PMC-iPPG to show the inadequacy of the $A=F^{-1}$ forward model of previous work. The work of ~\cite{nowara2020near} makes the strong assumption that the camera
pixel signal is simply the pulse signal with additive gaussian noise, which makes $A = F^{-1}$ suitable; they achieve good results because of careful signal extraction/noise removal during the time-series extraction phase, not because of the signal recovery algorithm. The work of \cite{shenoy2023unrolled} also uses $A = F^{-1}$ ; however, by “unrolling” proximal gradient descent with deep denoising operators, they effectively correct an inexact forward operator. Our experiments using PMC-iPPG, which only learns the signal prior, shows that $A = F^{-1}$ is suboptimal as described in~\cite{nowara2020near, shenoy2023unrolled}.}

\VS{Our next step was to address this deficiency by learning the forward model~\cite{arridge2023inverse, Lunz2020OnLO}. However, we found this problem to be ill-posed for a variety of reasons as described in Appendix A.3.1 and Figure~\ref{fig:face_change}.}

\begin{itemize}
    \item \VS{\textbf{Stochastiscity:} The mapping between the blood volume pulse signal and the camera pixel signal is stochastic and time-dependent due to unconstrained motion and its associated changes in specular and diffuse reflections. This is reflected in Figure~\ref{fig:face_change}, which shows the significant changes in the signal over four second time intervals.}
    \item \VS{\textbf{One-to-many mapping:} The pixel signal from different regions of the face correspond to one blood volume pulse signal, which is captured at the finger. A deterministic function can not map a single ground-truth pulse to different facial regions. An example of this is shown in Figure~\ref{fig:face_change} when comparing the pixel signals at the left and right cheek.}
\end{itemize}

\VS{All these lead us to the decision to use stochastic interpolants, which \textit{implicitly} learns the optimal transport between the distribution of camera pixel signals and BVP signals. Modeling the signal recovery process as such allows for greater flexibility, especially when accounting for the physics of the problem as well as long-tailed events such as sharp motions. By matching distributions, RIS-iPPG captures non-linear interactions that define the forward process while also allowing for uncertainty quantification, one of our primary goals. }

\begin{figure}
    \centering
    \includegraphics[width=0.8\textwidth]{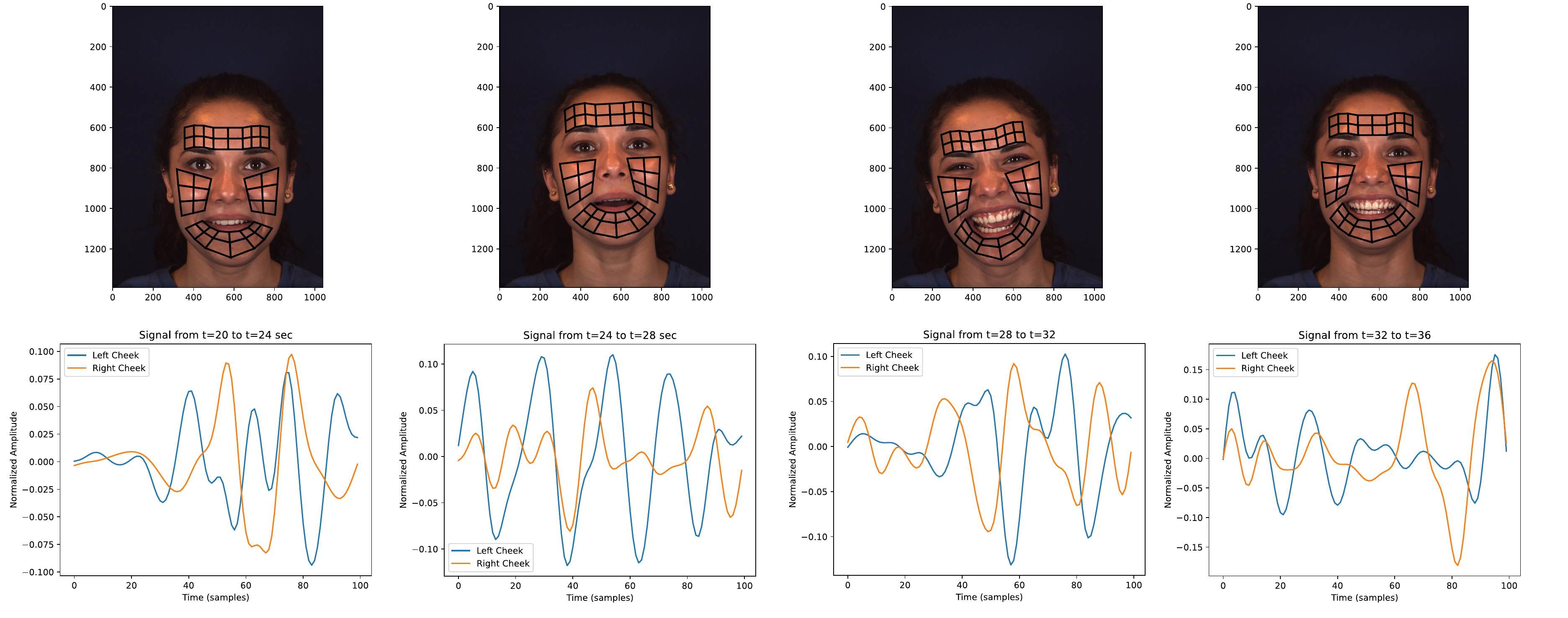}
    \caption{Change in facial position, and region detection over a short time period}
    \label{fig:face_change}
\end{figure}

\subsubsection{Conditioning on a guidance signal}\label{subsec:appendix_guidance}
Many previous works solve traditional imaging inverse problems via guidance. We experimented with guidance in RIS-iPPG. We reasoned that our guidance signal should tell us something about the noise when extracting signal measurements from video; therefore, we use the raw signals from the color channels of the video frames, which we assume capture motion noise via sharp changes in intensity. Then, we used the RIS framework to map our extracted signal from Section~\ref{sec:implementation} to the ground-truth. The learned flow and score networks used the raw color channel signals as guidance when predicting flow and score. We then evaluated our learned models with guidance as in the main paper.

We compare the heart rate estimation performance using each color channel as a guidance signal, as well as using no guidance in Table~\ref{tab:guidance_testing}. Clearly, the guidance signal made performance worse. However, we can not conclude that guidance signals, in general, hurt iPPG performance as we did not perform a thorough analysis of the entire design space of guidance signals. Nevertheless, we argue that a model regularized using the RCL loss is better as it is independent of the noise profile and focuses only on the temporal correlations of the pulse.

\begin{table}[h!]
    \centering
    \caption{Testing various guidance signals on the MMSE-HR dataset}
    \begin{tabular}{c|c|c}
        \toprule
         Guidance Signal & MAE (bpm) $\downarrow$ & RMSE (bpm) $\downarrow$ \\
         \hline
         Blue & 5.79 & 12.71 \\
         Red & 4.17 & 11.78 \\
         Green & 4.05 & 11.18\\
         No Guidance & 3.72 & 6.55 \\
         \bottomrule
    \end{tabular}
    
    \label{tab:guidance_testing}
\end{table}

\begin{table}[h!]
    \centering
    \caption{\VS{Comparing Conditional VAEs and RIS-iPPG on the MMSE-HR dataset}}
    \resizebox{0.60\textwidth}{!}{\begin{tabular}{c|c|c}
         \toprule
         Method &  MAE (bpm) $\downarrow$ & RMSE (bpm) $\downarrow$ \\
         \hline
         Conditional VAEs & 6.17 & 9.96 \\
         Bayesian Neural Networks & 3.12 & 5.63 \\
         RIS-iPPG & 1.97 & 3.73 \\
         \bottomrule
    \end{tabular}}
    \label{tab:pmc_ippg_testing}
\end{table}

\begin{table}[]
    \caption{Varying the RCL loss weight $\lambda_{\text{RCL}}$ and the stride $\delta$ on the UBFC-rPPG Dataset}
    \centering
    \setlength\tabcolsep{6pt}
    
     \resizebox{0.75\textwidth}{!}{\begin{tabular}{cc|cccccccccc|c}
    \toprule
    & & \multicolumn{10}{c|}{Window Shift $\delta$} &   \\
    & & 1 & 2 & 3 & 4 & 5 & 6 & 7 & 8 & 9 & 10  & Average\\
    \hline
    \parbox[t]{2mm}{\multirow{10}{*}{\rotatebox[origin=c]{90}{RCL weight $\lambda_{\text{rcl}}$}}} & 0.0 & 3.72 & 3.72 & 3.72 & 3.72 & 3.72 & 3.72 & 3.72 & 3.72 & 3.72 & 3.72 & 3.72 \\
    & 0.1 & 6.71 & 3.41 & 1.19 & 4.55 & 2.79 & 3.59 & 3.05 & 7.61 & 2.39 & 3.23 & 3.85 \\
    & 0.2 & 3.77 & 3.71 & 5.63 & 2.69 & 2.77 & 2.87 & 3.59 & 3.47 & 3.71 & 0.53 & 3.27\\
    & 0.3 & 1.61 & 4.01 & 2.03 & 1.79 & 4.74 & 2.03 & 1.67 & 5.09 & 5.09 & 2.39 & 3.04\\
    & 0.4 & 2.59 & 3.35 & 2.79 & 1.07 & 4.19 & 1.67 & 2.61 & 2.69 & 2.09 & 1.07 & 2.47\\
    & 0.5 & 3.53 & 3.23 & 3.89 & 1.13 & 1.92 & 0.77 & 3.47 & 6.41 & 1.55 & 1.13 & 2.70\\
    & 0.6 & 3.91 & 2.95 & 2.45 & 0.59 & 2.12 & 3.85 & 3.65 & 2.69 & 0.77 & 0.48 & \textbf{2.44}\\
    & 0.7 & 5.27 & 1.73 & 2.57 & 1.73 & 2.92 & 2.09 & 2.44 & 4.01 & 2.45 & 3.71 & 2.89\\
    & 0.8 & 3.53 & 1.97 & 0.83 & 2.81 & 4.57 & 1.97 & 2.51 & 2.21 & 3.35 & 3.77 & 2.75\\
    & 0.9 & 2.75 & 4.79 & 3.77 & 2.21 & 1.49 & 1.85 & 1.85 & 2.45 & 1.67 & 2.51 & 2.53\\
    & 1.0 & 2.63 & 4.31 & 5.27 & 2.21 & 2.79 & 3.71 & 4.91 & 2.21 & 1.61 & 2.45 & 3.21\\
    \hline
    Average & & 3.63 & 3.37 & 3.10 & 2.31 & 3.09 & 2.55 & 3.04 & 3.86 & 2.58 & \textbf{2.27} &  \\
    \bottomrule

    \end{tabular}}
    \label{tab:rcl_ablation2}
\end{table}

\subsection{Comparison against other posterior sampling methods}

\VS{Below we address a comparison of RIS-iPPG against more standard posterior sampling methods.}

\VS{\textbf{Complexity of the Task}: While our signals are one-dimensional, the \textbf{inverse mapping} between the camera pixel signal and pulse signal is highly complex. In Figure~\ref{fig:face_change}, we show the how the camera pixel signal changes during sharp changes in motion every four seconds, as well as the resulting signals from two different regions of the face. These signals contain high-dimensional, non-Gaussian noise sources such as motion, specular reflections, quantization noise, etc. in the same frequency bands as the pulse. Simple models like Gaussian Processes (GP) or BNNs often assume Gaussian noise or unimodal posteriors, an assumption violated by our data distributions. RIS-iPPG can more effectively model this data.}

\VS{\textbf{Empirical Analysis}: We conduct experiments to measure the performance of Conditional VAEs (cVAE) and Bayesian Neural Networks (BNN) against flow-based models. We evaluate these methods for pulse rate estimation performance in Table~\ref{tab:pmc_ippg_testing}. For Conditional VAEs we adapt the work~\cite{sohn2015learning} and use the camera pixel signal as our conditioning element. For the BNN, we follow the guide of~\cite{jospin2022hands} and set the prior as recommended in the paper. Note that in Table~\ref{tab:pmc_ippg_testing}, we see significant increases in pulse rate estimation error using BNNs and cVAE compared to RIS-iPPG.}

\VS{\textbf{Performance of cVAE}: We notice that pulse rate estimation performance for VAEs is significantly worse than that of the flow-based models. VAEs are known to produce blurry samples due to the ELBO objective; for a task like ours in which we are capturing waveform morphology, VAEs will smoothe features such as the systolic peak resulting in inaccurate pulse rate estimation. Additionally, previous work has noted that VAEs tend to ignore latent variables if the decoder is overly powerful in sequence modeling tasks~\cite{chen2016variational, serban2017hierarchical, bowman2016generating}. It is also well-known that VAEs suffer from posterior collapse~\cite{dai2020usual, lucas2019don}
. Finally, VAEs often show poorer quality of samples as well~\cite{xu2024assessing}.}

\VS{\textbf{Performance of BNNs}: Bayesian neural networks also have their own challenges. One of the biggest challenges is carefully defining the prior distribution over the weights of the BNN~\cite{jospin2022hands}. Previous work has shown that the performance of BNNs are greatly affected by the choice of prior~\cite{silvestro2020prior}, and often the choice of distribution over the weights is non-intuitive for the selected task. It is unclear how the choice of prior over the weights translate to a prior over signal morphology. Furthermore, inference methods in BNNs often have drawbacks: MCMC methods are exact but do not scale well to large models while variational methods are inexact and only focus on a single mode of the posterior~\cite{jospin2022hands}. RIS-iPPG, without explicity setting a prior, effectively samples waveforms while capturing multiple modes of the posterior.}

\VS{\textbf{Performance of GPs}: Gaussian Processes (GPs) models suffer from computational complexity concerns, scaling cubicly with the size of the dataset~\cite{eriksson2018scaling}. Additionally, the choice of kernel in GPs often affects the generalization ability of the model~\cite{rasmussen1997evaluation, sun2018differentiable}; as in BNNs, prior knowledge about the kernel with reference to the task is critical for successful application of BNNs. Often, these kernels fail to capture the complex and non-stationary nature of motion artifacts.} 

\VS{\textbf{Benefit of RIS-iPPG}: Flow-based diffusion models solve many of the problems of VAEs, BNNs, and GPs. We can model complex data distributions, capture distinct waveform morphologies, model highly non-Gaussian and non-stationary noise, and capture multimodal posteriors. Given the recent research showing state-of-the-art performance and widespread adoption of flow-based diffusion models, it is critical that we address research in iPPG with respect to such models.}

\subsection{\VS{Inference Speed of RIS-iPPG versus other posterior sampling methods}}\label{subsec:runtime}

\VS{We measure the runtime performance at inference for three sampling methods in Table~\ref{tab:runtime}. At inference, BNNs are best in terms of speed; cVAEs and RIS-iPPG are about the same. However, the models in RIS-iPPG need to be evaluated for every timsetp; depending on the granularity with which the SDE steps are discretized, the cost of RIS-iPPG can be much higher. Note that we do not claim RIS-iPPG to be superior in terms of runtime; however, fast inference is important for real-world adoption and future work should aim to improve inference speed. }

\begin{table}[h!]
        \centering
        \caption{\VS{Comparing inference speed of different methods. All were tested on a single NVIDIA A5000 GPU}}
        \resizebox{0.40\textwidth}{!}{\begin{tabular}{c|c}
             \toprule
             Method &  Inference Time (ms)\\
             \hline
             Conditional VAE & 845.06 \\
             Bayesian Neural Networks & 473.91 \\
             RIS-iPPG & 850.97 \\
             \bottomrule
        \end{tabular}}
        \label{tab:runtime}
    \end{table}

\subsection{\VS{Deriving the Denoiser}}\label{subsec:denoiser_deriving}

\VS{In Section~\ref{sec:unreg_si}, we defined the score as $\mathbf{s}(t, \mathbf{x}) = -\mathbf{n_z}(t, \mathbf{x})/\gamma(t)$. We show how this is derived below.}

Define our stochastic interpolant as before:

\begin{equation}\label{eq:appendix_interpolant}
    \mathbf{x}_t = \underbrace{I(t, \mathbf{x}_0, \mathbf{x}_1)}_{\mathbf{x}'} + \gamma(t)\mathbf{z},  \text{ where } \mathbf{z} \sim \mathcal{N}(0, Id), t\in[0, 1]
\end{equation}

The distribution of $\mathbf{x}_t$ is $p(\mathbf{x}_t)\sim\mathcal{N}(\mathbf{x}_t|\mathbf{x}', \gamma^2(t))$. Find the gradient of the log of this distribution with respect to $\mathbf{x}_t$ is then given by

\begin{equation}
    \begin{split}
        \nabla_{\mathbf{x}}\text{log } p(\mathbf{x}_t) &=  \nabla_{\mathbf{x}} \text{log } \frac{1}{(\sqrt{2\pi \gamma^2(t)})^d} \text {exp }\Big(\frac{-||\mathbf{x}_t - \mathbf{x}'||^2}{2\gamma^2(t)}\Big) \\
        &= \nabla_{\mathbf{x}}\Big(\frac{-||\mathbf{x}_t - \mathbf{x}'||^2}{2\gamma^2(t)} - \text{log }\big(\sqrt{2\pi \gamma^2(t)}\big)^d\Big) \\
        &= -\frac{\mathbf{x}_t - \mathbf{x}}{\gamma^2(t)} = - \frac{\mathbf{z}}{\gamma(t)}
    \end{split}
\end{equation}

This is the claim from Section~\ref{sec:unreg_si}.

\subsection{\VS{Intuition Behind the RCL loss}}

\VS{We regularize our models by minimizing the RCL loss, or one-minus the normalized correlation coefficient of the error residuals between time-shifted versions of the training signal. The intuition behind this is based on the slowly-varying nature of physiological signals: in short time-windows barring an acute stressor, human physiology remains nearly constant. Therefore, in our model, the error between the predicted flows in overlapping time windows should be nearly identical. We show this through the following theorem:}

\begin{theorem}
    \textit{The RCL loss of Equation~\ref{eq:rcl_loss} is minimized when the error residuals point in the same direction.}
\end{theorem}

\begin{proof}
    As in Section~\ref{subsec:rcl}, define two pairs of training data as 
    \begin{equation}
        (\mathbf{x}_0(i), \mathbf{x}_1(i)) \text{ and } (\mathbf{x}_0' = \mathbf{x}_0(i - \delta), \mathbf{x}_1' = \mathbf{x}_1(i - \delta))
    \end{equation}

    As described in Section~\ref{sec:unreg_si} we construct two flow targets $\mathbf{v}(t, \mathbf{x}_t)$ and $\mathbf{v}(t, \mathbf{x}'_t)$ which are predicted by our flow model. With $i$ as the time index,  let $\mathbf{p} = \mathbf{v}(t, \mathbf{x}_t(i)) - \mathbf{v}_\theta(t, \mathbf{x}_t(i))$ and $\mathbf{q} = \mathbf{v}(t, \mathbf{x}_t(i-\delta)) - \mathbf{v}_\theta(t, \mathbf{x}_t(i - \delta))$: we seek to align these vectors by aligning these vector by minimizing the RCL loss

    \begin{equation}
    \begin{split}
        \min_{\theta} \mathcal{L}_{\text{RCL}}(\mathbf{p}, \mathbf{q}) &= \min_{\theta} \left[1- \frac{T \cdot \mathbf{p}^\top\mathbf{q} - \mu_p\mu_q}{\sqrt{(T\cdot \mathbf{p^\top p} - \mu_p^2)(T \cdot \mathbf{q}^\top \mathbf{q} - \mu_q^2)}}\right] \\
        &=\max_{\theta} \frac{T \cdot \mathbf{p}^\top\mathbf{q} - \mu_p\mu_q}{\sqrt{(T\cdot \mathbf{p^\top p} - \mu_p^2)(T \cdot \mathbf{q}^\top \mathbf{q} - \mu_q^2)}} \\
    \end{split}
    \end{equation}

    We can simplify this optimization problem by first recognizing that $T$ is a scaling factor that is independent of the flow. We can always zero-mean and $L_2$ normalize the signals $\mathbf{p}$ and $\mathbf{q}$ before computing the maximization above. Assume that the normalized error residuals for $\mathbf{p}$ and $\mathbf{q}$ are $\mathbf{y}(i)$ and $\mathbf{y}(i - \delta)$ since we assume a time-shift in the signals results in a time-shift in the residuals. This is achieved through using convolutional operators, which is the basis upon which the flow prediction model is built. Then

    \begin{equation}
    \begin{split}
        \max_{\theta} \mathbf{y}(i)^\top \mathbf{y}(i - \delta) &= \max_{\theta} ||\mathbf{y}(i)|| \cdot
        ||\mathbf{y}(i - \delta)|| \cdot \text{cos}(\phi)\\
        &= \max_{\theta} \text{cos}(\phi)
    \end{split}
    \end{equation}

    where $\phi$ is the angle between the vectors $\mathbf{y}(i)$ and $\mathbf{y}(i - \delta)$. This value is maximized when $\phi = 0$, or the vectors point in the same direction. \qed
    
\end{proof}

\subsection{\VS{Deriving the Sampling SDE}}\label{subsec:appendix_deriving_sde}

\VS{In Theorem~\ref{eq:si_theorem}, we state that the Fokker-Planck equations satisfy the continuity theorem. We show that this is the case (replicated and simplified from~\cite{albergo2023stochasticinterpolantsunifyingframework}:}

\begin{theorem}{(Theorem 2.6 of~\cite{albergo2023stochasticinterpolantsunifyingframework})}\label{eq:si_theorem_appendix}
The probability distribution of the interpolant $x_t$ defined in Equation~\eqref{eq:interpolant} is absolutely continuous with respect to the Lebesgue measure at times $t\in[0,1]$ and solves the transport equation
\begin{equation}\label{eq:continuity}
    \frac{\partial}{\partial t} p_t + \nabla \cdot (\mathbf{b}p) = 0 
\end{equation}

In addition, the forward and backward Fokker-Planck equations are satisfied

\begin{equation}\label{eq:fokker-planck1}
    \frac{\partial}{\partial t} p_t + \nabla \cdot (\mathbf{b}_Fp) = \epsilon(t)\Delta p, \mathbf{b}_F = \mathbf{b}(t,\mathbf{x}) + \epsilon(t)\mathbf{s}(t,\mathbf{x}) 
\end{equation}
\begin{equation}\label{eq:fokker-planck2}
    \frac{\partial}{\partial t} p_t + \nabla \cdot (\mathbf{b}_Bp) = -\epsilon(t)\Delta p, \mathbf{b}_B = \mathbf{b}(t,\mathbf{x}) - \epsilon(t)\mathbf{s}(t,\mathbf{x}) 
\end{equation}

where $\epsilon(t)$ is some noise schedule.
\end{theorem}

\begin{proof}
    \VS{We would like to show that the drift coefficient satisfies the Continuity Equation in Equation~\eqref{eq:continuity}. Consider the diffusion term from Equation~\ref{eq:fokker-planck1}. From SDE theory, we know that we can write the Laplacian of the distribution $p$ as:}

    \begin{equation}
        \epsilon(t)\Delta p = \epsilon(t)\nabla \cdot(p\nabla \text{ log }p) = \epsilon(t)\nabla \cdot (s(t,\mathbf{x})p)
    \end{equation}

    \VS{Now, we would like to show that the drift coefficient satisfies the continuity equation. Substituting the above result into Equation~\eqref{eq:fokker-planck1}, we have}

  \begin{equation}
    \begin{split}
        \frac{\partial}{\partial t} p_t + \nabla \cdot (\mathbf{b}_Fp) &= \epsilon(t)\Delta p \\
        \frac{\partial}{\partial t} p_t + \nabla \cdot ((\mathbf{b}(t,\mathbf{x}) + \epsilon(t)\mathbf{s}(t,\mathbf{x}))p) &= \epsilon(t)\nabla \cdot (ps(t,\mathbf{x})) \\
        \frac{\partial}{\partial t} p_t + \nabla \cdot ((\mathbf{b}(t,\mathbf{x})p) + \epsilon(t)\ \nabla \cdot (\mathbf{s}(t,\mathbf{x})p) &= \epsilon(t) \nabla \cdot (\mathbf{s}(t,\mathbf{x})p) \\
        \frac{\partial}{\partial t} p_t + \nabla \cdot ((\mathbf{b}(t,\mathbf{x})p)  &= 0
    \end{split}
\end{equation}

    \VS{The derivation is similar for Equation~\ref{eq:fokker-planck2}.}
\end{proof}

\VS{We now seek to derive the equation for the drift coefficient from Equation~\ref{eq:sampling_drift}. The sampling SDE is defined as }

\begin{equation}\label{eq:appendix_sampling_sde}
    d\mathbf{x}_t = \mathbf{b}_F(t, \mathbf{x}_t)dt + \sigma_tdW_t, \text{  } \sigma_t = \sqrt{2\epsilon(t)}
\end{equation}

and the drift coefficient $\mathbf{b}_F(t, \mathbf{x}_t)$ is defined as
\begin{equation}
    \begin{split}
        \mathbf{b}_F &= \mathbf{b}(t,\mathbf{x}) + \epsilon(t)\mathbf{s}(t,\mathbf{x}) \\
        \mathbf{b}_F &= \mathbb{E}[\frac{\partial}{\partial t }\big(I(\mathbf{x}_0, \mathbf{x}_1, t) + \gamma(t) \mathbf{z}\big)] + \epsilon(t)\mathbf{s}(t,\mathbf{x}) \\
        \mathbf{b}_F &= \mathbb{E}[\frac{\partial}{\partial t }I(\mathbf{x}_0, \mathbf{x}_1, t) + (\frac{\partial}{\partial t }\gamma(t)) \frac{\mathbf{-n_z}(t, \mathbf{x}\big)}{\gamma(t)}\big] + \epsilon(t)\mathbf{s}(t,\mathbf{x}) \\
        \mathbf{b}_F &= \mathbf{v}_\theta(t, \mathbf{x}) - \big(\frac{\partial}{\partial t }\gamma(t)\big)\gamma(t) \mathbf{n_z}(t, \mathbf{x})\big)+ \epsilon(t)\mathbf{s}(t,\mathbf{x}) \\
    \end{split}
\end{equation}

\VS{The first equality restates the drift equation from Equation~\ref{eq:fokker-planck1}. In the second equality, we replace substitute the definition of our iPPG interpolant for $\mathbf{b}(t, \mathbf{x})$. In the third equality, we take the derivative with respect to each term, and replace $\mathbf{z}$ with our denoiser $\mathbf{s}(t, \mathbf{x}) = \frac{-\mathbf{n_z}(t, \mathbf{x})}{\gamma(t)}$. We simplify terms in the last equation.}

\qed

\subsection{Stability, Data Requirements, and Approximation Quality}
\VS{The work of \cite{albergo2023stochasticinterpolantsunifyingframework} discusses this in detail in Sections 2.4 and 2.5, with proofs of likelihood control, density estimation, and cross-entropy calculations in the appendix. Demonstrating these properties in practice is important. We do so as follows:}

    \begin{itemize}
        \item \textbf{\VS{Stability and Convergence:}} \VS{To show our stability and convergence, in Appendix A.8 we plot the training and validation loss curves for the flow, score, and RCL components in Figure~\ref{fig:loss_curves}. We see that both training and validation curves decrease monotonically, confirming that we are successfully approximating our intended drift coefficient. We also provide a proof of stability in Theorem~\ref{thm:stability_thm}.}
        \item \textbf{\VS{Data Requirements:}} \VS{The datasets we used recorded at least one minute of video per subject at at least 25 FPS, with many subjects recording multiple videos under stimuli. This provides adequate data to sample the manifold of physiological signals under varying levels of noise.}
        \item \textbf{\VS{Approximation Quality:}} \VS{Our Gauge R\&R analysis, shown in Appendix~\ref{sec:gauge_analysis} and Figure~\ref{fig:r&r_analysis}, quantifies the repeatability and precision of our system. The variance due to repeatability is negligible (1.7\%) compared to the variance between different signal frequencies (``Part"). This indicates high-fidelity approximation.}

        \item\textbf{\VS{Robustness and Generalization via RCL:}}: \VS{Regarding generalization, the Residual Correlation Loss (RCL) acts as an inductive bias (or smoothness prior). By constraining the solution manifold to signals that are consistent across time shifts, we theoretically reduce the hypothesis space to physically plausible signals. This improves robustness against high-frequency outliers (e.g., sudden motion artifacts) that violate this correlation structure. We have also added training/validation curves in Figure 7 to show that our model learns the flow and score correctly.}

        \item\textbf{\VS{Error Composition (Approximation vs. Discretization):}} 
        \VS{We model the total sampling error as the sum of the \textit{network approximation error} and the \textit{solver discretization error}. As detailed in Equation~\ref{eq:error_composition}, the error bound is given by:}
        \begin{equation}\label{eq:error_composition}
            \text{Total Error} \leq \mathbb{E}\left[\int_{0}^{1} \|b_{F}(t,x_{t}) - b_{\theta}(t,x_{t})\| dt\right] + \mathcal{O}(\Delta t)
        \end{equation}
        \VS{The first term represents the training quality (how well our neural network $b_{\theta}$ approximates the true vector field $b_F$). The second term, $\mathcal{O}(\Delta t)$, represents the error introduced by the SDE solver steps. This confirms that the error is additive and controllable via training convergence (Figure~\ref{fig:loss_curves}) and step-size selection.}
        \end{itemize}

    We also show the stability of our system through a proof below.

    \begin{theorem}\label{thm:stability_thm}
    \VS{Let $\mathbf{\hat{x}}_1$ be a camera measurement and $\mathbf{\tilde{x}}_1 = \mathbf{\hat{x}}_1 + \delta$ be a perturbed measurement, where $\|\delta\| < \infty$. Let $\mathbf{\hat{x}}_0$ and $\mathbf{\tilde{x}}_0$ be generated by the same Wiener process $W_t$. If the learned drift coefficient $\mathbf{b}_{F, \theta}(t, \mathbf{x})$ is Lipschitz continuous with constant $L$, then the reconstruction error is bounded by:}
    \begin{equation}
        \|\mathbf{\hat{x}}_0 - \mathbf{\tilde{x}}_0 \| \leq e^L \|\delta\|
    \end{equation}
    \VS{The system is bounded-input, bounded-output (BIBO) stable.}
\end{theorem}

\begin{proof}
    Define the reverse time SDE for both processes as:
    \begin{equation}
        \begin{aligned}
            d\mathbf{\hat{x}}_t &= \mathbf{b}_{F, \theta}(t, \mathbf{\hat{x}}_t)\,dt + \sigma_t\,dW_t \\
            d\mathbf{\tilde{x}}_t &= \mathbf{b}_{F, \theta}(t, \mathbf{\tilde{x}}_t)\,dt + \sigma_t\,dW_t
        \end{aligned}
    \end{equation}
    Subtracting these equations, the noise terms cancel out, and we get the ODE for the difference:
    \begin{equation}
        \frac{d}{dt}(\mathbf{\hat{x}}_t - \mathbf{\tilde{x}}_t) = \mathbf{b}_{F, \theta}(t, \mathbf{\hat{x}}_t) - \mathbf{b}_{F, \theta}(t, \mathbf{\tilde{x}}_t)
    \end{equation}
    Now, let the squared error between our two solutions be $u(t) = \|\mathbf{\hat{x}}_t - \mathbf{\tilde{x}}_t\|^2$. By the Chain rule, we have:
    \begin{equation}
        \begin{aligned}
            \frac{d}{dt}u(t) &= \frac{d}{dt} \langle \mathbf{\hat{x}}_t - \mathbf{\tilde{x}}_t, \mathbf{\hat{x}}_t - \mathbf{\tilde{x}}_t \rangle \\
            &= 2\left\langle\mathbf{\hat{x}}_t - \mathbf{\tilde{x}}_t, \frac{d}{dt}(\mathbf{\hat{x}}_t - \mathbf{\tilde{x}}_t) \right\rangle \\
            &= 2\left\langle\mathbf{\hat{x}}_t - \mathbf{\tilde{x}}_t, \mathbf{b}_{F, \theta}(t, \mathbf{\hat{x}}_t) - \mathbf{b}_{F, \theta}(t, \mathbf{\tilde{x}}_t) \right\rangle \\
            &\leq 2 \|\mathbf{\hat{x}}_t - \mathbf{\tilde{x}}_t\|\cdot L \|\mathbf{\hat{x}}_t - \mathbf{\tilde{x}}_t\| \\
            &= 2Lu(t)
        \end{aligned}
    \end{equation}
    In the inequality step, we used the Cauchy-Schwarz inequality and the Lipschitz condition. Now, applying the differential form of Gronwall's inequality backwards from $t=1$ to $t=0$, we get:
    \begin{equation}
        \begin{aligned}
            u(0) &\leq u(1)e^{2L(1-0)} \\
            \|\mathbf{\hat{x}}_0 - \mathbf{\tilde{x}}_0\|^2 &\leq e^{2L} \|\mathbf{\hat{x}}_1 - \mathbf{\tilde{x}}_1\|^2 \\
            \|\mathbf{\hat{x}}_0 - \mathbf{\tilde{x}}_0\| &\leq e^{L} \|\mathbf{\hat{x}}_1 - \mathbf{\tilde{x}}_1\|
        \end{aligned}
    \end{equation}
    This concludes the proof showing BIBO Stability.
\end{proof}
\begin{figure}[h!]
    \centering
    \begin{subfigure}[b]{0.45\textwidth}
        \centering
        \includegraphics[width=\textwidth]{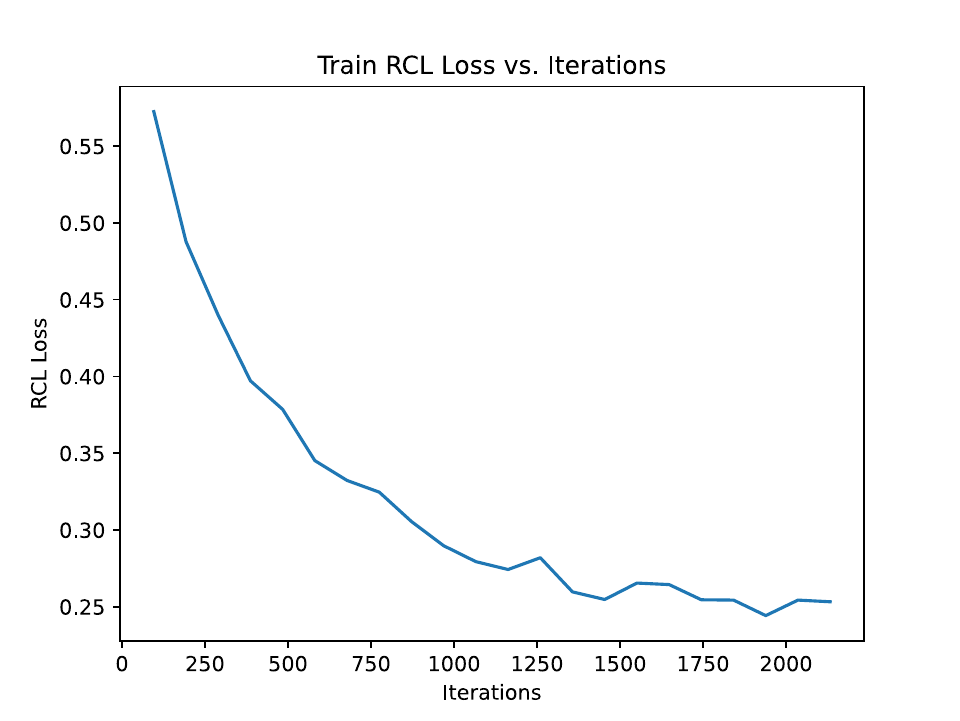}
        \caption{RCL training loss vs training iterations}
        \label{fig:sub1}
    \end{subfigure}
    \hfill
    \begin{subfigure}[b]{0.45\textwidth}
        \centering
        \includegraphics[width=\textwidth]{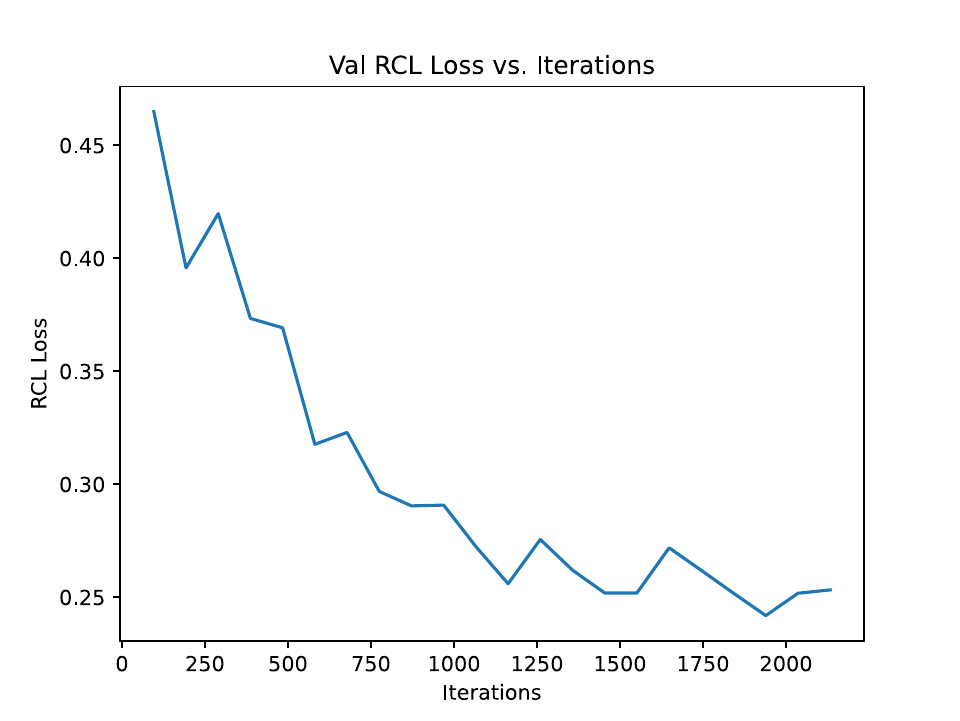}
        \caption{RCL validation loss vs training iterations}
        \label{fig:sub2}
    \end{subfigure}
    \\[1em]
    \begin{subfigure}[b]{0.45\textwidth}
        \centering
        \includegraphics[width=\textwidth]{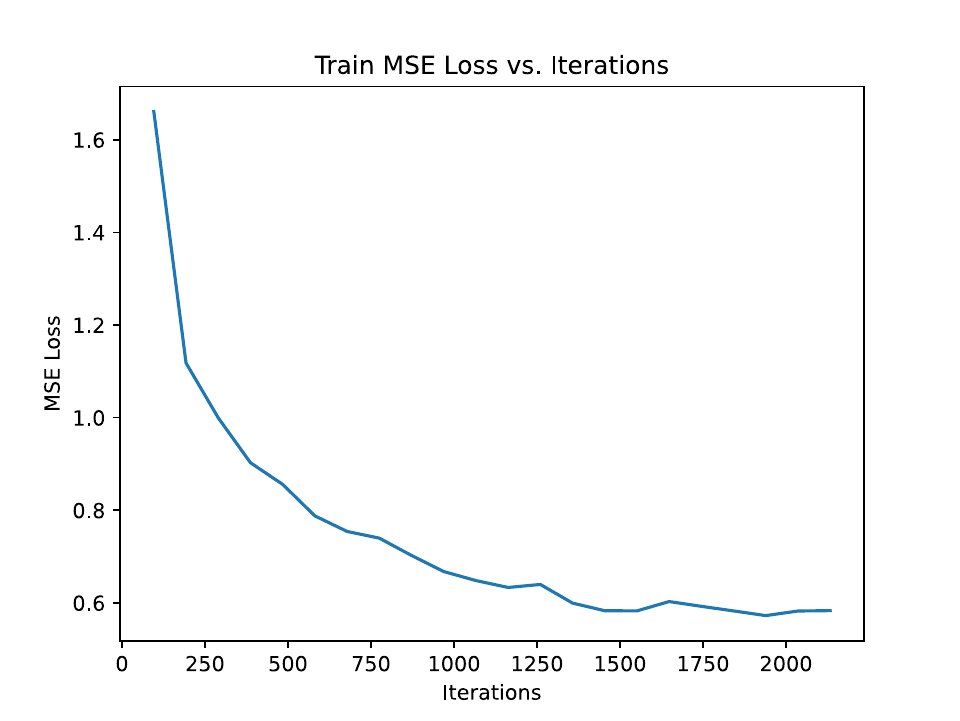}
        \caption{Flow+Score+RCL training loss vs training iterations}
        \label{fig:sub3}
    \end{subfigure}
    \hfill
    \begin{subfigure}[b]{0.45\textwidth}
        \centering
        \includegraphics[width=\textwidth]{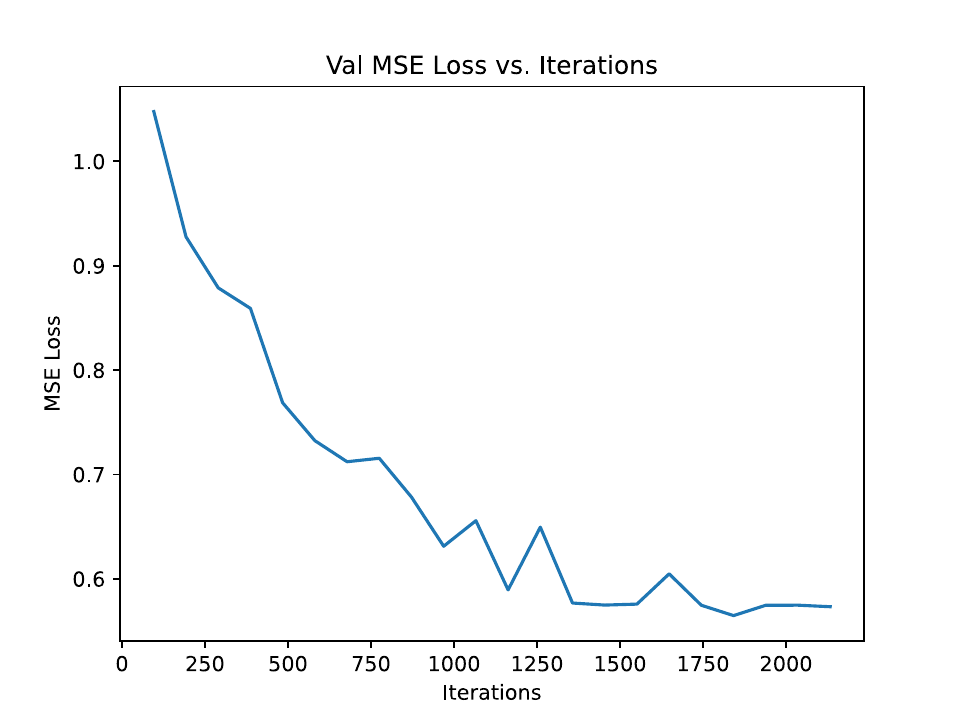}
        \caption{Flow+Score+RCL validation loss vs training iterations}
        \label{fig:sub4}
    \end{subfigure}
    \caption{\VS{The training and validation loss curves for the RCL loss and the entire loss (flow loss + score loss + RCL loss of Equation~\ref{eq:full_loss})}}
    \label{fig:loss_curves}
\end{figure}

\subsection{Additional Quantitative/Qualitative results}


\subsubsection{\VS{RCL loss between predicted flows versus the error residuals}}

\VS{We investigate whether impose temporal regularization on the predicted flows from adjacent time windows as compared to the error residuals from adjacent time windows results in equivalent or better performance. We hypothesize that the RCL on predicted flows would be suboptimal: this would encourage the two time-shifted versions of the signal to align (i.e. resulting in a time-shift $\delta = 0$) despite the fact that we know \textit{a priori} that they should be consistent, but not aligned. Therefore, we decided to align the error residuals instead of the predicted flows themselves.}

\VS{To verify our intuition, we trained our model with RCL regularization between predicted flows instead of error residuals. In Figure~\ref{fig:val_curves_diff_met}, we plot the validation loss of the RCL with flows versus the validation loss with error residuals. It is apparent that the validation error decreases typically for the RCL loss with error residuals; using predicted flows, however, the loss oscillates and the network does not learn. Therefore, we use the RCL regularization with error residuals.}

\begin{figure}
    \centering
    \includegraphics[width=0.5\linewidth]{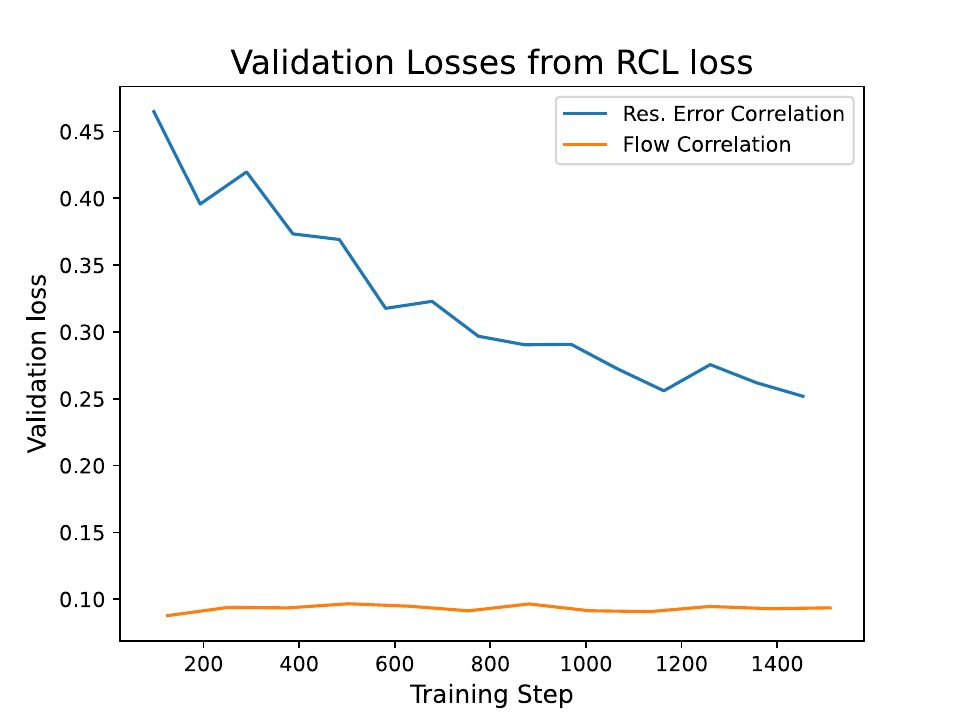}
    \caption{\VS{Validation loss when using the RCL loss with predicted flows versus error residuals. The netork learns when using error residuals, but does not learn when using predicted flows.}}
    \label{fig:val_curves_diff_met}
\end{figure}

\subsubsection{Inclusion of RCL Loss: Effect of overlap and weight}\label{subsec:rcl_overlap_weight}
\VS{Similarly to Table~\ref{tab:rcl_ablation1}, we perform an ablation study on the stride $\delta$ and the the weight $\lambda_{\text{RCL}}$ during model training. Ablations on the UBFC-rPPG dataset are included in Table~\ref{tab:rcl_ablation2}.}

\begin{table}
  \centering
  \caption{\VS{Cross-dataset model evaluation. In some scenarios (e.g. MMSE-HR $\rightarrow$ UBFC-rPPG), we achieve good performance, but in other cross-dataset experiments, our method fails.}}
  \label{tab:cross_dataset_results}
  \setlength\tabcolsep{2pt}
  \resizebox{0.49\textwidth}{!}{
    \begin{tabular}{c|c|cc}
      \toprule
      \textbf{Train} & \textbf{Test} & \textbf{MAE (bpm)} $\downarrow$ & \textbf{RMSE (bpm)} $\downarrow$ \\
      \hline
      \multirow{2}{*}{MMSE-HR} & PURE &  15.52 & 27.26\\
      & UBFC-RPPG & 1.73 & 4.67 \\
      \hline
      \multirow{2}{*}{PURE} & MMSE-HR &  7.93 & 15.19 \\
      & UBFC-rPPG & 21.59 & 33.49 \\
      \hline
      \multirow{2}{*}{UBFC-rPPG} & MMSE-HR &  6.59 & 12.26 \\
      & PURE & 0.26 & 0.53\\
      \bottomrule
    \end{tabular}
  }
\end{table}

\subsubsection{Cross-dataset evaluation}
We evaluate the model's ability to generalize to unseen data. In the experiments shown in Table~\ref{tab:cross_dataset_results}, the columns ``Train'' indicates the dataset on which our flow and score models are trained, while ``Test'' is the dataset on which these models are evaluated. Once again, we report the heart rate estimation metrics from before. We note that in some scenarios, for example training on MMSE-HR and testing on UBFC-rPPG, we perform well with an average error of 1.73 beats per minute. This suggests that the domain of the MMSE-HR dataset (video recording conditions, motion, etc.) is statistically similar to the test set of UBFC-rPPG. However, our experimentation show that this is not always the case; for example training with the PURE dataset but testing on either the MMSE-HR or UBFC-rPPG datasets shows poor performance. A visual inspection of the PURE dataset, as compared to the MMSE-HR and PURE datasets, shows significantly different lighting, as well as limited and controlled motion. Both the MMSE-HR and UBFC-rPPG datasets contain significantly more unconstrained motion---with tasks specifically designed to induce such motion---indicating a significant domain shift in the data.


We surmise that this degradation is due to the learned implicit forward model. We implicitly learn a stochastic function $A(\cdot)$ and noise in Equation~\ref{eq:sig_model}, conditioned on two distributions. When performing cross dataset evaluation, we replace one distribution (source) with another (target), resulting in large errors. One possible solution could be to project the target domain samples into the source domain and re-run the SDE. Addressing the domain shift issues is important, which we leave for follow-up work.

\subsubsection{Calibration Curves on Protected Attributes}\label{sec:appendix_protected_calibration}

\VS{Given the relative accuracy of RIS-iPPG, we now explore UQ metrics on both protected subgroups. We notice that for most metrics, UQ is better for light skin tones as compared to dark skin tones (and for the other metrics the light and dark skin tone metrics are nearly identical. The NLL for light skin tones, however, is significantly better than that of the dark skin tones, indicating that the model fails to learn the distribution of pulse in darker skin tones. The UQ metrics for men are generally better than those for women, though similarly to the pulse rate estimation results, the metrics are quite similar.}

\begin{figure}[h]
    \centering
    \begin{subfigure}[t]{0.24\textwidth}
        \centering
        \includegraphics[width=\textwidth]{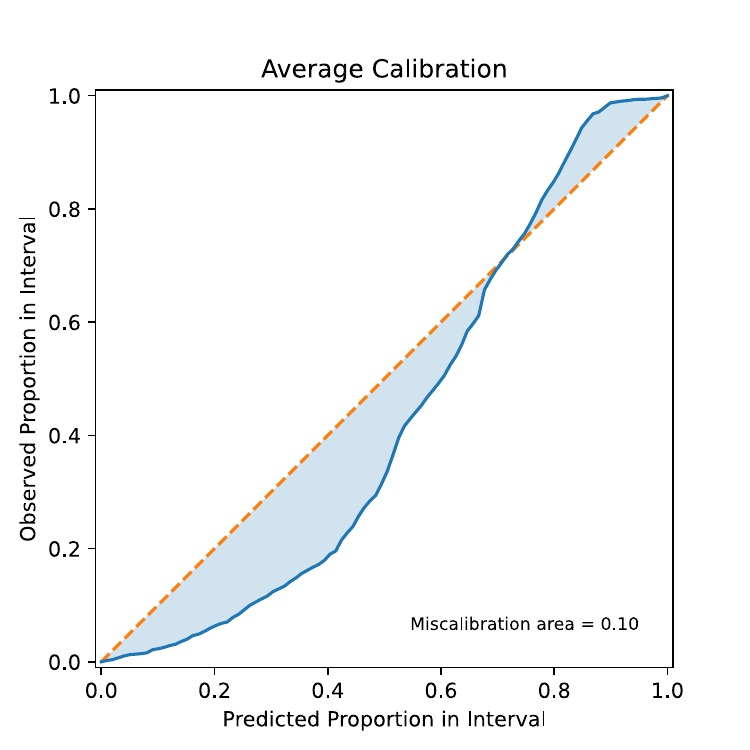}
        \caption{\VS{Dark Skin tones}}
        \label{fig:dark}
    \end{subfigure}
    \hfill
    \begin{subfigure}[t]{0.24\textwidth}
        \centering
        \includegraphics[width=\textwidth]{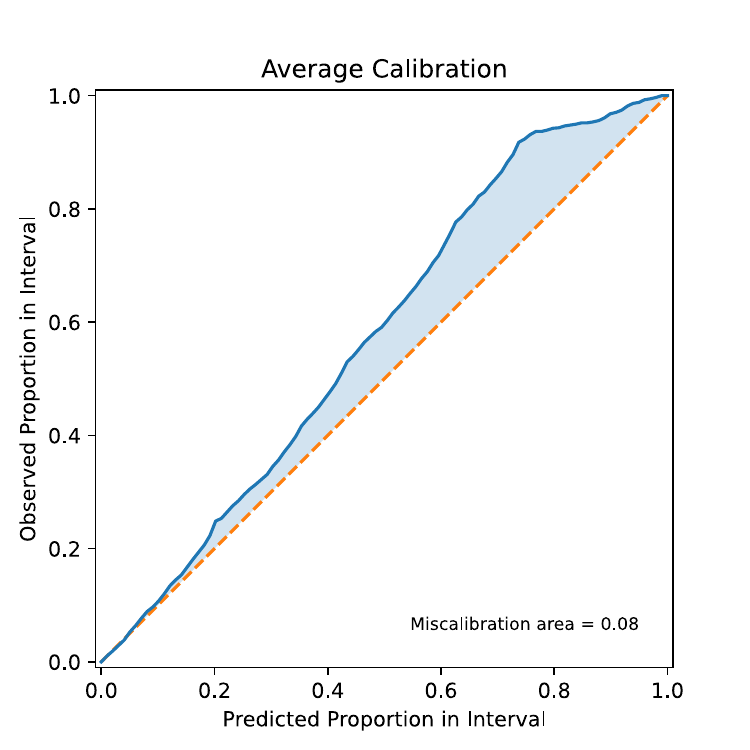}
        \caption{\VS{Light Skin tones}}
        \label{fig:light}
    \end{subfigure}
    \hfill
    \begin{subfigure}[t]{0.24\textwidth}
        \centering
        \includegraphics[width=\textwidth]{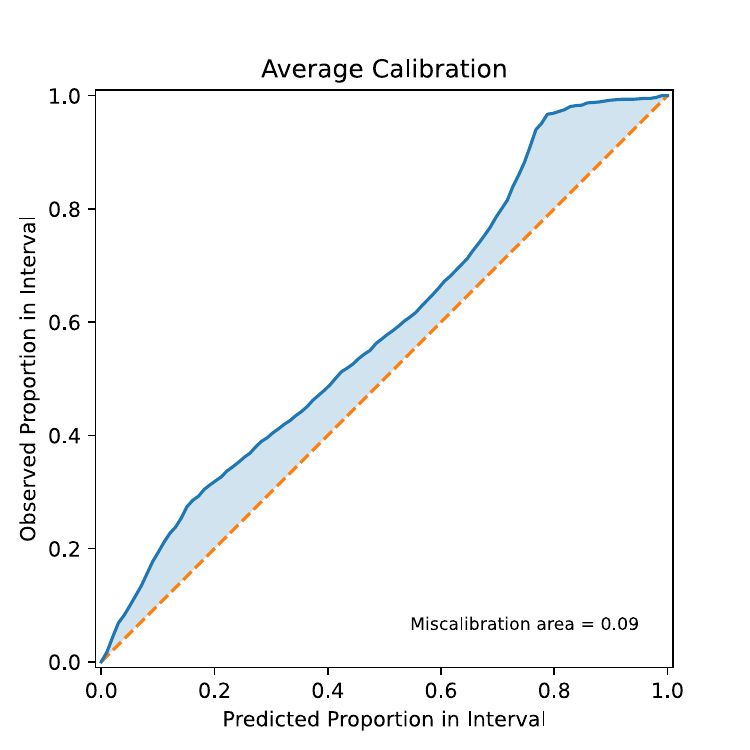}
        \caption{\VS{Men}}
        \label{fig:men}
    \end{subfigure}
    \hfill
    \begin{subfigure}[t]{0.24\textwidth}
        \centering
        \includegraphics[width=\textwidth]{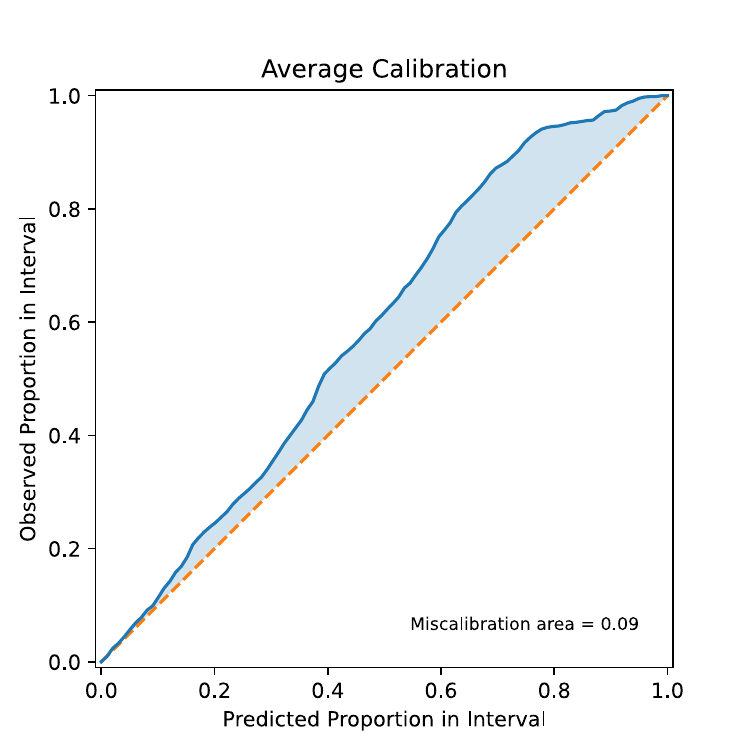}
        \caption{\VS{Women}}
        \label{fig:women}
    \end{subfigure}
    \caption{\VS{Average calibration curves for the four protected attributes on the MMSE-HR dataset. In all four scenarios our method performs well; however, our method is relatively worse on dark skin tones.}}
    \label{fig:four_figures}
\end{figure}

\VS{In addition to the quantitative metrics, we plot the calibration curves in Figure~\ref{fig:four_figures} on protected attributes. On the dark skin tones, our method is over confident when the observed (i.e. ground-truth) proportion is low, but when observing a higher proportion our model is underconfident. For light skin tones, men, and women, our model is consistently underconfident at predicting the spectral magnitude in each frequency bin. In all cases, however, our calibration error is at or lower 0.10, with the largest error at 0.10}.

\subsubsection{Additional Qualitative Results}\label{sec:appendix_qualitative}

\VS{We plot additional qualitative results on all datasets in Figures~\ref{fig:correlation_mmse}, \ref{fig:correlation_ubfc},  \ref{fig:correlation_pure}. The orange signals are the ground-truth in all cases, the green signals are the measurements, and the bolded blue signals are the means from 100 realizations of our algorithm. The light blue shading represents the 95\% confidence interval of the power over each frequency bin of our spectrum. In Figure~\ref{fig:correlation_mmse}, the first and second rows indicate that our algorithm is able to regress a measurement with an incorrect heart rate to the correct heart rate while attenuating the power of unrelated frequencies. Furthermore, we capture two modes of our distribution: in the first row of examples, the RCL loss helps us capture the mode at the measurements as well as the mode at the ground-truth heart rate. The second row shows similar results, with greater uncertainty around the measurements when the RCL loss is included. In the third row, we see that the pulse rate prediction is nearly identical with and without the RCL loss; however, the absence of the RCL loss encourages significant uncertainty at many frequencies above the predicted heart rate, while the RCL loss encourages higher uncertainty only around the harmonic. We see similar results on the UBFC-rPPG and PURE dataset in Figure \ref{fig:correlation_ubfc} and \ref{fig:correlation_pure}.}

\begin{figure}[h!]
    \centering
    \caption{\VS{Inference via solving an SDE. We plot the signal measurements as well as the SDE solution at $\Delta t= 0.1$ time steps in the range $t \in [0, 1]$.}}
    \label{fig:time-layout}

    \begin{subfigure}{\textwidth}
        \centering
        \includegraphics[width=0.5\linewidth, height=4cm]{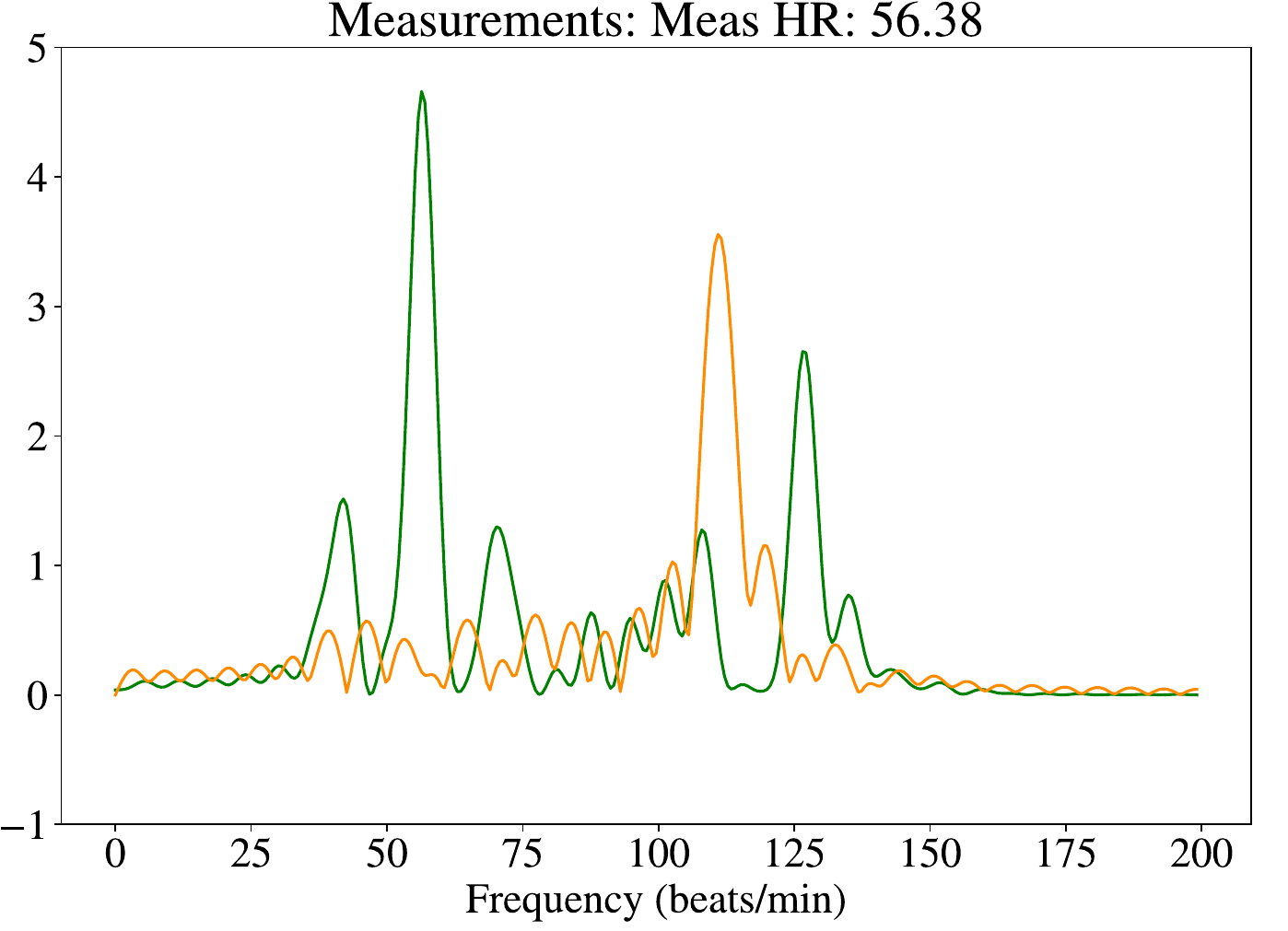}
        \caption{The original signal measurements (green) and ground-truth (orange).}
        \label{fig:top-sub}
    \end{subfigure}

    \vspace{1cm} 

    \begin{subfigure}{\textwidth}
        \centering
        \label{fig:grid}
        \small 

        \begin{subfigure}[b]{0.3\linewidth}
            \centering
            \includegraphics[width=\linewidth, height=2.5cm]{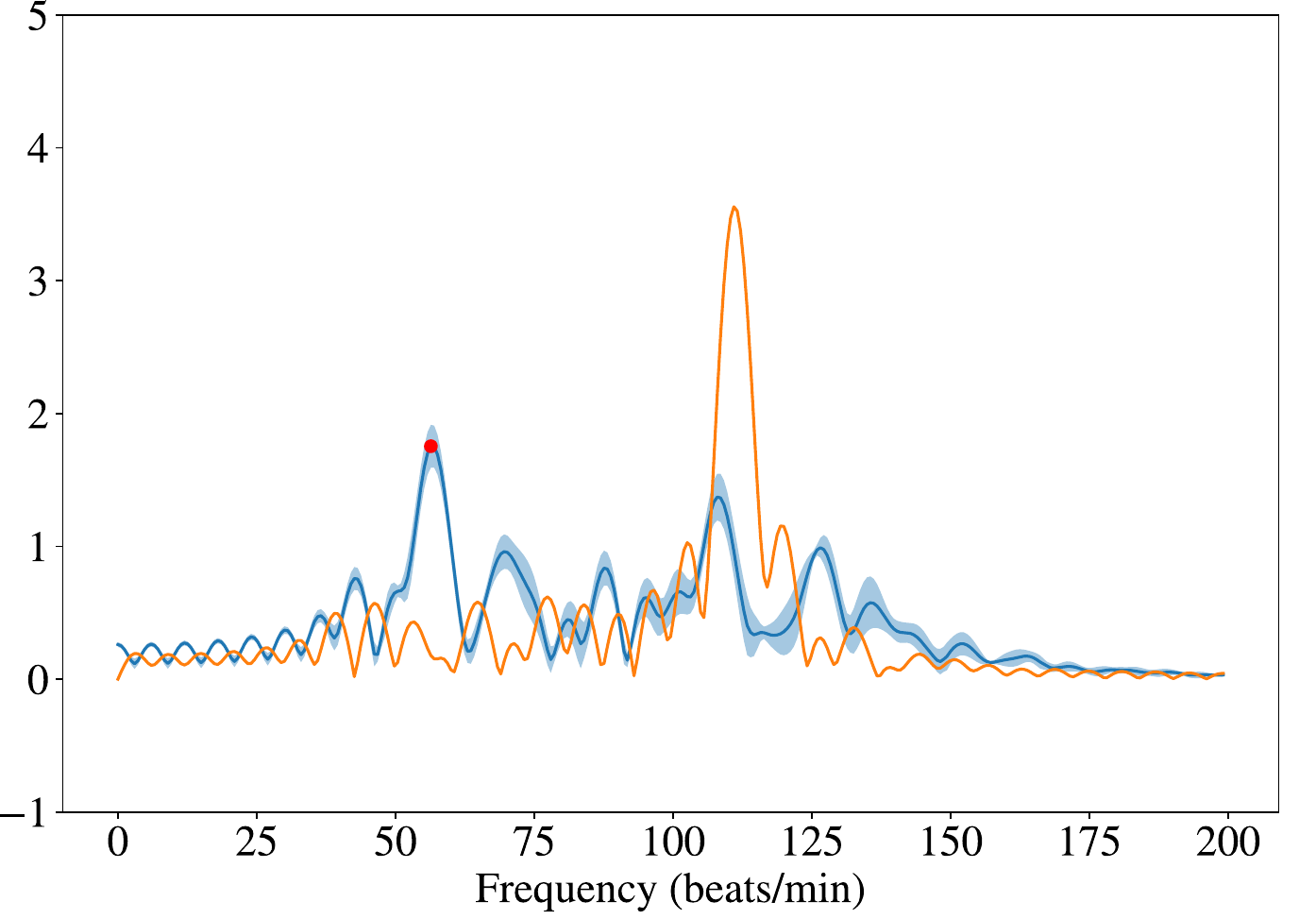}
            \caption{$t=0.1$}
            \label{fig:grid-1}
        \end{subfigure}
        \hfill
        \begin{subfigure}[b]{0.3\linewidth}
            \centering
            \includegraphics[width=\linewidth, height=2.5cm]{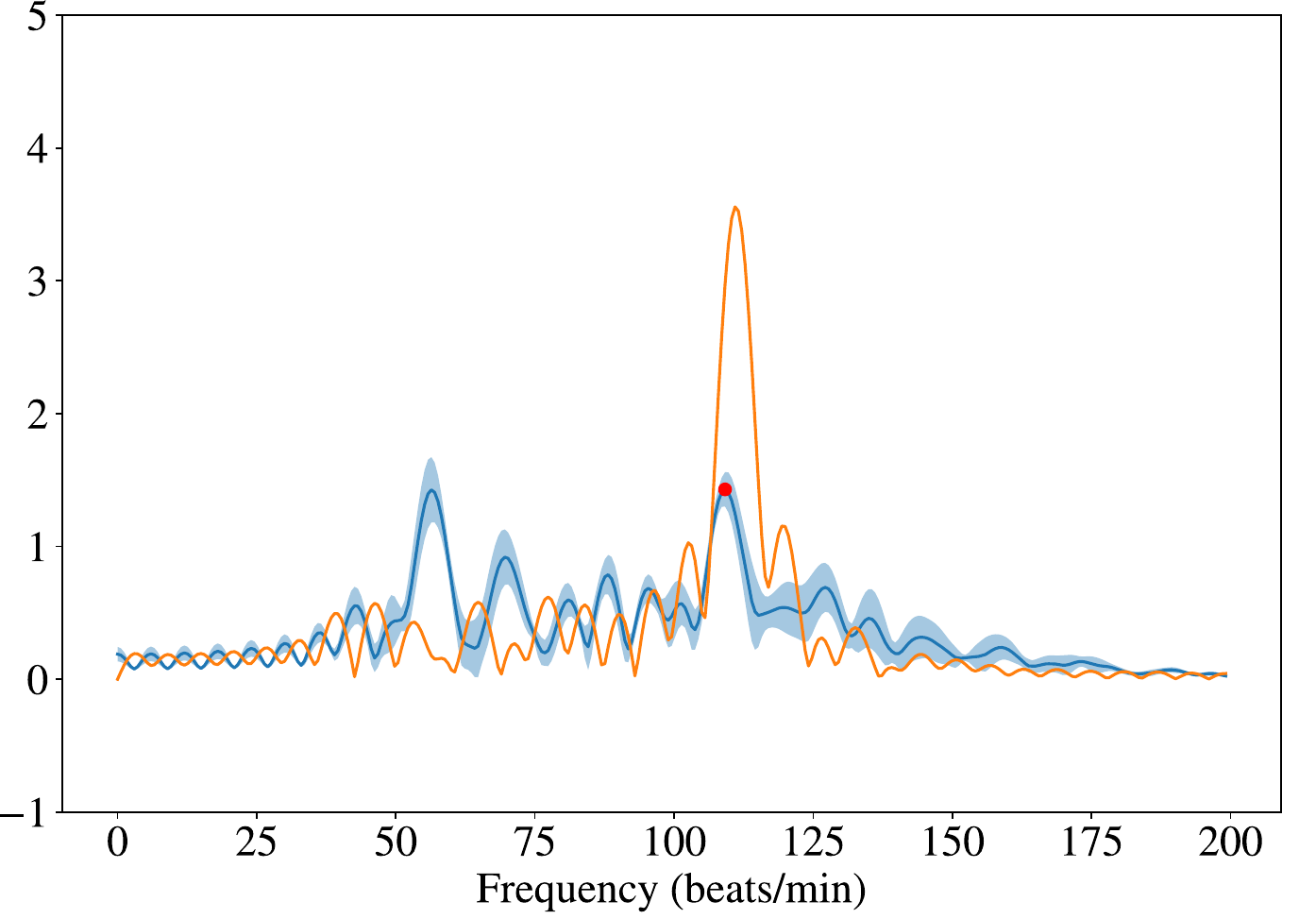}
            \caption{$t=0.2$}
            \label{fig:grid-2}
        \end{subfigure}
        \hfill
        \begin{subfigure}[b]{0.3\linewidth}
            \centering
            \includegraphics[width=\linewidth, height=2.5cm]{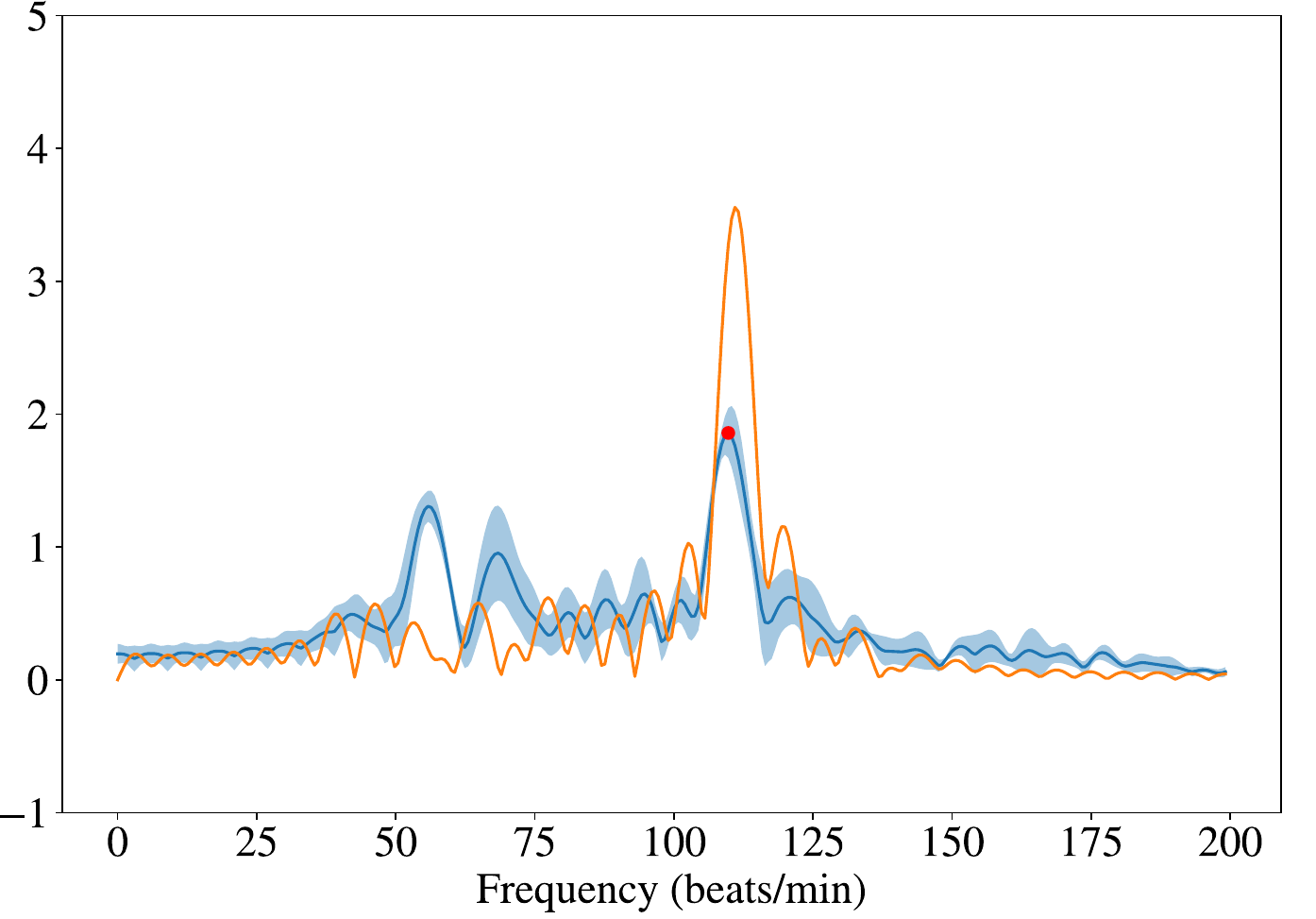}
            \caption{$t=0.3$}
            \label{fig:grid-3}
        \end{subfigure}

        \vspace{0.5cm} 

        \begin{subfigure}[b]{0.3\linewidth}
            \centering
            \includegraphics[width=\linewidth, height=2.5cm]{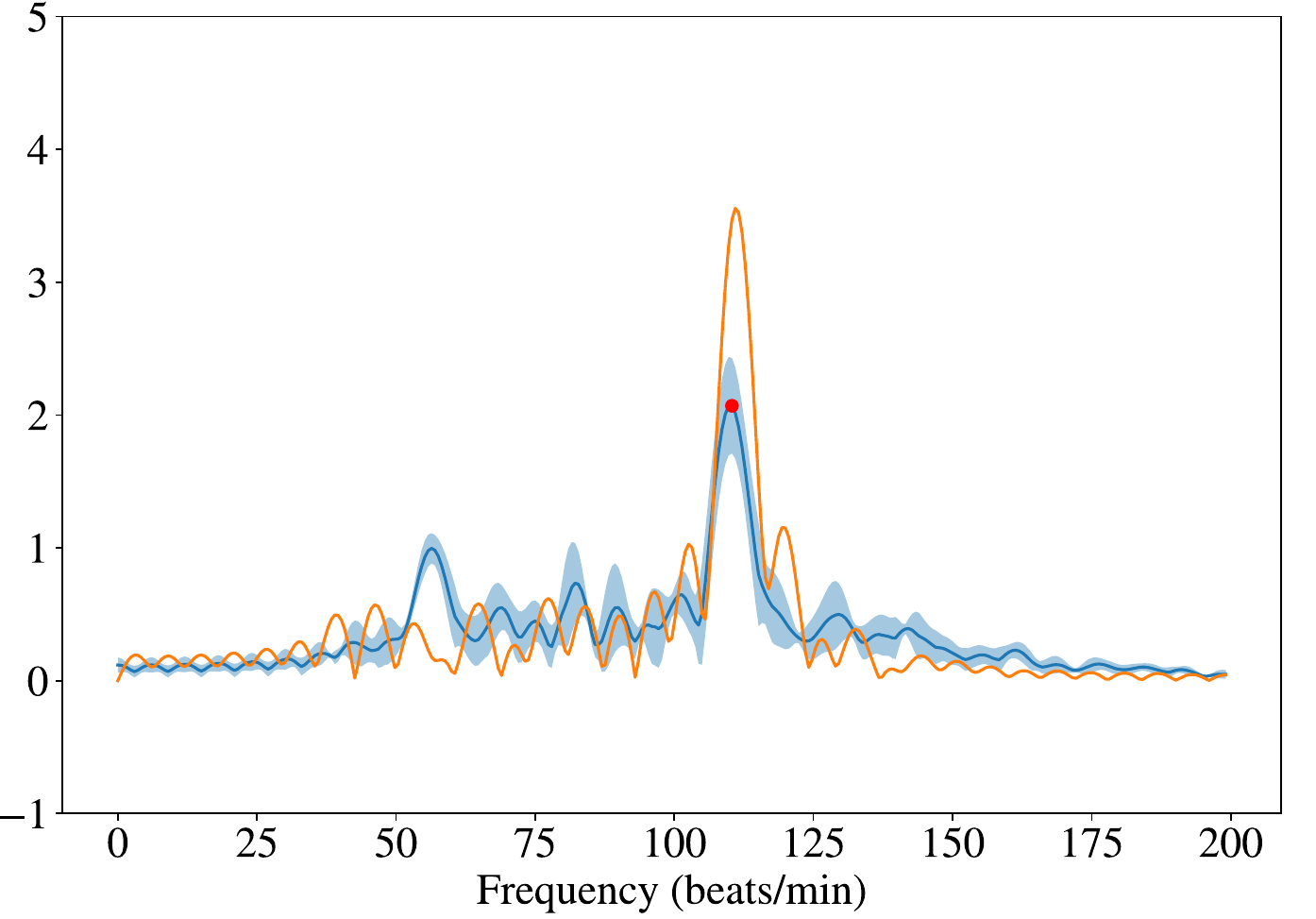}
            \caption{$t=0.4$}
            \label{fig:grid-4}
        \end{subfigure}
        \hfill
        \begin{subfigure}[b]{0.3\linewidth}
            \centering
            \includegraphics[width=\linewidth, height=2.5cm]{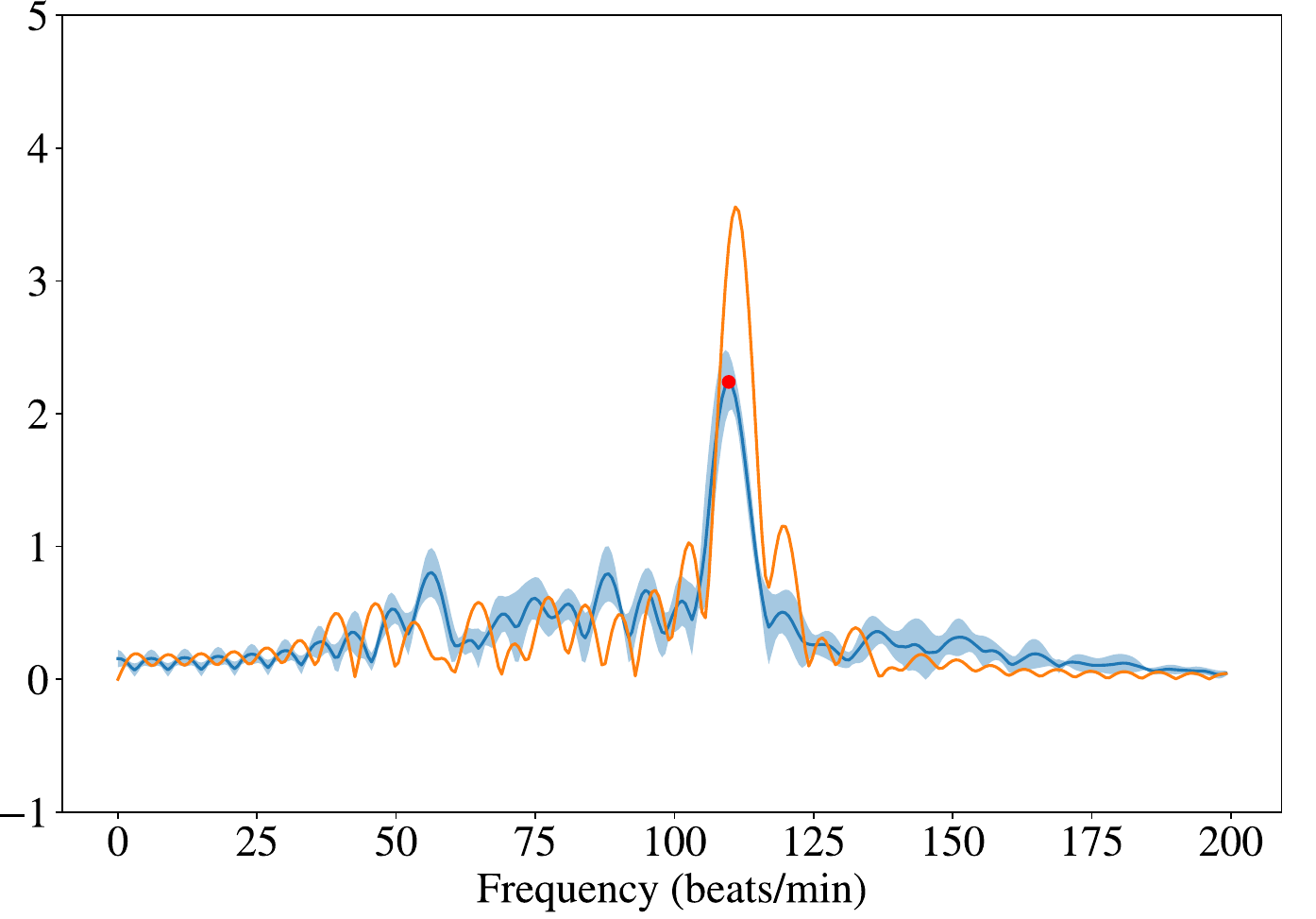}
            \caption{$t=0.5$}
            \label{fig:grid-5}
        \end{subfigure}
        \hfill
        \begin{subfigure}[b]{0.3\linewidth}
            \centering
            \includegraphics[width=\linewidth, height=2.5cm]{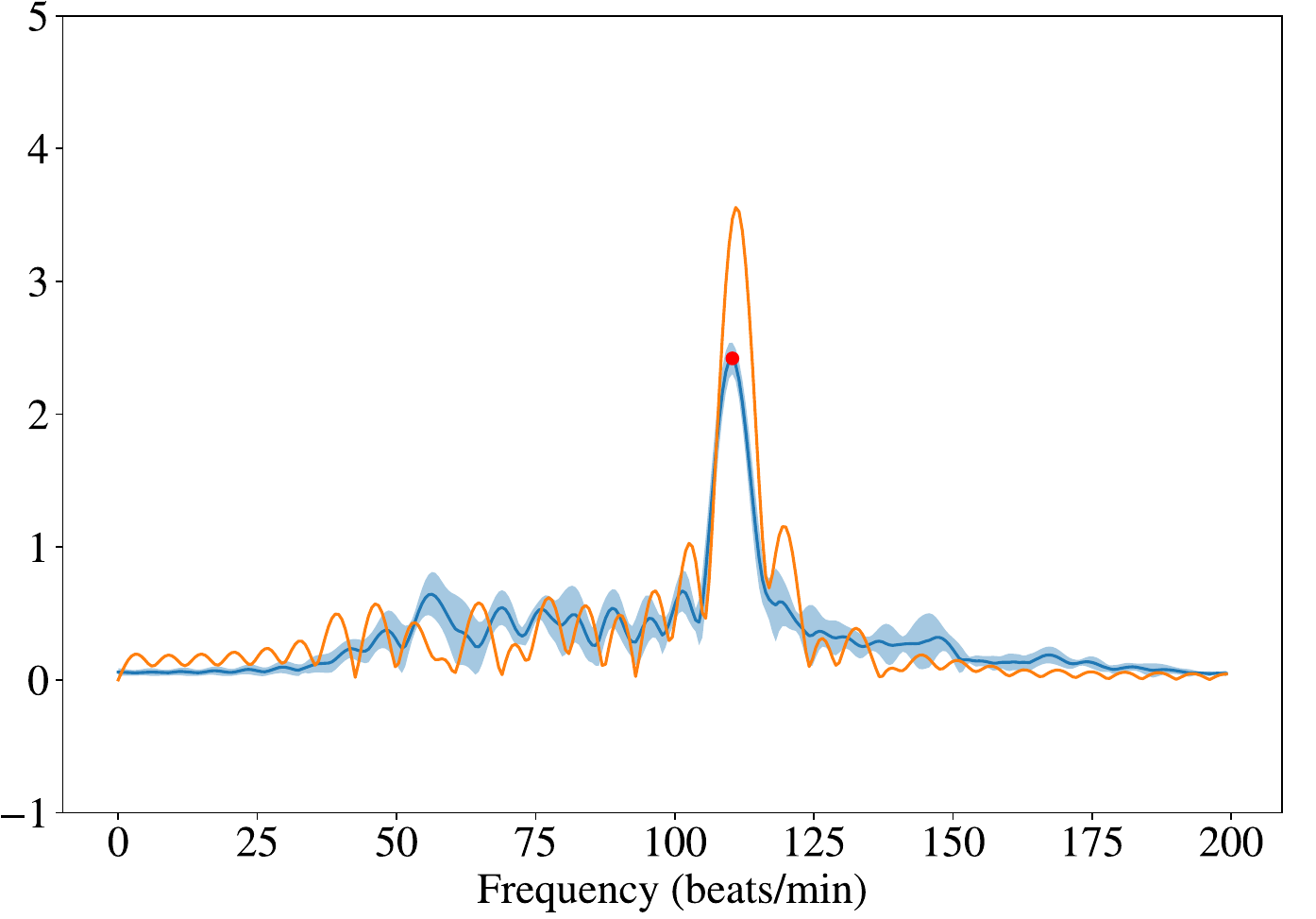}
            \caption{$t=0.6$}
            \label{fig:grid-6}
        \end{subfigure}

        \vspace{0.5cm} 

        \begin{subfigure}[b]{0.3\linewidth}
            \centering
            \includegraphics[width=\linewidth, height=2.5cm]{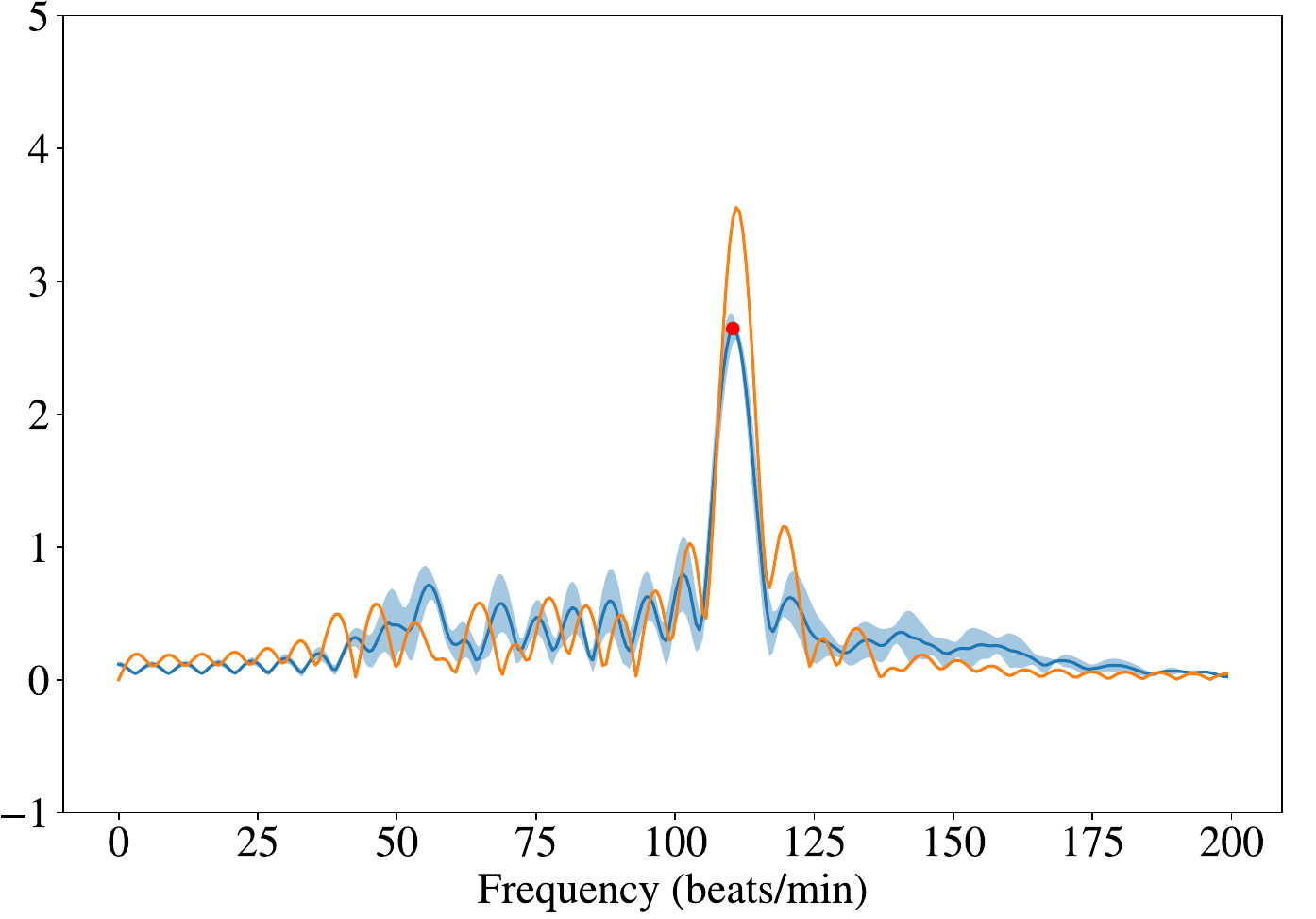}
            \caption{$t=0.7$}
            \label{fig:grid-7}
        \end{subfigure}
        \hfill
        \begin{subfigure}[b]{0.3\linewidth}
            \centering
            \includegraphics[width=\linewidth, height=2.5cm]{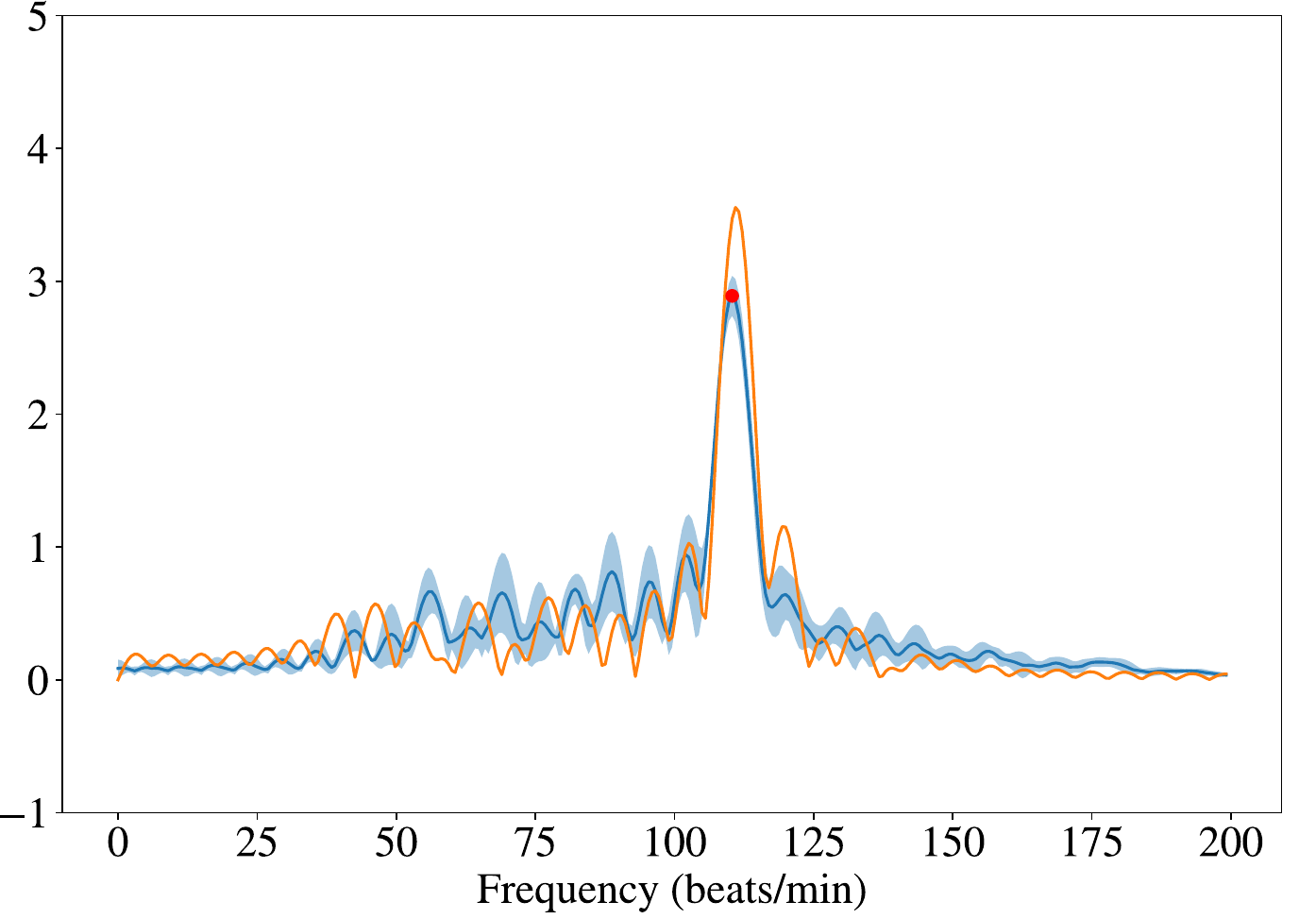}
            \caption{$t=0.8$}
            \label{fig:grid-8}
        \end{subfigure}
        \hfill
        \begin{subfigure}[b]{0.3\linewidth}
            \centering
            \includegraphics[width=\linewidth, height=2.5cm]{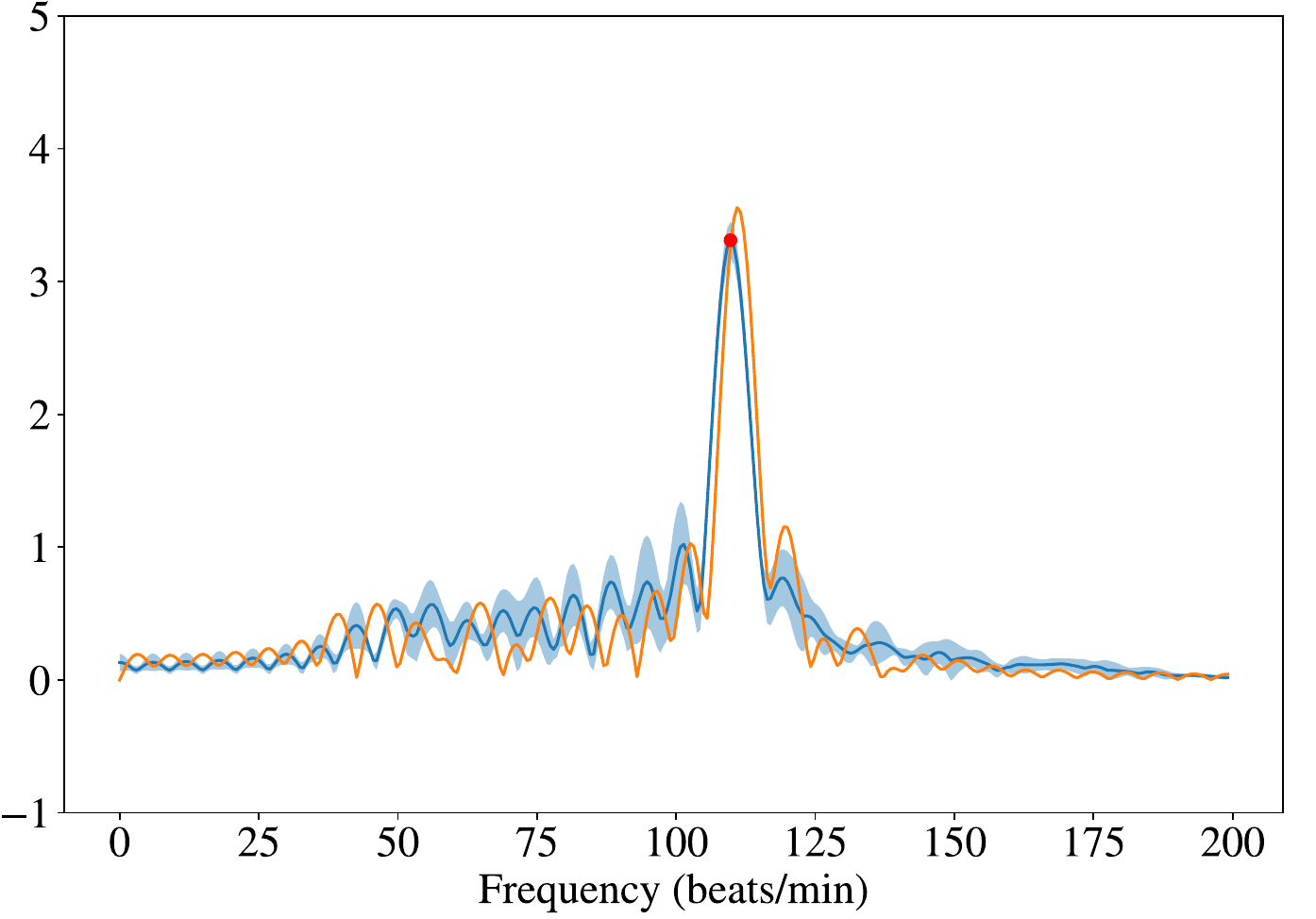}
            \caption{$t=0.9$}
            \label{fig:grid-9}
        \end{subfigure}
    \end{subfigure}

    \vspace{1cm} 

    \begin{subfigure}{\textwidth}
        \centering
        \includegraphics[width=0.6\linewidth, height=5cm]{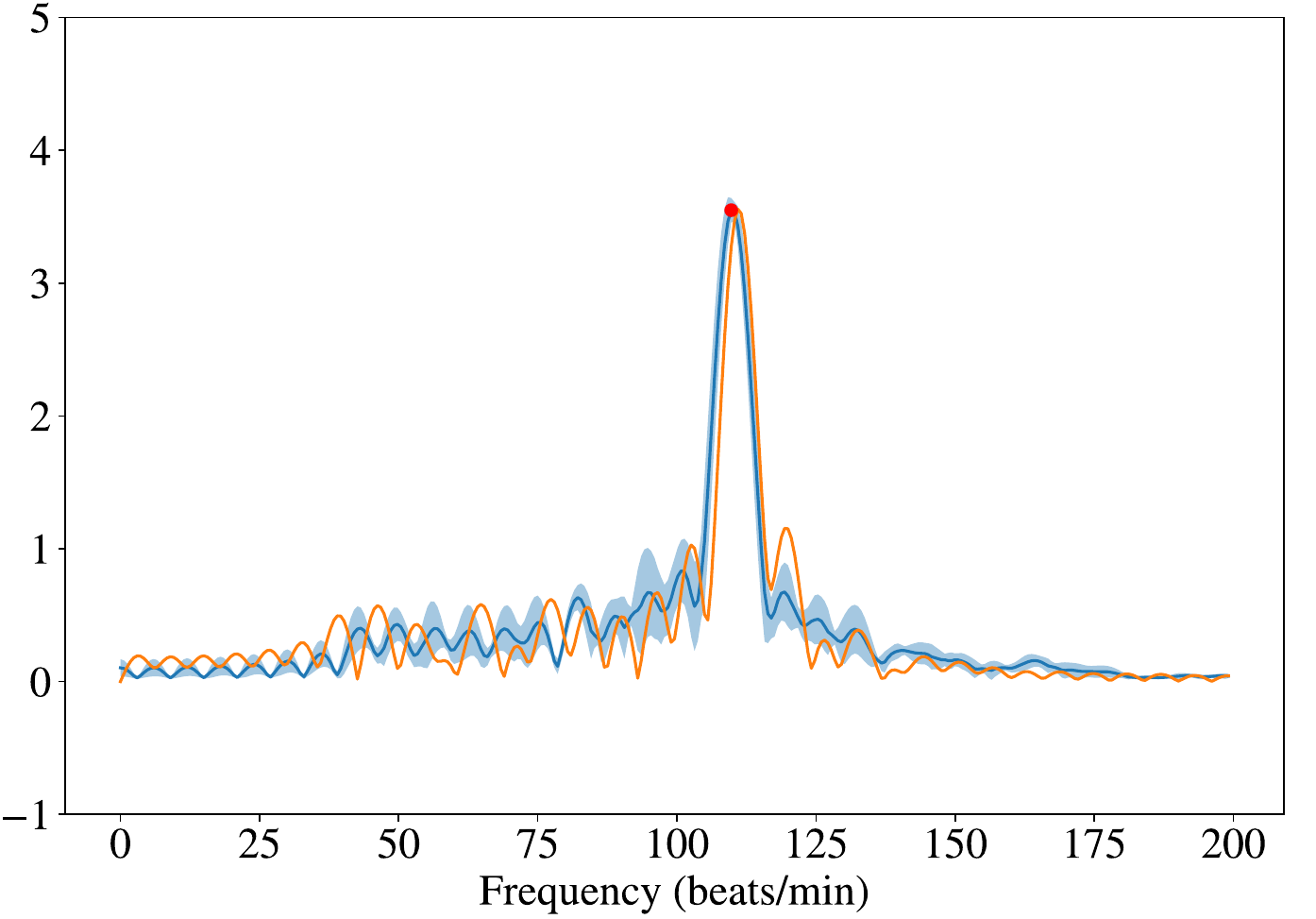}
        \caption{$t=1.0$}
        \label{fig:bottom-sub}
    \end{subfigure}

\end{figure}

\begin{figure}
    \centering
    \includegraphics[width=0.99\textwidth]{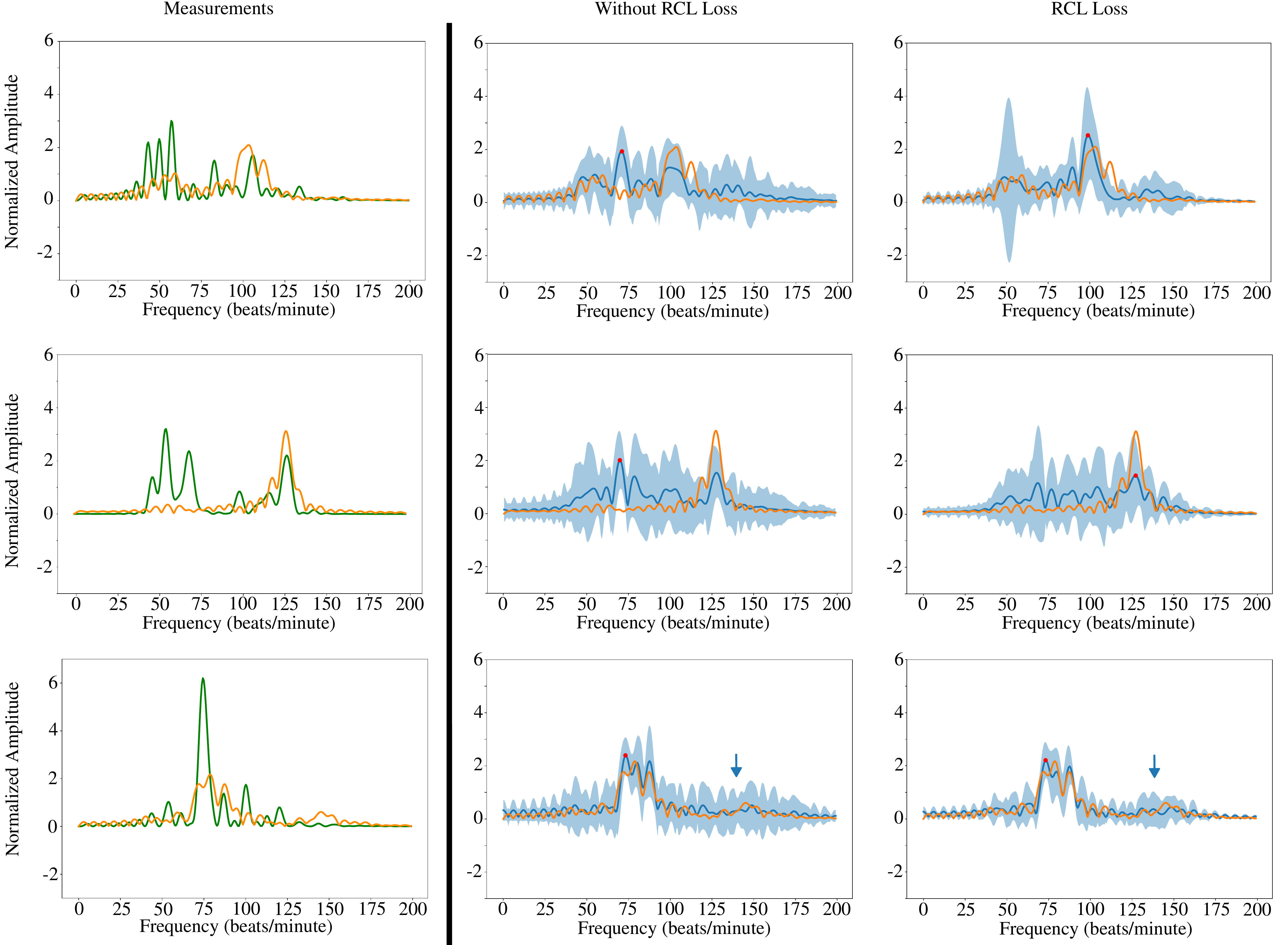}
    \caption{Qualitative results with and without the RCL loss. We plot the camera pixel measurements (green), ground-truth PPG (orange), the mean of 100 realizations of sampling (bolded blue), and  95\% confidence interval of the power in each bin (light blue). Our algorithm with the RCL loss is able to predict the modes of the distribution (first two rows), while also limiting uncertainty except in the frequency bins around the harmonic (light blue arrow, third row).}
    \label{fig:correlation_mmse}
\end{figure}

\begin{figure}
    \centering
    \includegraphics[width=0.99\textwidth]{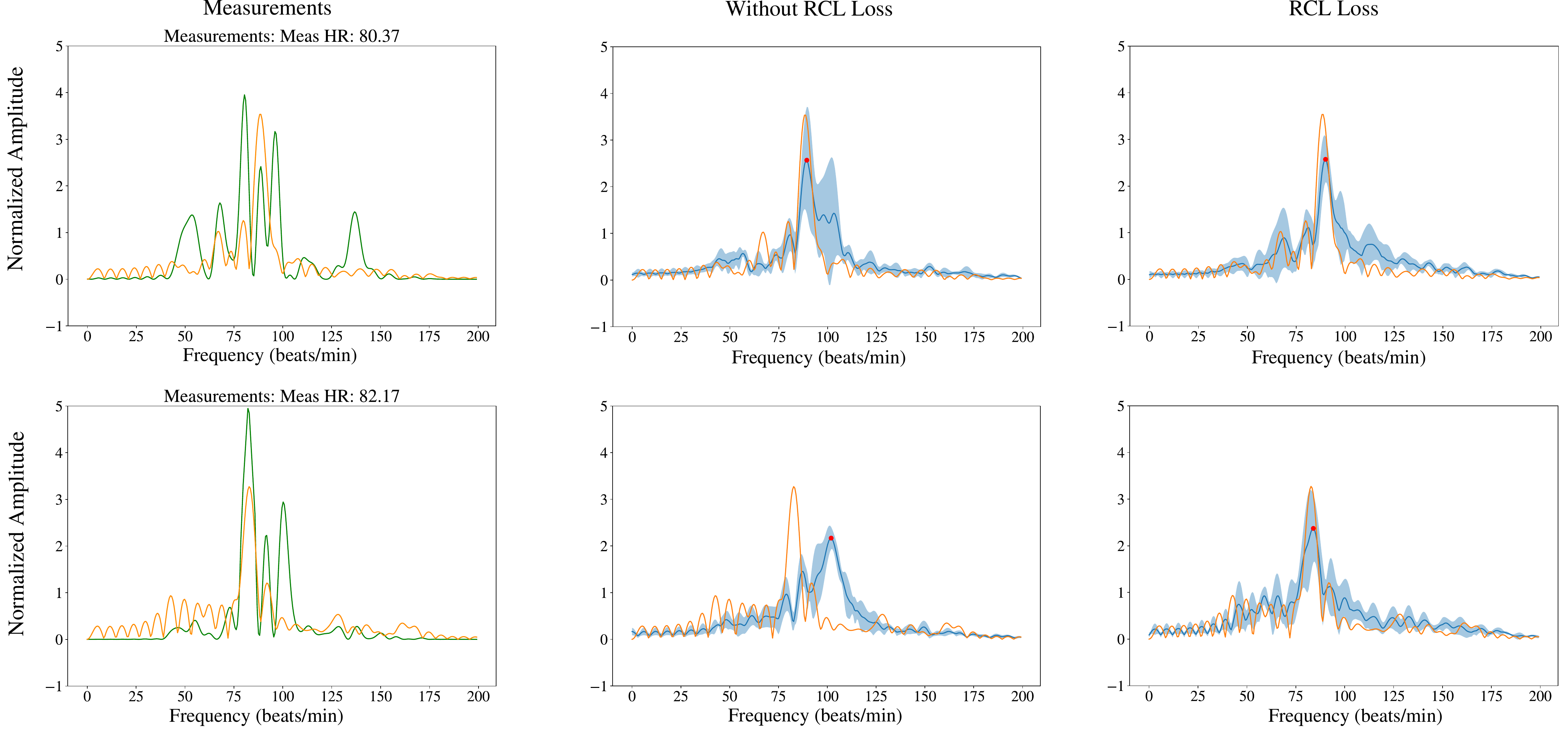}
    \caption{Results on the UBFC-rPPG dataset. The orange signals are the ground-truth, while the green signals are the measurements. The heavy blue signals are the means of 100 measurements. Models with regularization show significantly less variance in prediction.}
    \label{fig:correlation_ubfc}
\end{figure}

\begin{figure}
    \centering
    \includegraphics[width=0.99\textwidth]{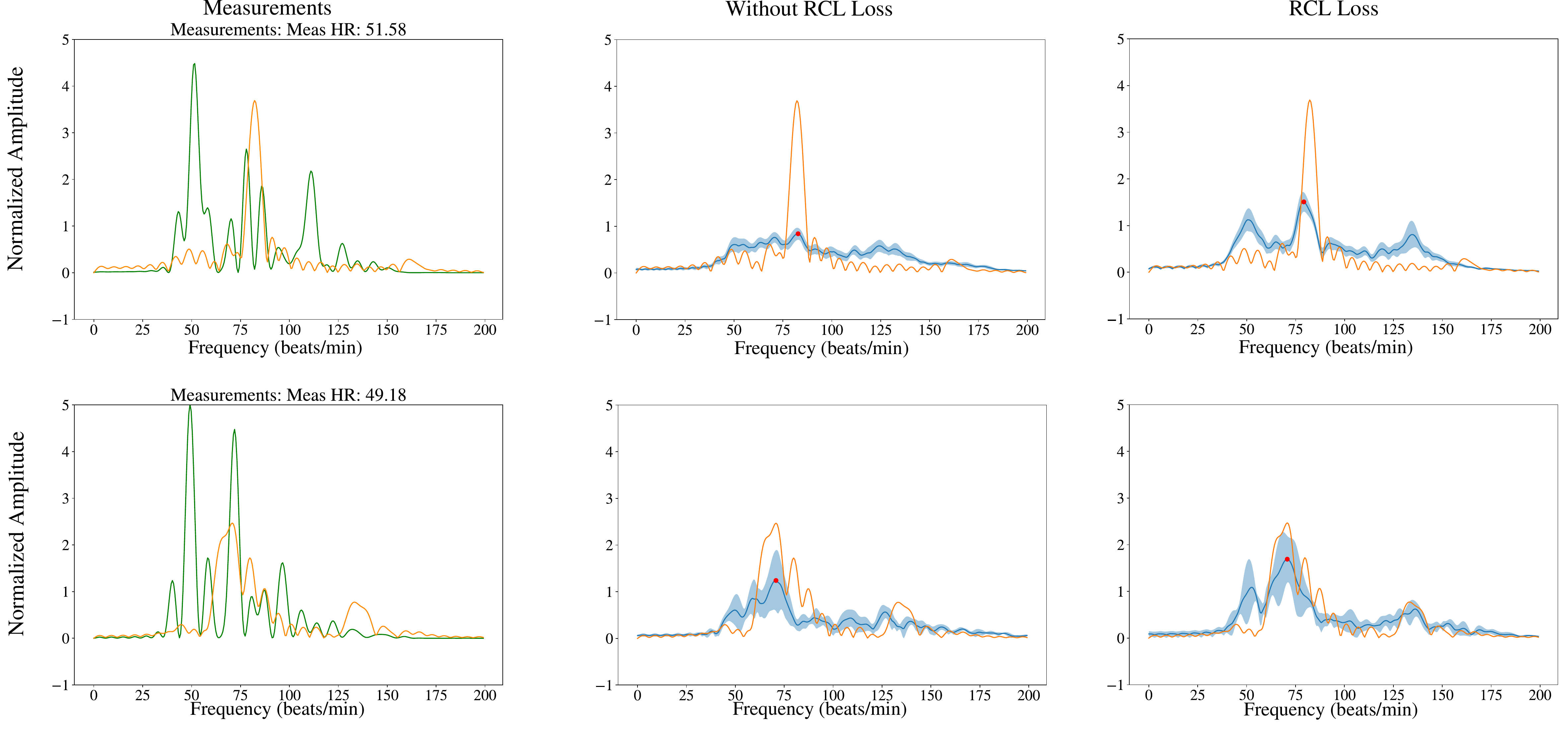}
    \caption{Results on the PURE dataset. The orange signals are the ground-truth, while the green signals are the measurements. The heavy blue signals are the means of 100 measurements. We notice similar performance with and without the RCL loss}
    \label{fig:correlation_pure}
\end{figure}

\subsubsection{Sample-Level Gauge R\&R Results}\label{sec:gauge_analysis}
While the previous metrics capture population level metrics, we revisit the quote from \cite{tonekaboni2019clinicians} regarding the need for tools that are ``repeatedly successful in prognosticating a patient's condition in [the doctor's] personal experience'', and provide tools to answer the question: can we trust the results? To answer this, we borrow techniques from industrial mathematics and perform an ANOVA Gauge repeatability and reproducibility (R\&R) test, which quantifies the amount of variability in the samples due to the measurement system itself, and determines whether the measurement system itself is acceptable. This test quantifies the precision of the system (as compared to accuracy, which is presented Section~\ref{subsec:results}). In traditional industrial mathematics, a gauge (for example, a lathe) drills multiple sized holes (parts) multiple times (samples) by a multiple people (operators), after which the diameter of the hole is measured. The diameters are compared against each other to quantify the variance in the process (lathe) and the measurement system themselves, after which metrics such as repeatability, reproducibility, part-to-part variation, and more are quantified to determine whether a system is acceptable.

\begin{figure}[htbp]
    \centering

    \includegraphics[width=0.8\textwidth]{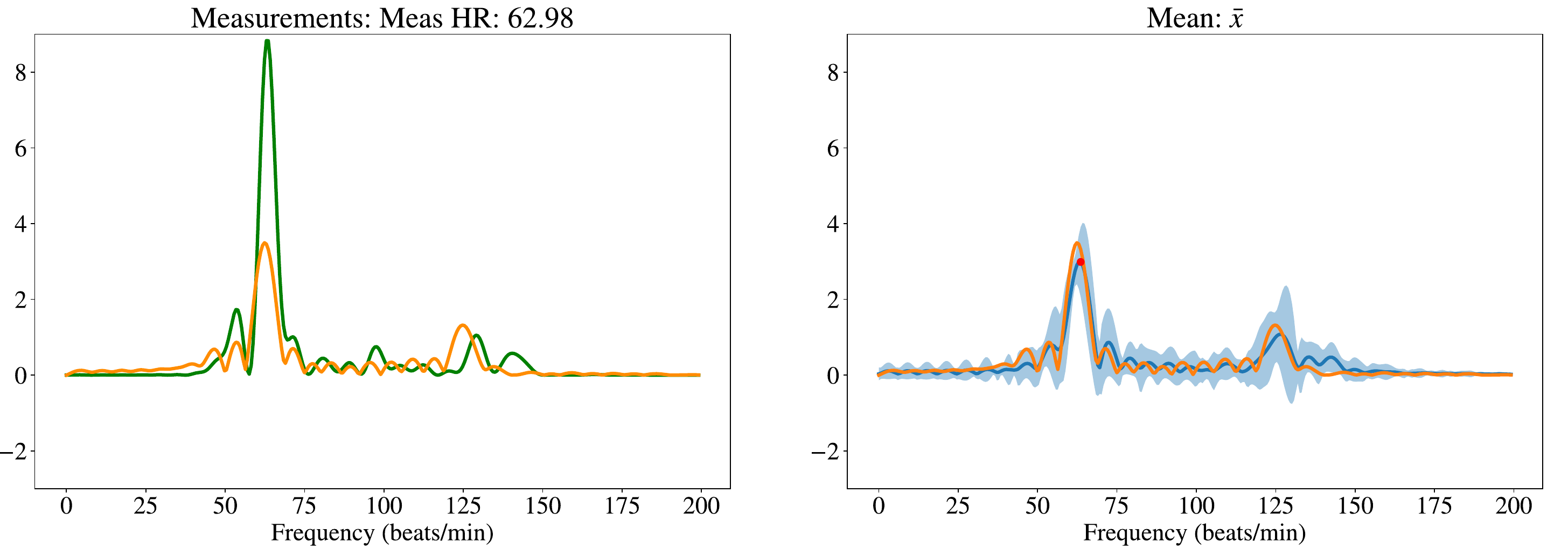} 
    \label{fig:graph}

    \vspace{0.5cm} 

    \begin{tabular}{lcc}
        \toprule
        Source & Variance & \% Variance \\
        \midrule
        Repeatability & 0.0047 & 1.7168 \\
        Reproducibility & 1.1868$\times10^{-4}$ & 0.0434 \\
        Operator & 1.1868$\times10^{-4}$ & 0.0434 \\
        Part & 0.2684 & 98.2398 \\
        \bottomrule
    \end{tabular}
    \label{fig:table}

    \caption{The measurements (green), ground-truth(orange), and mean signal and confidence intervals (heavy blue and light blue). The Table shows the Gauge R\&R test for precision.}
    \label{fig:r&r_analysis}
\end{figure}

We adopt this analysis by considering our ``gauge'' to be the stochastic algorithm, the ``parts'' to be the frequency bins below 200bpm of the magnitude of the Fourier transform of our solution, the ``samples'' to be the number of samples we generate from our SDE, and ``operators'' to be the different facial regions from which we measure signals for a single subject. Since this test operates on a single test example, we chose a sample which we know to be accurate (via heart rate absolute error), compute Gauge R\&R test, and display the results in Figure~\ref{fig:r&r_analysis}. We first state that our estimate is \textit{accurate}; the camera measurements themselves were close to the ground-truth, and the solution denoised this accurate spectrum. Our Gauge R\&R analysis then analyzed the precision of our system, and sources of variation. The table in Figure~\ref{fig:r&r_analysis} shows to which metric we can attribute the majority of our variation; clearly, the largest variation is from the part-to-part variation, while the smallest variation comes from repeatability (i.e. measuring the same frequency bin multiple times) versus the reproducibility (i.e. different facial regions agreeing on the power in the same frequency bin). Given that the largest variation is in the part-to-part model, we conclude that the precision of our system is sufficient.

\subsection{Reconstruction error vs HR error}\label{sec:appendix_nll_comp}

There are some scenarios in which the HR error and the reconstruction error diverge. We show an example in Figure~\ref{fig:nll_comp} in which the reconstruction error is better without the RCL loss, but the heart rate error is lower with the RCL loss. The model with the RCL loss produces higher error and standard deviation, even though on average the heart rate prediction is better.

We show a scenario in Figure~\ref{fig:appendix_error_plots} in which our normalized error is high even though our reconstruction is quite good. We plot our signal and reconstruction in the first row. In the bottom row we plot the absolute error between the mean signal and ground-truth, the standard deviation of the power measured across each frequency bin for all samples, the normalized error across each frequency bin, and the bins which have a normalized error greater and less than as in~\cite{sun2024provable}. We see significant normalized error because our signal is very confident in reconstruction (low standard deviation) relative to the absolute error, which causes the normalized error to explode (which happens similarly in the computation of Equation~\ref{eq:nll}). One of the reasons this happens is because of the ground-truth: the ground-truth signal is captured at the \textit{finger} while we reconstruct the signal from the \textit{face}. While these signals are very similar, they are not the same, resulting in error. In fact, we do not want to reconstruct the finger signal exactly---we want to reconstruct the signal from the face. Ideally we would capture ground-truth PPG signal from the face, but given the data collection constraints, collecting finger PPG is the best option.

\begin{figure}
    \centering
    \includegraphics[width=0.99\textwidth]{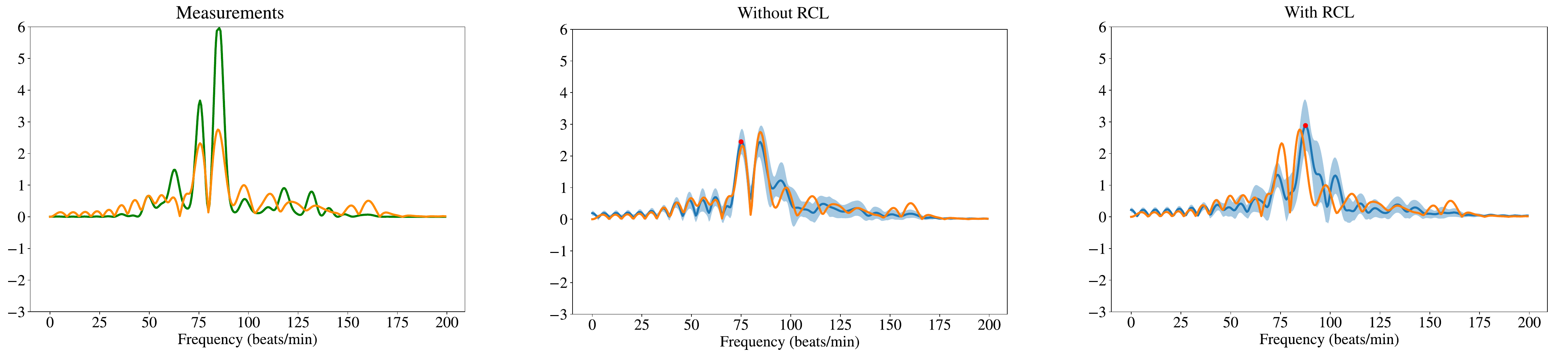}
        \caption{The reconstruction error is better without the RCL loss, but the heart rate estimation error is worse. This scenario is not necessarily uncommon. Without the RCL loss the network is overconfident in a wrong heart rate prediction.}
    \label{fig:nll_comp}
\end{figure}

\begin{figure}
    \centering
    \includegraphics[width=0.99\textwidth]{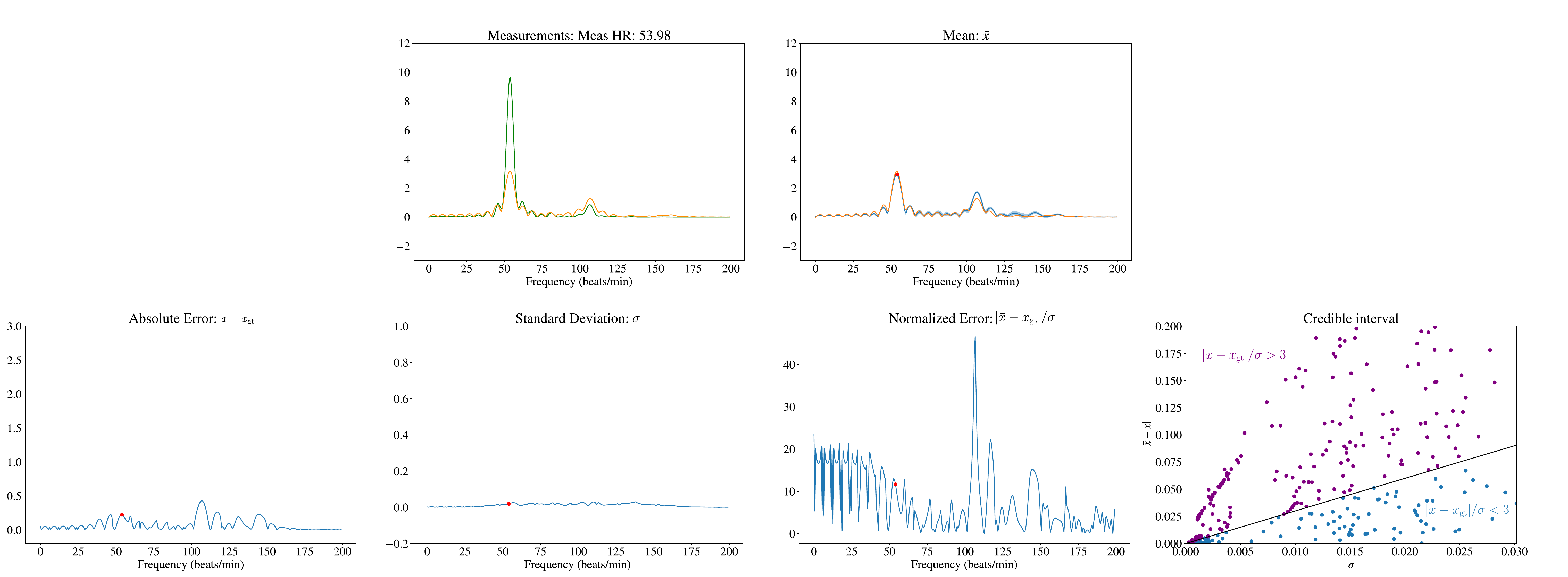}
    \caption{We plot the Measurements and reconstruction in the top row, and also plot the absolute error of the mean signal against the ground-truth and the standard deviation of the power in each frequency bin across all samples. We use these quantities to predict the normalized error, and plot the frequency bins in which the normalized error is greater and less than 3. Given that the ground-truth PPG signal is just an estimate of the pulse from face, we can see large normalized error during inference even though our reconstruction is nearly correct.}
    \label{fig:appendix_error_plots}
\end{figure}

\end{document}